\documentclass[a4paper,11pt]{article}
\usepackage[T1]{fontenc}
\usepackage{tikz}
\usetikzlibrary{arrows}
\usetikzlibrary{calc}
\usetikzlibrary{snakes}
\usepackage{cite}
\usepackage{graphicx}
\usepackage{array}
\usepackage{mdwmath}
\usepackage{mdwtab}
\usepackage{eqparbox}
\usepackage{fixltx2e}
\usepackage{url}
\usepackage{pgf}
\usepackage{lipsum}
\usepackage{multirow}
\usepackage{subcaption}
\usepackage[cmex10]{amsmath}
\usepackage{amsfonts,amssymb,amsthm}
\usepackage{xspace}
\usepackage{xcolor}
\usepackage{algorithmic}
\usepackage{bm}
\usepackage{bbm}
\usepackage{booktabs}
\usepackage{etoolbox}
\usepackage{makecell}
\usepackage{nicefrac}
\usepackage{textcomp}
\usepackage{stmaryrd}
\usepackage{mathrsfs}
\usepackage{paralist}
\usepackage{comment}

\usepackage[font=footnotesize,labelsep=space,labelfont=bf]{caption}

\usepackage{url}

\usepackage[pdftex,bookmarks=true,pdfstartview=FitH,colorlinks,linkcolor=blue,filecolor=blue,citecolor=blue,urlcolor=blue,pagebackref=false]{hyperref}
\makeatletter

\newcommand{\Rmnum}[1]{\expandafter\@slowromancap\romannumeral #1@}
\makeatother

\newtheorem{theorem}{Theorem}
\newtheorem*{theorem*}{Theorem}
\newtheorem{definition}{Definition}
\newtheorem{lemma}{Lemma}
\newtheorem*{lemma*}{Lemma}
\newtheorem{claim}{Claim}
\newtheorem{corollary}{Corollary}
\newtheorem*{cor*}{Corollary}
\newtheorem{remark}{Remark}

\newcommand{\namedref}[2]{\hyperref[#2]{#1~\ref*{#2}}}

\newcommand{\Algorithmref}[1]{\namedref{Algorithm}{algo:#1}}

\newcommand{\Sectionref}[1]{\namedref{Section}{sec:#1}}
\newcommand{\Subsectionref}[1]{\namedref{Section}{subsec:#1}}

\newcommand{\Appendixref}[1]{\namedref{Appendix}{app:#1}}
\newcommand{\Theoremref}[1]{\namedref{Theorem}{thm:#1}}
\newcommand{\Corollaryref}[1]{\namedref{Corollary}{cor:#1}}

\newcommand{\Definitionref}[1]{\namedref{Definition}{defn:#1}}
\newcommand{\Lemmaref}[1]{\namedref{Lemma}{lem:#1}}
\newcommand{\Remarkref}[1]{\namedref{Remark}{remark:#1}}
\newcommand{\Claimref}[1]{\namedref{Claim}{claim:#1}}

\newcommand{\Figureref}[1]{\namedref{Figure}{fig:#1}}

\newcommand{\Footnoteref}[1]{\namedref{Footnote}{foot:#1}}
\newcommand{\Pageref}[1]{\hyperref[#1]{page~\pageref*{#1}}}

\definecolor{darkred}{rgb}{0.5, 0, 0} 
\definecolor{darkblue}{rgb}{0,0,0.5} 
\hypersetup{
    colorlinks=true,
    linkcolor=darkred,
    citecolor=darkblue,
    urlcolor=darkblue   
}

\newcommand{\bzero}{\ensuremath{{\bf 0}}\xspace}

\newcommand{\doublebarbx}{\ensuremath{\bar{\bar{{\bf x}}}}\xspace}

\newcommand{\bx}{\ensuremath{{\bf x}}\xspace}

\newcommand{\bwx}{\ensuremath{\widehat{\bf x}}\xspace}
\newcommand{\btx}{\ensuremath{\widetilde{\bf x}}\xspace}
\newcommand{\by}{\ensuremath{{\bf y}}\xspace}

\newcommand{\bz}{\ensuremath{{\bf z}}\xspace}

\newcommand{\bu}{\ensuremath{{\bf u}}\xspace}
\newcommand{\bv}{\ensuremath{{\bf v}}\xspace}

\newcommand{\calS}{\ensuremath{\mathcal{S}}\xspace}

\newcommand{\bbE}{\ensuremath{\mathbb{E}}\xspace}

\newcommand{\I}{\ensuremath{\mathcal{I}}\xspace}

\newcommand{\R}{\ensuremath{\mathbb{R}}\xspace}

\newcommand{\bbZ}{\ensuremath{\mathbb{Z}}\xspace}

\renewcommand{\paragraph}[1]{\smallskip\noindent{\bf #1}~}

\usepackage{algorithm}
\usepackage{algorithmic}
\usepackage{fullpage}
\usepackage{authblk}

\usepackage{etoolbox}\AtBeginEnvironment{algorithmic}{\small}

\begin{document}

\title{Qsparse-local-SGD: Distributed SGD with 
Quantization, Sparsification, and Local Computations}

\author{Debraj Basu\thanks{Adobe Inc.; dbasu@adobe.com; work done while at UCLA.},\ \ Deepesh Data\thanks{UCLA; deepesh.data@gmail.com},\ \ Can Karakus\thanks{Amazon Web Services; cakarak@amazon.com; work done while at UCLA.},\ \ Suhas Diggavi\thanks{UCLA; suhasdiggavi@ucla.edu}}

\date{}

\maketitle

\begin{abstract}
Communication bottleneck has been identified as a significant issue in
distributed optimization of large-scale learning models. Recently,
several approaches to mitigate this problem have been proposed,
including different forms of gradient compression or computing
local models and mixing them iteratively. In this paper, we propose
\emph{Qsparse-local-SGD} algorithm, which combines aggressive
sparsification with quantization and local computation along with
error compensation, by keeping track of the difference between the
true and compressed gradients. We propose both synchronous and
asynchronous implementations of \emph{Qsparse-local-SGD}. We analyze
convergence for \emph{Qsparse-local-SGD} in the \emph{distributed} setting for
smooth non-convex and convex objective functions. We demonstrate that
\emph{Qsparse-local-SGD} converges at the same rate as vanilla
distributed SGD for many important classes of sparsifiers and
quantizers. We use \emph{Qsparse-local-SGD} to train ResNet-50 on ImageNet and 
show that it results in significant savings over the state-of-the-art, in the number of bits transmitted to reach target accuracy.
\end{abstract}

\section{Introduction}
\label{sec:intro}
Stochastic Gradient Descent (SGD) \cite{RobbinsMonro51} and its many
variants have become the workhorse for modern large-scale optimization
as applied to machine  learning \cite{Bottou10,bach_nonasymp}.
We consider a setup, in which SGD is applied to the \emph{distributed} setting, where $R$
different nodes compute \emph{local} stochastic gradients on their \emph{own} datasets
$\mathcal{D}_r$. Co-ordination between them is done by aggregating
 these local computations to update the overall parameter
$\mathbf{x}_t$ as,
$$\mathbf{x}_{t+1}= \mathbf{x}_t-\frac{\eta_t}{R}\sum_{r=1}^Rg_t^r,$$
where
$g_t^r$, for $r=1,2,\hdots,R$, is the local stochastic gradient at the $r$'th
machine for a local loss function $f^{(r)}(\mathbf{x})$ of the
parameter vector $\bx$, where $f^{(r)}:\mathbb{R}^d\rightarrow \mathbb{R}$ and $\eta_t$ is the learning rate.

Training of high dimensional models is typically performed at a large
scale over bandwidth limited networks. Therefore, despite the
distributed processing gains, it is well understood by now that
exchange of full-precision gradients between nodes causes
communication to be the bottleneck for many large-scale
models \cite{alistarh-sparsified,terngrad,signsgd1, teertha}.
For example,
consider training the BERT architecture for language models
\cite{BERT-arxiv18} which has about 340 million parameters, implying that each
full precision exchange between workers is over 1.3GB.  Such a
communication bottleneck could be significant in emerging edge
computation architectures suggested by federated
learning \cite{KonecnyThesis,McMahan16,AbadiOSDI16}.  In such an
architecture, data resides on and can even be generated by personal
devices such as smart phones, and other edge (IoT) devices, in
contrast to data-center architectures.  Learning is envisaged with
such an ultra-large scale, heterogeneous environment, with potentially
unreliable or limited communication. These and other applications have
led to many recently proposed methods, which are broadly based on
three major approaches:
\begin{enumerate}
\item \emph{Quantization}
of gradients, where nodes locally quantize the gradient (perhaps with
randomization) to a small number of bits
\cite{QSGD,signsgd1,ecqsgd,terngrad,teertha}.
\item \emph{Sparsification} of gradients, \emph{e.g.,} where nodes
locally select $\mathrm{Top}_k$ values of the gradient in absolute value and transmit
these at full
precision \cite{speech2,AjiHeafield17,memSGD,alistarh-sparsified,ecqsgd,DeepCompICLR18},
while maintaining errors in local nodes for later compensation.
\item \emph{Skipping communication rounds}, whereby nodes average their models after locally updating their models for several steps
\cite{alibaba_local,Coppola15, ZhangDuchiJMLR13,localsgd2,chen2016scalable,coop_sgd}.
\end{enumerate} 

In this work we propose a \emph{Qsparse-local-SGD} algorithm, which
combines aggressive sparsification with quantization and local
computations, along with error compensation, by keeping track of the
difference between the true and compressed gradients. We propose both
synchronous and asynchronous
implementations of \emph{Qsparse-local-SGD} in a \emph{distributed} setting, where the nodes perform computations on their local datasets. In our asynchronous
model, the distributed nodes' iterates evolve at the same rate, but
update the gradients at arbitrary times; see \Sectionref{loc-async}
for more details.
We
analyze convergence for \emph{Qsparse-local-SGD} in the \emph{distributed}
case, for smooth non-convex and smooth strongly-convex objective functions. We
demonstrate that \emph{Qsparse-local-SGD} converges at the same rate
as vanilla distributed SGD for many important classes of sparsifiers
and quantizers. We implement \emph{Qsparse-local-SGD} for ResNet-50
using the ImageNet dataset, and for a softmax multiclass classifier using the MNIST dataset, and we achieve target accuracies with about a factor of 15-20 savings over the state-of-the-art \cite{alistarh-sparsified,memSGD,localsgd2}, in the total number of
bits transmitted.

\subsection{Related Work}
The use of quantization for communication efficient gradient methods
has decades rich history \cite{GitlinMazo73} and its recent use in
training deep neural networks \cite{speech1,speech2} has re-ignited
interest.  Theoretically justified gradient compression using unbiased
stochastic quantizers has been proposed and analyzed
in \cite{QSGD,terngrad,teertha}.  Though methods
in \cite{wangni,atomo} use induced sparsity in the quantized
gradients, explicitly sparsifying the gradients more aggressively by
retaining $\mathrm{Top}_k$ components, \emph{e.g.,} $k<1\%$, has been
proposed \cite{speech2,AjiHeafield17,DeepCompICLR18,alistarh-sparsified,memSGD},
combined with error compensation to ensure that all co-ordinates do
get eventually updated as needed.  \cite{ecqsgd} analyzed error
compensation for QSGD, without $\mathrm{Top}_k$ sparsification while
focusing on quadratic functions. Another approach for mitigating the
communication bottlenecks is by having infrequent communication, which
has been popularly referred to in the literature as \textit{iterative
parameter mixing}, see \cite{Coppola15}, and \textit{model averaging},
see \cite{localsgd2,alibaba_local,zhang_local} and references
therein. Our work is most closely related to and builds on the recent
theoretical results
in \cite{alistarh-sparsified,memSGD,localsgd2,alibaba_local}.  The
analysis for the centralized $\mathrm{Top}_k$ (among other
sparsifiers) was considered in \cite{memSGD},
and \cite{alistarh-sparsified} analyzed a distributed version with the
assumption of closeness of the aggregated $\mathrm{Top}_k$ gradients
to the centralized $\mathrm{Top}_k$ case, see Assumption 1
in \cite{alistarh-sparsified}.  Local-SGD, where several local
iterations are done before sending the \emph{full} gradients, was
studied in \cite{localsgd2,alibaba_local}, without any gradient
compression beyond local iterations. Our work generalizes these works
in several ways. We prove convergence for the \emph{distributed}
sparsification and error compensation algorithm, without the
assumption of \cite{alistarh-sparsified}, by using the perturbed
iterate methods \cite{perturbed,memSGD}. We analyze non-convex as well
as convex objectives for the distributed case with local
computations. A proof of sparsified SGD for convex
objective functions and for the \emph{centralized} case, without local
computations
\footnote{At the completion of our work, we
recently found that in parallel to our work \cite{efsignsgd} examined
use of sign-SGD quantization, \emph{without sparsification} for the
centralized model. Another recent work in \cite{stich_gossip} studies
the decentralized case with sparsification for strongly convex
functions. In contrast, our work, developed independent of these
works, uses quantization, sparsification and local computations for
the distributed case, for both non-convex and strongly convex
objectives.}  was given in \cite{memSGD}.  Our techniques compose a
(stochastic or deterministic $1$-bit sign) quantizer with
sparsification and local computations using error compensation. While
our focus has only been on mitigating the communication bottlenecks in
training high dimensional models over bandwidth limited networks, this
technique works for any compression operator satisfying a regularity
condition (see
\Definitionref{compression}) including our composed operators.

\subsection{Contributions}
We study a distributed set of $R$ worker nodes, each of which perform
computations on locally stored data, denoted by $\mathcal{D}_r$.
Consider the empirical-risk minimization of the loss function
$$f(\mathbf{x})=\frac{1}{R}\sum_{r=1}^Rf^{(r)}(\bx)$$where
$f^{(r)}(\bx)=\underset{i\sim \mathcal{D}_r}{\mathbb{E}}\left[f_i(\bx)\right]$,
where $\underset{i\sim \mathcal{D}_r}{\mathbb{E}}[\cdot]$ denotes
expectation
over a random sample chosen from the local data set
$\mathcal{D}_r$. Our setup can also handle different local
functional forms, beyond dependence on the local data set
$\mathcal{D}_r$, which is not explicitly written for notational
simplicity.
For
$f:\mathbbm{R}^d\rightarrow \mathbbm{R}$, we denote
$\bx^*:=\textrm{arg}\min_{\bx\in\mathbb{R}^d}f(\bx)$ and
$f^*:=f(\bx^*)$.
The distributed nodes perform computations and
provide updates to the master node that is responsible for aggregation
and model update.  We develop \emph{Qsparse-local-SGD}, a distributed
SGD composing gradient quantization and explicit sparsification
(\emph{e.g.,} $\mathrm{Top}_k$ components), along with local
iterations. We develop the algorithms and analysis for both
synchronous as well as asynchronous operations, in which workers can
communicate with the master at arbitrary time intervals.  To the best
of our knowledge, these are the first algorithms which combine
quantization, aggressive sparsification, and local computations for
distributed optimization. With some minor modifications to \emph{Qsparse-local-SGD}, 
it can also be used in a peer-to-peer setting, where the aggregation is done without any help from the master node, 
and each worker exchanges its updates with all other workers.

Our main theoretical results are the convergence analyses
of \emph{Qsparse-local-SGD} for both non-convex as
well as convex objectives; see
\Theoremref{convergence-non-convex-fixed-local-het} and \Theoremref{convergence-strongly-convex-local} for the synchronous
case, as well as \Theoremref{async_convergence-non-convex-local-het2} and \Theoremref{async_convergence-strongly-convex-local-het},
for the asynchronous operation.  Our analyses also demonstrate
natural gains in convergence that distributed, mini-batch operation
affords, and has convergence similar to equivalent vanilla SGD with
local iterations (see \Corollaryref{convergence-non-convex-fixed-local-het-cor} and \Corollaryref{convergence-strongly-convex-local-cor}), for both the non-convex
case (with convergence rate $\sim \frac{1}{\sqrt{T}}$ for fixed
learning rate) as well as
the convex case (with convergence rate $\sim \frac{1}{T}$,
for diminishing learning rate). We also demonstrate that
quantizing and sparsifying the gradient, even after local iterations
asymptotically yields an almost ``free'' efficiency gain (also
observed numerically in \Sectionref{expmt}
non-asymptotically). The numerical results on ImageNet dataset
implemented for a ResNet-50 architecture and for the convex case for multi-class logistic classification on MNIST \cite{mnist} dataset demonstrates that one can get
significant communication savings, while retaining equivalent
state-of-the-art performance. The combination of quantization,
sparsification, and local computations poses several challenges for
theoretical analyses, including the analyses of impact of local
iterations (block updates) of parameters on quantization and
sparsification (see \Lemmaref{bounded-memory-decaying-lr}-\ref{lem:bounded-memory-fixed-lr} in \Sectionref{loc-sync}), as well as asynchronous updates and its combination with distributed compression
(see \Lemmaref{sequence-bound-asynchronous}-\ref{lem:bounded-memory-fixed-lr-asynchronous} in \Sectionref{loc-async}).

\subsection{Paper Organization}
In \Sectionref{operators}, we demonstrate that composing
certain classes of quantizers with sparsifiers satisfies a
certain regularity condition that is needed for several convergence
proofs for our algorithms. We describe the synchronous implementation
of \emph{Qsparse-local-SGD} in \Sectionref{loc-sync}, and
outline the main convergence results for it in
\Sectionref{main_results_sync}, briefly giving the proof ideas in
\Sectionref{pf_outline}.  We describe our asynchronous
implementation of \emph{Qsparse-local-SGD} and provide the theoretical
convergence results in \Sectionref{loc-async}. The
experimental results are given in \Sectionref{expmt}. Many
of the proof details are given in the
appendices, given as part of the supplementary material.

\section{Communication Efficient Operators}
\label{sec:operators}
Traditionally, distributed stochastic gradient descent affords to send full precision (32 or 64 bit) unbiased gradient updates across workers to peers or to a central server that helps with aggregation. However, communication bottlenecks that arise in bandwidth limited networks limit the applicability of such an algorithm at a large scale when the parameter size is massive or the data is widely distributed on a very large number of worker nodes. In such settings, one could think of updates which not only result in convergence, but also require less bandwidth thus making the training process faster. In the following sections we discuss several useful operators from literature and enhance their use by proposing a novel class of composed operators.

We first consider two different techniques used in the literature for mitigating the communication bottleneck in distributed optimization, namely, quantization and sparsification. 
In quantization, we reduce precision of the gradient vector by mapping each of its components by a deterministic \cite{signsgd1,efsignsgd} or randomized \cite{QSGD,terngrad,teertha,ZhangNIPS13} map to a finite number of quantization levels. In sparsification, we sparsify the gradients vector before using it to update the parameter vector, by taking its top $k$ components, denoted by $\mathrm{Top}_k$, or choosing $k$ components uniformly at random, denoted by $\mathrm{Rand}_k$, \cite{memSGD,stich_gossip}.

\subsection{Quantization}
SGD computes an unbiased estimate of the gradient, which can be used to update the model iteratively and is extremely useful in large scale applications. It is well known that the first order terms in the rate of convergence are affected by the variance of the gradients. 
While stochastic quantization of gradients could result in a variance blow up, it preserves the unbiasedness of the gradients at low precision; and, therefore, when training over bandwidth limited networks, the convergence would be much faster; see \cite{QSGD,terngrad,teertha,ZhangNIPS13}. 
\begin{definition}[Randomized quantizer]\label{defn:QuantDef}
We say that $Q_s:\R^d\to\R^d$ is a randomized quantizer with $s$ quantization levels, if the following holds for every $\mathbf{x}\in\R^d$:
\textsf{(i)} $\mathbb{E}_Q[Q_s(\mathbf{x})]=\mathbf{x}$; \textsf{(ii)}
$\mathbb{E}_Q[\|Q_s(\mathbf{x})\|^2]\leq (1+\beta_{d,s})\|\mathbf{x}\|^2$, where $\beta_{d,s}>0$ could be a function of $d$ and $s$. 
Here expectation is taken over the randomness of $Q_s$.
\end{definition}
\noindent Examples of randomized quantizers include
\begin{enumerate}
\item \textit{QSGD}~\cite{QSGD,terngrad}, which independently quantizes components of $\mathbf{x}\in\R^d$ into $s$ levels, with $\beta_{d,s}=\min(\frac{d}{s^2},\frac{\sqrt{d}}{s})$.
\item \textit{Stochastic $s$-level
    Quantization}~\cite{teertha,ZhangNIPS13}, which independently quantizes
every component of $\mathbf{x}\in\R^d$ into $s$ levels between
$\textrm{argmax}_ix_i$ and $\textrm{argmin}_ix_i$, with
$\beta_{d,s}=\frac{d}{2s^2}$.
\item \textit{Stochastic Rotated Quantization}~\cite{teertha}, which is a stochastic quantization,
preprocessed by a random rotation, with $\beta_{d,s}=\frac{2\log_2(2d)}{s^2}$.
\end{enumerate}

Instead of quantizing randomly into $s$ levels, we can take a deterministic approach 
and round off each component of the vector to the nearest level. In particular, we can just take the sign, which has shown promise in \cite{signsgd1,efsignsgd}.
\begin{definition}[Deterministic Sign quantizer]\label{defn:sign_operator}
A deterministic quantizer $Sign:\R^d\to\{+1,-1\}^d$ is defined as follows: 
for every vector $\mathbf{x}\in\R^d$, the $i$'th component of $Sign(\mathbf{x})$, for $i\in[d]$, is defined as $\mathbbm{1}\{x_i\geq 0\}-\mathbbm{1}\{x_i<0\}$.
\end{definition}
Such methods drew interest since \textsc{Rprop} \cite{rprop}, which only used the temporal behavior of the sign of the gradient. This is an example where the biased 1-bit quantizer as in \Definitionref{sign_operator} is used. This further inspired optimizers, such as \textsc{RMSprop} \cite{rmsprop}, \textsc{Adam} \cite{adam}, which incorporate appropriate adaptive scaling with momentum acceleration and have demonstrated empirical superiority in non-convex applications.
\subsection{Sparsification}
As mentioned earlier, we consider 
two important examples of sparsification operators: $\mathrm{Top}_k$ and
$\mathrm{Rand}_k$. For any $\mathbf{x}\in\R^d$, $\mathrm{Top}_k(\mathbf{x})$ is equal to
a $d$-length vector, which has at most $k$ non-zero components whose
indices correspond to the indices of the largest $k$ components (in
absolute value) of $\mathbf{x}$. Similarly, $\mathrm{Rand}_k(\mathbf{x})$ is a
$d$-length (random) vector, which is obtained by selecting $k$ components of $\mathbf{x}$ uniformly at random.
Both of these satisfy a so-called ``compression'' property as defined below, with
$\gamma=k/d$
\cite{memSGD}.

\begin{definition}[Compression operator \cite{memSGD}]\label{defn:compression}
A (randomized) function $Comp_k:\R^d\to\R^d$ is called a {\em
  compression} operator, if there exists a constant
  $\gamma\in(0,1]$ (that may depend on $k$ and $d$), such that for every
  $\mathbf{x}\in \R^d$, we have 
\begin{align}\label{eq:compression}
\mathbb{E}_C[\|\mathbf{x}-Comp_k(\mathbf{x})\|_2^2]\leq (1-\gamma)\|\mathbf{x}\|_2^2,
\end{align}
where expectation is taken over the randomness of the compression operator $Comp_k$.
\end{definition}
Note that stochastic quantizers, as defined in \Definitionref{QuantDef}, also satisfy this regularity condition in \Definitionref{compression} for $\beta_{d,s}\leq 1$.
Now we give a simple but important corollary, which allows us to apply different compression operators to different coordinates of a vector. 
As an application, in the case of training neural networks, we can apply different operators to different layers.
\begin{corollary}[Piecewise compression]\label{cor:piecewise_compression}
Let $C_i:\mathbb{R}^{d_i}\rightarrow\mathbb{R}^{d_i}$ for $i\in[L]$ denote possibly different compression operators with compression coefficients $\gamma_i$. Let $\mathbf{x}=\left[\mathbf{x}_1\,\mathbf{x}_2\,\ldots\mathbf{x}_L\right]$, where $\mathbf{x}_i\in\mathbb{R}^{d_i}$ for all $i\in[L]$. 
Then $C(\bx):=\left[C_1(\mathbf{x}_1)\,C_2(\mathbf{x}_2)\,\ldots C_L(\mathbf{x}_L)\right]$ is a compression operator with the compression coefficient being equal to $\gamma_{min}=\underset{i\in[L]}{\min}\gamma_i$.
\end{corollary}
\begin{proof}
Fix an arbitrary $\mathbf{x}\in\R^d$. The result follows from the following set of inequalities:
$\mathbb{E}_C\|\mathbf{x}-C(\mathbf{x})\|_2^2 = \sum_{i=1}^L\mathbb{E}_{C_i}\|\mathbf{x}_i-C_i(\mathbf{x}_i)\|_2^2 
\stackrel{\text{(a)}}{\leq} \sum_{i=1}^L(1-\gamma_i)\|\mathbf{x}_i\|_2^2 
\leq (1-\gamma_{min})\|\mathbf{x}\|_2^2$,
where inequality (a) follows because each $C_i$ is a compression operator with the compression coefficient $\gamma_i$.
\end{proof}
\noindent\Corollaryref{piecewise_compression} allows us to apply different compression operators to different coordinates of the updates which can based upon their dimensionality and sparsity patterns.

\subsection{Composition of Quantization and Sparsification}\label{subsec:composed-operators}
Now we show that we can compose deterministic/randomized quantizers with sparsifiers and the resulting operator is a compression operator. 
First we compose a general stochastic quantizer with an explicit sparsifier, such as $\textrm{Top}_k(\bx)$ and $\textrm{Rand}_k(\bx)$, 
and show that the resulting operator is a ``compression'' operator. A proof is provided in \Appendixref{op_1}.
\begin{lemma}[Compression of a composed operator]\label{lem:composed-compression}
Let $Comp_k\in\{\mathrm{Top}_k,\mathrm{Rand}_k\}$. Let $Q_s:\mathbb{R}^d\rightarrow\mathbb{R}^d$ be a quantizer with parameter $s$ that satisfies \Definitionref{QuantDef}. 
Let $Q_sComp_k:\R^d\to\R^d$ be defined as $Q_sComp_k(\mathbf{x}) := Q_s(Comp_k(\mathbf{x}))$ for every $\mathbf{x}\in\R^d$. If $k,s$ are such that $\beta_{k,s}<1$, then $Q_sComp_k:\R^d\to\R^d$ is a compression operator with the compression coefficient being equal to $\gamma=(1-\beta_{k,s})\frac{k}{d}$, i.e., for every $\mathbf{x}\in\R^d$, we have
\begin{align*}
  \mathbb{E}_{C,Q}[\|\mathbf{x}-Q_sComp_k(\mathbf{x})\|_2^2]\leq \left[1-\left(1-\beta_{k,s}\right)\frac{k}{d}\right]\|\mathbf{x}\|_2^2,
\end{align*}
where expectation is taken over the randomness of the compression operator $Comp_k$ as well as of the quantizer $Q_s$.
\end{lemma}
 For the different quantizers mentioned earlier, the conditions
 when their composition with $Comp_k$ gives $\beta_{k,s}<1$ are:
 \begin{enumerate}
 \item  \emph{QSGD:} for $k<s^2$, we get.
 $\gamma=\left(1-\frac{k}{s^2}\right)\frac{k}{d}$
 \item  \emph{Stochastic k-level Quantization:} for $k<2s^2$, we get
 $\gamma=\left(1-\frac{k}{2s^2}\right)\frac{k}{d}$.
 \item \emph{Stochastic Rotated Quantization:} for $k<2^{s^2/2-1}$, we get $\gamma=\left(1-\frac{2\log_2(2k)}{s^2}\right)\frac{k}{d}$.
 \end{enumerate} 

\begin{remark}
Observe that for a given stochastic quantizer that satisfies \Definitionref{QuantDef}, we have a prescribed operating regime of $\beta_{k,s}<1$. This results in an upper bound on the coarseness of the quantizer, which happens because the quantization leads to a blow-up of the second moment; see condition {\sf (ii)} of \Definitionref{QuantDef}. 
However, by employing \Corollaryref{piecewise_compression}, we show that this can be alleviated to some extent via an example.

Consider an operator as described in \Lemmaref{composed-compression}, where the quantizer, $Q_s:\R^d\rightarrow \R^d$ in use is QSGD \cite{QSGD,terngrad}, and the sparsifier, $Comp_k$ is $Top_k$ \cite{alistarh-sparsified,memSGD}. Apply it to a vector $\mathbf{x}=\left[\mathbf{x}_1\,\mathbf{x}_2\,\ldots\mathbf{x}_L\right]\in\mathbb{R}^d$ in a piecewise manner, i.e., $Q_{s_i}Comp_{k_i}:\R^{d_i}\rightarrow\R^{d_i}$ to smaller vectors $\mathbf{x}_i\in\mathbb{R}^{d_i}$ as prescribed in \Corollaryref{piecewise_compression}. Define $\beta_{k_i,s_i}=\frac{k_i}{s_i^2}$ as the coefficient of the variance bound as in \Definitionref{QuantDef} for the quantizer $Q_{s_i}$, used for $\mathbf{x}_i$ and $k:=\sum_{i=1}^L k_i$. 
Observe that the regularity condition in \Definitionref{compression} can be satisfied by having ${k_i}<s_i^2$. Therefore, the piecewise compression operator allows a coarser quantizer than when the operator is applied to the entire vector together where we require $\beta_{k,s}=\frac{k}{s^2}<1$, thus providing a small gain in communication efficiency. For example, consider the composed operator being applied on a per layer basis to a deep neural network. We can now afford to have a much coarser quantizer than when the operator is applied to all the parameters at once.
\end{remark}
As discussed above, stochastic quantization results in a variance blow-up which limits our regime of operation, when we combine that with sparsification.
However, it turns out that, we can expand our regime of operation unrestrictedly by scaling the vector $Q_sComp_k(\bx)$ appropriately.
We summarize the result in the following lemma, which is proved in \Appendixref{op_2}.
\begin{lemma}[Composing sparsification with stochastic quantization]\label{lem:composed-quantizer}
Let $Comp_k\in\{\mathrm{Top}_k,\mathrm{Rand}_k\}$. Let $Q_s:\mathbb{R}^d\rightarrow\mathbb{R}^d$ be a stochastic quantizer with parameter $s$ that satisfies \Definitionref{QuantDef}. 
Let $Q_sComp_k:\R^d\to\R^d$ be defined as $Q_sComp_k(\mathbf{x}) := Q_s(Comp_k(\mathbf{x}))$ for every $\mathbf{x}\in\R^d$. 
Then $\frac{Q_sComp_k(\mathbf{x})}{1+\beta_{k,s}}$ is a compression operator with the compression coefficient being equal to $\gamma=\frac{k}{d(1+\beta_{k,s})}$, 
i.e., for every $\bx\in\R^d$, we have
\begin{align*}
  \mathbb{E}_{C,Q}\left[\left\|\mathbf{x}-\frac{Q_sComp_k(\mathbf{x})}{1+\beta_{k,s}}\right\|_2^2\right]\leq \left[1-\frac{k}{d(1+\beta_{k,s})}\right]\|\mathbf{x}\|_2^2.
\end{align*}
\end{lemma}
\begin{remark}\label{remark:scaled-vs-unscaled}
Note that, unlike $Q_sComp_k(\mathbf{x})$, the {\em scaled} version $\frac{Q_sComp_k(\mathbf{x})}{1+\beta_{k,s}}$ is always a compression operator for all values of $\beta_{k,s}>0$.
Furthermore, observe that, if $\beta_{k,s}<1$, then we have $(1-\beta_{k,s})\frac{k}{d}<\frac{k}{d(1+\beta_{k,s})}$, which implies that
 even in the operating regime of $\beta_{k,s}<1$, which is required in \Lemmaref{composed-compression}, the {\em scaled} composed operator $\frac{Q_sComp_k(\mathbf{x})}{1+\beta_{k,s}}$ of \Lemmaref{composed-quantizer} gives better compression than what we get from the {\em unscaled} composed operator $Q_sComp_k(\mathbf{x})$ of \Lemmaref{composed-compression}.
So, appropriately scaled composed operator is always a better choice for compression.
\end{remark}
We can also compose a {\em deterministic} 1-bit quantizer $Sign$ with $Comp_k$. For that we need some notations first.
For $Comp_k\in\{\mathrm{Top}_k,\mathrm{Rand}_k\}$ and given vector $\mathbf{x}\in\R^d$, let $\calS_{Comp_k(\mathbf{x})}\in\binom{[d]}{k}$ denote the set of $k$ indices chosen for defining $Comp_k(\mathbf{x})$. For example, if $Comp_k=\mathrm{Top}_k$, then $\calS_{Comp_k(\mathbf{x})}$ denote the set of $k$ indices corresponding to the largest $k$ components of $\mathbf{x}$; if  $Comp_k=\mathrm{Rand}_k$, then $\calS_{Comp_k(\mathbf{x})}$ denote a set of random set of $k$ indices in $[d]$.
The composition of $Sign$ with $Comp_k\in\{\mathrm{Top}_k,\mathrm{Rand}_k\}$ is denoted by $SignComp_k:\R^d\to\R^d$, and for $i\in[d]$, the $i$'th component of $SignComp_k(\mathbf{x})$ is defined as
\[(SignComp_k(\mathbf{x}))_i :=
\begin{cases}
\mathbbm{1}\{x_i\geq 0\}-\mathbbm{1}\{x_i<0\} & \text{ if } i \in \calS_{Comp_k(\mathbf{x})}, \\
0 & \text {otherwise}.
\end{cases}
\]
In the following lemma we show that $SignComp_k$ is a compression operator; a proof of which is provided in \Appendixref{op_3}.
\begin{lemma}[Composing sparsification with deterministic quantization]\label{lem:composed-sign}
For $Comp_k\in\{\mathrm{Top}_k,\mathrm{Rand}_k\}$, the operator $$\frac{\|Comp_k(\mathbf{x})\|_m \,SignComp_k(\mathbf{x})}{k}$$ for any $m\in\mathbb{Z}_+$ is a compression operator with the compression coefficient $\gamma_m$ being equal to
\[\gamma_m = 
\begin{cases}
\max\left\{\frac{1}{d}, \frac{k}{d}\left(\frac{\|Comp_k(\mathbf{x})\|_1}{\sqrt{d}\|Comp_k(\mathbf{x})\|_2}\right)^2\right\} & \text{ if } m=1, \\
\frac{k^{\frac{2}{m}-1}}{d} & \text{ if } m\geq 2.
\end{cases}
\]
\end{lemma}
\begin{remark}
Observe that for $m=1$, depending on the value of $k$, either of the terms inside the max can be bigger than the other term. For example, if $k=1$, then $\|Comp_k(\mathbf{x})\|_1=\|Comp_k(\mathbf{x})\|_2$, which implies that the second term inside the max is equal to $1/d^2$, which is much smaller than the first term. On the other hand, if $k=d$ and the vector $\mathbf{x}$ is dense, then the second term may be much bigger than the first term.
\end{remark}

\section{Distributed Synchronous Operation}
\label{sec:loc-sync}
\label{sec:synch_algo}
Let $\I_T^{(r)}\subseteq[T]:=\{1,\hdots,T\}$ with $T\in\I_T^{(r)}$ denote a set of indices for which worker $r\in[R]$ synchronizes with the master.  
In a synchronous setting, $\I_T^{(r)}$ is same for all the workers. Let $\I_T:=\I_T^{(r)}$ for any $r\in[R]$.
Every worker $r\in[R]$ maintains a local parameter vector $\bwx_t^{(r)}$ which is updated in each iteration $t$.
If $t\in\I_T$, every worker $r\in[R]$ sends the compressed and error-compensated update $g_t^{(r)}$ computed on the net progress made since the last synchronization to the master node, and updates its local memory $m_{t}^{(r)}$. Upon receiving $g_t^{(r)}, r=1,2.\hdots,R$, master aggregates them, updates the global parameter vector, and sends the new model $\bx_{t+1}$ to all the workers; upon receiving which, they set their local parameter vector $\bwx_{t+1}^{(r)}$ to be equal to the global parameter vector $\bx_{t+1}$. Our algorithm is summarized in \Algorithmref{memQSGD-synchronous}.

\begin{algorithm}[h]
   \caption{Qsparse-local-SGD}
   \label{algo:memQSGD-synchronous}
\begin{algorithmic}[1]
   \STATE Initialize $\mathbf{x}_0=\widehat{\mathbf{x}}_0^{\left(r\right)}=m_0^{\left(r\right)}=\bzero,\,\,\forall r\in [R]$. 
Suppose $\eta_t$ follows a certain learning rate schedule. 
   \FOR{$t=0$ {\bfseries to} $T-1$}
  \STATE \textbf{On Workers:}
   \FOR{$r=1$ {\bfseries to} $R$}
   \STATE $\widehat{\mathbf{x}}_{t+\frac{1}{2}}^{\left(r\right)}\leftarrow\widehat{\mathbf{x}}_t^{\left(r\right)}-\eta_t\nabla f_{i_t^{(r)}}\left(\widehat{\mathbf{x}}_t^{\left(r\right)}\right)$; $i_t^{(r)}$ is a mini-batch of size $b$ uniformly in $\mathcal{D}_r$
   
   \IF{$t+1\notin\mathcal{I}_T$}
   \STATE $\mathbf{x}_{t+1}\leftarrow \mathbf{x}_{t}$, $m_{t+1}^{\left(r\right)}\leftarrow m_{t}^{\left(r\right)}$ and $\widehat{\mathbf{x}}_{t+1}^{(r)}\leftarrow\widehat{\mathbf{x}}_{t+\frac{1}{2}}^{(r)}$
      \ELSE
   \STATE $g_t^{\left(r\right)}\leftarrow Q\,Comp_k\left(m_t^{\left(r\right)}+\mathbf{x}_t-\widehat{\mathbf{x}}_{t+\frac{1}{2}}^{\left(r\right)}\right)$, send $g_t^{\left(r\right)}$ to the master
   \STATE $m_{t+1}^{\left(r\right)}\leftarrow m_t^{\left(r\right)}+\mathbf{x}_t-\widehat{\mathbf{x}}_{t+\frac{1}{2}}^{\left(r\right)}-g_t^{\left(r\right)}$
   \STATE Receive $\mathbf{x}_{t+1}$ from the master and set $\widehat{\mathbf{x}}_{t+1}^{\left(r\right)}\leftarrow \mathbf{x}_{t+1}$
   \ENDIF
   
   \ENDFOR
   \STATE \textbf{At Master:}
   \IF{$t+1\notin\mathcal{I}_T$}
   \STATE $\mathbf{x}_{t+1}\leftarrow \mathbf{x}_{t}$
   \ELSE
   \STATE Receive $g_t^{\left(r\right)}$ from $R$ workers and compute $\mathbf{x}_{t+1}=  \mathbf{x}_t-\frac{1}{R}\sum_{r=1}^Rg_t^{\left(r\right)}$\label{lst:update}
   \STATE Broadcast $\mathbf{x}_{t+1}$ to all workers
   \ENDIF
   \ENDFOR
   \STATE \textbf{Comment:} $\bwx_{t+\frac{1}{2}}^{(r)}$ is used to denote an intermediate variable between iterations $t$ and $t+1$
\end{algorithmic}
\end{algorithm}

\subsection{Assumptions}\label{subsec:assumptions}
All results in this paper use the following two standard assumptions.
\begin{enumerate}
    \item \textbf{Smoothness:} 
The local function $f^{(r)}:\R^d\to\R$ at each worker $r\in[R]$ is $L$-smooth, i.e., for every
$\mathbf{x},\mathbf{y}\in\R^d$, we have $f^{(r)}(\mathbf{y})\leq f^{(r)}(\mathbf{x})+\langle\nabla
f^{(r)}(\mathbf{x}),\by-\bx\rangle+\frac{L}{2}\|\by-\bx\|_2^2$. 
\item \textbf{Bounded second moment:} For every $\bwx_t^{(r)}\in\R^d,r\in[R],t\in[T]$ and for some constant $0\leq G<\infty$, we have $\underset{i\sim\mathcal{D}_r}{\mathbb{E}}[\|\nabla f_i(\bwx_t^{(r)})\|_2^2]\leq G^2$. This is a standard assumption in \cite{shalev-shwartz,nemirovski,hogwild_recht,hazan,rakhlin,memSGD,localsgd2,alibaba_local,stich_gossip,alistarh-sparsified}. 
Relaxation of the uniform boundedness of the gradient allowing arbitrarily different gradients of local functions in heterogenous settings as done for SGD in \cite{nguyen18,coop_sgd} is left for future work. 
This also imposes a \textbf {bound on the variance}:
$\underset{i\sim\mathcal{D}_r}{\mathbb{E}}[\|\nabla f_i(\widehat{\bx}_t^{(r)})-\nabla f^{(r)}(\widehat{\bx}_t^{(r)})\|_2^2]\leq \sigma_r^2$,
where $\sigma_r^2\leq G^2$ for every $r\in[R]$. 
\end{enumerate}
In this section we present our main convergence results with synchronous updates, obtained by running \Algorithmref{memQSGD-synchronous} for smooth functions, both non-convex and strongly convex. 
To state our results, we need the following definition from \cite{localsgd2}.
\begin{definition}[Gap \cite{localsgd2}]
Let $\I_T=\{t_0,t_1,\hdots,t_k\}$, where $t_i<t_{i+1}$ for $i=0,1,\hdots,k-1$. The gap of $\I_T$ is defined as $gap(\I_T):=\max_{i\in[k]}\{(t_i-t_{i-1})\}$, 
which is equal to the maximum difference between any two consecutive synchronization indices.
\end{definition}

 \subsection{Error Compensation}
 \label{sec:error_comp}
 Sparsified gradient methods, where workers send the largest $k$ coordinates of the updates based on their magnitudes have been investigated in the literature and serves as a communication efficient strategy for distributed training of learning models. However, the convergence rates are subpar to distributed vanilla SGD. Together with some form of error compensation, these methods have been empirically observed to converge as fast as vanilla SGD in \cite{speech2,AjiHeafield17,DeepCompICLR18,alistarh-sparsified,memSGD}. In \cite{alistarh-sparsified,memSGD}, sparsified SGD with such feedback schemes has been carefully analyzed. Under analytic assumptions, \cite{alistarh-sparsified} proves the convergence of distributed $\textrm{Top}_k$ SGD with error feedback. The net error in the system is accumulated by each worker locally on a per iteration basis and this is used as feedback for generating the future updates. \cite{memSGD} did the analysis for the centralized $\textrm{Top}_k$ SGD for strongly convex objectives. 

In \Algorithmref{memQSGD-synchronous}, the error introduced in every iteration is accumulated into the memory of each worker, which is compensated for in the future rounds of communication. This feedback is the key to recovering the convergence rates matching vanilla SGD. The operators employed provide a controlled way of using both the current update as well as the compression errors from the previous rounds of communication. Under the assumption of the uniform boundedness of the gradients, we analyze the controlled evolution of memory through the optimization process; the results are summarized in \Lemmaref{bounded-memory-decaying-lr} and \Lemmaref{bounded-memory-fixed-lr} below. 
\subsubsection{Decaying Learning Rate}
Here we show that if we run \Algorithmref{memQSGD-synchronous} with a decaying learning rate $\eta_t$, then the local memory at each worker contracts and goes to zero as $\mathcal{O}(\eta_t)^2$.
\begin{lemma}[Memory contraction]\label{lem:bounded-memory-decaying-lr}
Let $gap(\mathcal{I}_T)\leq H$ and $\eta_t=\frac{\xi}{a+t}$, where $\xi$ is a constant and $a>\frac{4H}{\gamma}$, with $\gamma$ being the compression coefficient of the compression operator. 
Then there exists a constant $C\geq \frac{4a\gamma(1-\gamma^2)}{a\gamma-4H}$, such that the following holds for every $t\in \mathbb{Z}^+$ and $r\in[R]$:
\begin{align}
    \mathbb{E}\|m_t^{(r)}\|_2^2\leq 4\frac{\eta_t^2}{\gamma^2}CH^2G^2. \label{memory-bound-deacy-learning-rate}
\end{align}
\end{lemma}
We prove \Lemmaref{bounded-memory-decaying-lr} in \Appendixref{mem_dec}.
Note that for fixed $\gamma,H$, the memory decays as $\mathcal{O}(\eta_t^2)$.
This implies that the net error in the algorithm from the compression of updates in each round of communication is compensated for in the end. 

\subsubsection{Fixed Learning Rate}
In the following lemma, which is proved in \Appendixref{mem_fix}, we show that if we run \Algorithmref{memQSGD-synchronous} with a fixed learning rate $\eta_t=\eta, \forall t$, then the local memory at each worker is bounded.
It can be verified that the proof of \Lemmaref{bounded-memory-decaying-lr} also holds for fixed learning rate, and we can trivially bound $\mathbb{E}\|m_t^{(r)}\|_2^2$ in this case by simply putting $\eta_t=\eta$ in \eqref{memory-bound-deacy-learning-rate}.
However, we can get a better bound (saving a factor of $\frac{C}{1-\gamma^2}$, which is bigger than 4) by directly working with a fixed learning rate.
\begin{lemma}[Bounded memory]\label{lem:bounded-memory-fixed-lr}
Let $gap(\mathcal{I}_T)\leq H$. Then the following holds for every worker $r\in[R]$ and for every $t\in \mathbb{Z}^+$:
\begin{align}
    \mathbb{E}\|m_t^{(r)}\|_2^2\leq 4\frac{\eta^2(1-\gamma^2)}{\gamma^2}H^2G^2. \label{memory-bound-fixed-learning-rate}
\end{align}
\end{lemma}
Note that, for fixed $\gamma,H$, the memory is upper bounded by a constant $\mathcal{O}(\eta^2)$. 
Observe that since the memory accumulates the past errors due to compression and local computation, in order to asymptotically reduce the memory to zero, the learning rate would have to be reduced once in a while throughout the training process.

\subsection{Main Results}\label{sec:main_results_sync}
We leverage the perturbed iterate analysis as in \cite{perturbed,memSGD} to provide convergence guarantees for \textit{Qsparse-local-SGD}. Under the assumptions of \Subsectionref{assumptions}, the following theorems hold when \Algorithmref{memQSGD-synchronous} is run with any compression operator (including our composed operators).
\begin{theorem}[Smooth (non-convex) case with fixed learning rate]\label{thm:convergence-non-convex-fixed-local-het}
Let $f^{(r)}(\mathbf{x})$ be $L$-smooth for every $i\in[R]$.  Let
$QComp_k:\R^d\to\R^d$ be a compression operator whose compression
coefficient is equal to $\gamma\in(0,1]$.  Let
$\{\bwx_t^{(r)}\}_{t=0}^{T-1}$ be generated according to \Algorithmref{memQSGD-synchronous} with
$QComp_k$, for step sizes $\eta=\frac{\widehat{C}}{\sqrt{T}}$ (where $\widehat{C}$ is a constant such that $\frac{\widehat{C}}{\sqrt{T}}\leq \frac{1}{2L}$) and  $gap(\mathcal{I}_T)\leq H$. Then we have\vspace{-0.15cm}
\begin{align*}
    \mathbb{E}\|\nabla f(\mathbf{z}_T)\|_2^2&\leq \left(\frac{\mathbb{E}[f(\bx_0)]-f^*}{\widehat{C}}+\widehat{C} L\left(\frac{\sum_{r=1}^R\sigma_r^2}{bR^2}\right)\right)\frac{4}{\sqrt{T}}+8\left(4\frac{(1-\gamma^2)}{\gamma^2}+1\right) \frac{\widehat{C}^2L^2G^2H^2}{T}.
\end{align*}
Here $\mathbf{z}_T$ is a random variable which samples a previous parameter $\widehat{\mathbf{x}}_t^{(r)}$ with probability $1/RT$. 
\end{theorem}
\begin{corollary}\label{cor:convergence-non-convex-fixed-local-het-cor}
Let $\mathbb{E}[f(\bx_0)]-f^*\leq J^2$, where $J<\infty$ is a constant,\footnote{Even classical SGD requires knowing an upper bound on $\|\bx_0-\bx^*\|$ in order to choose the learning rate. Smoothness of $f$ translates this to the difference of the function values.} $\sigma_{max}=\max_{r\in[R]}\sigma_r$, and $\widehat{C}^2=\frac{bR(\mathbb{E}[f(\bx_0)]-f^*)}{\sigma_{max}^2L}$, we have\vspace{-0.25cm}
\begin{align*}
    \mathbb{E}\|\nabla f(\mathbf{z}_T)\|_2^2\leq\mathcal{O}\left(\frac{J\sigma_{max}}{\sqrt{bRT}}\right)+\mathcal{O}\left(\frac{J^2bRG^2H^2}{\sigma_{max}^2\gamma^2T}\right).
\end{align*}
In order to ensure that the compression does not affect the dominating terms while converging at a rate of $\mathcal{O}\left(1/\sqrt{bRT}\right)$, we would require $H=\mathcal{O}\left(\gamma T^{1/4}/(bR)^{3/4}\right)$.
\end{corollary}
\Theoremref{convergence-non-convex-fixed-local-het} is proved in \Appendixref{smooth_proof_sync} 
and provides non-asymptotic guarantees, where we observe that compression does not affect the first order term. 
Here, we are required to decide the horizon $T$ before running the algorithm. Therefore, in order to converge to a fixed point, the learning rate needs to follow a piecewise schedule (i.e., the learning rate would have to be reduced once in a while throughout the training process), which is also the case in our numerics in \Sectionref{numerics_NC}. The corresponding asymptotic result (with decaying learning rate) is given below.
\begin{theorem}[Smooth (non-convex) case with decaying learning rate]\label{thm:convergence-non-convex-decay-local-het}
Let $f^{(r)}(\mathbf{x})$ be $L$-smooth for every $r\in[R]$.  Let
$QComp_k:\R^d\to\R^d$ be a compression operator whose compression
coefficient is equal to $\gamma\in(0,1]$. Let
$\{\bwx_t^{(r)}\}_{t=0}^{T-1}$ be generated according to \Algorithmref{memQSGD-synchronous} with
$QComp_k$, for step sizes $\eta_t=\frac{\xi}{(a+t)}$ and  $gap(\mathcal{I}_T)\leq H$, where $a>1$ is
such that we have
$a> \max\{\frac{4H}{\gamma},2\xi L, H \}$ and $C\geq \frac{4a\gamma(1-\gamma^2)}{a\gamma-4H}$.  Then the following
holds.
\begin{align*}
    \mathbb{E}\|\nabla f(\mathbf{z}_T)\|^2&\leq \frac{\mathbb{E}f(\mathbf{x}_0)-f^*}{P_T}+\frac{L\xi^2}{(a-1)P_T}\left(\frac{\sum_{r=1}^R\sigma_r^2}{bR^2}\right)
    +\left(\frac{8C}{\gamma^2}+8\right)\frac{\xi^3L^2G^2H^2}{2(a-1)^2P_T}.
\end{align*}
Here\ {\sf(i)} $\delta_t:=\frac{\eta_t}{4R}$; 
{\sf(ii)} $P_{T}:=\sum_{t=0}^{T-1}\sum_{r=1}^R\delta_t$, which is lower bounded as
$P_{T}\geq \frac{\xi}{4}\ln{\left(\frac{T+a-1}{a}\right)}$;
and {\sf(iii)} $\mathbf{z}_T$ is a random variable which samples a previous parameter $\widehat{\mathbf{x}}_t^{(r)}$ with probability $\delta_t/P_{T}$.
\end{theorem}
Note that \Theoremref{convergence-non-convex-decay-local-het} gives a convergence rate of $\mathcal{O}(\frac{1}{\log T})$. We prove it in \Appendixref{proof_convergence-non-convex-decay-local-het}.

\begin{theorem}[Smooth and strongly convex case with a decaying learning rate]\label{thm:convergence-strongly-convex-local}
Let $f^{(r)}\left(\mathbf{x}\right)$ be $L$-smooth 
and $\mu$-strongly convex.
Let $QComp_k:\R^d\to\R^d$ be a compression operator whose compression coefficient is equal to $\gamma\in\left(0,1\right]$.
Let $\{\widehat{\bx}_t^{(r)}\}_{t=0}^{T-1}$ be generated according to \Algorithmref{memQSGD-synchronous} with $QComp_k$, for step sizes $\eta_t=\nicefrac{8}{\mu\left(a+t\right)}$ with $gap(\mathcal{I}_T)\leq H$, where $a>1$ is 
such that we have $a>\max\{ \nicefrac{4H}{\gamma},32\kappa,H\}$, $\kappa=\nicefrac{L}{\mu}$.
Then the following holds
\begin{align*}
&\mathbb{E}[f\left(\overline{\mathbf{x}}_T\right)]-f^* \leq \frac{L a^3}{4S_T}\|\mathbf{x}_0-\mathbf{x}^*\|_2^2+\frac{8LT\left(T+2a\right)}{\mu^2 S_T}A+\frac{128LT}{\mu^3 S_T}B. 
\end{align*}
Here \textsf{(i)} $A=\frac{\sum_{r=1}^R\sigma_r^2}{bR^2}$ ,
$B=4\left(\left(\frac{3\mu}{2}+3L\right)\frac{ CG^2H^2}{\gamma^2}+3L^2G^2H^2\right)$, where $C\geq \frac{4a\gamma(1-\gamma^2)}{a\gamma-4H}$; 
\textsf{(ii)} $\overline{\mathbf{x}}_T:=\frac{1}{S_T}\sum_{t=0}^{T-1}\left[w_t\left(\frac{1}{R}\sum_{r=1}^R\widehat{\mathbf{x}}_t^{\left(r\right)}\right)\right]$, where $w_t=\left(a+t\right)^2$;
and \textsf{(iii)} $S_T=\sum_{t=o}^{T-1}w_t\geq \frac{T^3}{3}$.
\end{theorem}
\begin{corollary}\label{cor:convergence-strongly-convex-local-cor}
For $a>\max\{ \frac{4H}{\gamma},32\kappa,H\}$, $\sigma_{max}=\max_{r\in[R]}\sigma_r$, and using
$\mathbb{E}\|\bx_0-\bx^*\|_2^2\leq \frac{4G^2}{\mu^2}$ from Lemma 2
in \cite{rakhlin}, we have 
\begin{align*}
    \mathbb{E}[f\left(\overline{\mathbf{x}}_T\right)]-f^*\leq \mathcal{O}\left(\frac{G^2H^3}{\mu^2\gamma^3T^3}\right)+\mathcal{O}\left(\frac{\sigma_{max}^2}{\mu^2 bRT}+\frac{H\sigma_{max}^2}{\mu^2 bR\gamma T^2}\right)+\mathcal{O}\left(\frac{G^2H^2}{\mu^3\gamma^2T^2}\right).
\end{align*}
In order to ensure that the compression does not affect the dominating terms while converging at a rate of $\mathcal{O}\left(1/(bRT)\right)$, we would require $H=\mathcal{O}\left(\gamma\sqrt{T/(bR)}\right)$.
\end{corollary}
\Theoremref{convergence-strongly-convex-local} is proved in \Appendixref{convex_proof_sync}. For no compression and only local computations, i.e., for $\gamma=1$, and under the same assumptions, we recover/generalize a few recent results from literature with similar convergence rates:
\begin{enumerate}
\item We recover \cite[Theorem 1]{alibaba_local}, which does local SGD for the non-convex case;
\item We generalize \cite[Theorem 2.2]{localsgd2}, which does local SGD for a strongly convex case and requires the unbiasedness assumption of gradients,\footnote{The unbiasedness of gradients at every worker can be ensured by assuming that each worker samples data points from the {\em entire} dataset.} to the distributed case.
\end{enumerate}
We emphasize that unlike \cite{alibaba_local,localsgd2}, which only consider local computation, we combine quantization and sparsification with local computation, which poses several technical challenges; e.g., see proofs of \Lemmaref{bounded-memory-decaying-lr}, \ref{lem:bounded-memory-fixed-lr}, \ref{lem:memory-maintenance}.

\subsection{Proof Outlines}\label{sec:pf_outline}
In order to prove our results, we define virtual sequences for every worker $r\in[R]$ and for all $t\geq 0$ as follows:
\begin{align}
\widetilde{\mathbf{x}}_0^{(r)} := \widehat{\mathbf{x}}_0^{(r)} \quad\quad\textrm{and}\quad\quad\widetilde{\mathbf{x}}_{t+1}^{\left(r\right)} := \widetilde{\mathbf{x}}_{t}^{\left(r\right)}-\eta_t\nabla f_{i_t^{(r)}}\left(\widehat{\mathbf{x}}_t^{\left(r\right)}\right)\label{eq:virtual_seq_defn}
\end{align}
Here $\eta_t$ can be taken to be decaying or fixed, depending on the result that we are proving. 
Let $i_t$ be the set of random sampling of the mini-batches at each worker $\{i_t^{(1)},i_t^{(2)},\ldots,i_t^{(R)}\}$. We define
\begin{enumerate}
\item $\mathbf{p}_t:=\frac{1}{R}\sum_{r=1}^R\nabla f_{i_t^{(r)}}\left(\widehat{\mathbf{x}}_t^{\left(r\right)}\right)$, $\quad\overline{\mathbf{p}}_t:=\mathbb{E}_{i_t}[\mathbf{p}_t]=\frac{1}{R}\sum_{r=1}^R\nabla f^{(r)}\left(\widehat{\mathbf{x}}_t^{\left(r\right)}\right)$; 
\item $\widetilde{\mathbf{x}}_{t+1}:=\frac{1}{R}\sum_{r=1}^R\widetilde{\mathbf{x}}_{t+1}^{\left(r\right)}=\widetilde{\mathbf{x}}_t-\eta_t\mathbf{p}_t$, $\quad\widehat{\mathbf{x}}_t:=\frac{1}{R}\sum_{r=1}^R\widehat{\mathbf{x}}_t^{\left(r\right)}$.
\end{enumerate}

\subsubsection{Proof Outline of \Theoremref{convergence-non-convex-fixed-local-het} }
\begin{proof}
Since $f$ is $L$-smooth, we have from \eqref{eq:virtual_seq_defn} (with fixed learning rate $\eta_t=\eta$) that 
\begin{align}
    f(\widetilde{\mathbf{x}}_{t+1})-f(\widetilde{\mathbf{x}}_{t})&\leq -\eta\langle\nabla f(\widetilde{\mathbf{x}}_t),\mathbf{p}_t\rangle+\frac{\eta^2L}{2}\|\mathbf{p}_t\|_2^2.
    \end{align}
    With some algebraic manipulations provided in \Appendixref{smooth_proof_sync}, 
for $\eta\leq \nicefrac{1}{2L}$, 
we arrive at 
\begin{align}
    \frac{\eta}{4R}\sum_{r=1}^R\mathbb{E}\|\nabla f(\widehat{\mathbf{x}}_t^{(r)})\|_2^2\ &\leq\ \left(\mathbb{E}[f(\widetilde{\mathbf{x}}_t)]-\mathbb{E}[f(\widetilde{\mathbf{x}}_{t+1})]\right)+\eta^2L\mathbb{E}\|\mathbf{p}_t-\overline{\mathbf{p}}_t\|_2^2+2\eta L^2\mathbb{E}\|\widetilde{\mathbf{x}}_t-\widehat{\mathbf{x}}_t\|_2^2\notag\\
    &\hspace{2cm}+2\eta L^2\frac{1}{R}\sum_{r=1}^R\mathbb{E}\|\widehat{\mathbf{x}}_t-\widehat{\mathbf{x}}_t^{(r)}\|_2^2. \label{temp1}
\end{align}
Under the Assumption 2, stated in \Subsectionref{assumptions}, we have 
\begin{align}
\mathbb{E}\|\mathbf{p}_t-\overline{\mathbf{p}}_t\|_2^2\ \leq\ \frac{\sum_{r=1}^R\sigma_r^2}{bR^2}. \label{var_bound}
\end{align}
\noindent To bound $\mathbb{E}\|\btx_t-\bwx_t\|_2^2$ on the RHS of \eqref{temp1}, we first show below that $\bwx_t-\btx_t = \frac{1}{R}\sum_{r=1}^Rm_t^{(r)}$, i.e., the difference of the true and the virtual sequence is equal to the average memory; and then we can use the bound on the local memory terms from \Lemmaref{bounded-memory-fixed-lr}.
\begin{lemma}[Memory]\label{lem:memory-maintenance}
Let $\btx_t^{(r)},m_t^{(r)}$, $r\in[R]$, $t\geq0$ be generated according to \Algorithmref{memQSGD-synchronous} and let $\bwx_t^{(r)}$ be as defined in \eqref{eq:virtual_seq_defn}. Let $\btx_t=\frac{1}{R}\sum_{r=1}^R\btx_t^{(r)}$ and $\bwx_t=\frac{1}{R}\sum_{r=1}^R\bwx_t^{(r)}$. 
Then we have 
\begin{align*}
\widehat{\mathbf{x}}_t-\widetilde{\mathbf{x}}_t\ =\ \frac{1}{R}\sum_{r=1}^Rm_t^{(r)},
\end{align*}
i.e., the difference of the true and the virtual sequence is equal to the average memory.
\end{lemma}
\noindent A proof of \Lemmaref{memory-maintenance} is provided in \Appendixref{mem_main}. 
Since $\mathbb{E}\|\btx_t-\bwx_t\|_2^2\leq\frac{1}{R}\sum_{r=1}^R\mathbb{E}\|m_t^{(r)}\|_2^2$, by using \Lemmaref{bounded-memory-fixed-lr} to bound the local memory terms $\mathbb{E}\|m_t^{(r)}\|_2^2$, we get 
\begin{align}\label{eq:true-minus-virtual-sync}
\mathbb{E}\|\widetilde{\mathbf{x}}_t-\widehat{\mathbf{x}}_t\|_2^2\ \leq\ 4\frac{\eta^2(1-\gamma^2)}{\gamma^2}H^2G^2.
\end{align}
The last term on the RHS of \eqref{temp1} depicts the deviation of the local sequences $\btx_t^{(r)}$ from the global sequence $\btx_t$ which can be bounded as shown in \Lemmaref{sequence-bound-synchronous-fixed-lr}. The details are provided in \Appendixref{dev_fix}.
\begin{lemma}[Bounded deviation of local sequences]\label{lem:sequence-bound-synchronous-fixed-lr}
Let $gap(\mathcal{I}_T) \leq H$.
For $\widehat{\bx}_t^{(r)}$ generated according to \Algorithmref{memQSGD-synchronous} with a fixed learning rate $\eta$ and letting $\bwx_t=\frac{1}{R}\sum_{r=1}^R\bwx_t^{(r)}$, we have the following bound on the deviation of the local sequences:
\begin{align}
    \frac{1}{R}\sum_{r=1}^R\mathbb{E}\|\widehat{\mathbf{x}}_t-\widehat{\mathbf{x}}_t^{\left(r\right)}\|_2^2\ \leq\ \eta^2G^2H^2.\label{loc100}
\end{align}
\end{lemma}
\noindent Substituting the bounds from \eqref{var_bound}-\eqref{loc100} into \eqref{temp1} yields
\begin{align}
    \frac{\eta}{4R}\sum_{r=1}^R\mathbb{E}\|\nabla f(\widehat{\mathbf{x}}_t^{(r)})\|_2^2\ &\leq\ \mathbb{E}[f(\widetilde{\mathbf{x}}_t)]-\mathbb{E}[f(\widetilde{\mathbf{x}}_{t+1})]+\frac{\eta^2L}{bR^2}\sum_{r=1}^R\sigma_r^2+8\frac{\eta^3(1-\gamma^2)}{\gamma^2} L^2G^2H^2\notag\\
    &\hspace{2cm}+2\eta^3 L^2G^2H^2.
\end{align}
Performing a telescopic sum from $t=0$ to $T-1$ and dividing by $\frac{\eta T}{4}$ gives
\begin{align}
    \frac{1}{RT}\sum_{t=0}^{T-1}\sum_{r=1}^R\mathbb{E}\|\nabla f(\widehat{\mathbf{x}}_t^{(r)})\|_2^2\ &\leq\ \frac{4\left(\mathbb{E}[f(\widetilde{\mathbf{x}}_0)]-f^*\right)}{\eta T}+\frac{4\eta L}{bR^2}\sum_{r=1}^R\sigma_r^2+32\frac{\eta^2(1-\gamma^2)}{\gamma^2} L^2G^2H^2\notag\\
    &\hspace{2cm}+8\eta^2 L^2G^2H^2.
\end{align}
By letting $\eta=\widehat{C}/\sqrt{T}$, where $\widehat{C}$ is a constant such that $\frac{\widehat{C}}{\sqrt{T}}\leq\frac{1}{2L}$, we arrive at bound stated in \Theoremref{convergence-non-convex-fixed-local-het}.
\end{proof}

\subsubsection{Proof Outline of \Theoremref{convergence-non-convex-decay-local-het}}
\begin{proof}
Observe that \eqref{temp1} holds irrespective of the learning rate schedule, as long as learning rate is at most $\nicefrac{1}{2L}$; see \Appendixref{proof_convergence-non-convex-decay-local-het} for details.
Here $\eta_t\leq \frac{1}{2L}$ follows from our assumption that $a\geq 2\xi L$.
Substituting a decaying learning rate $\eta_t$ (such that $\eta_t\leq\nicefrac{1}{2L}$ holds for every $t\geq 0$) in \eqref{temp1} gives
\begin{align}
    \frac{\eta_t}{4R}\sum_{r=1}^R\mathbb{E}\|\nabla f(\widehat{\mathbf{x}}_t^{(r)})\|_2^2\ &\leq\ \left(\mathbb{E}[f(\widetilde{\mathbf{x}}_t)]-\mathbb{E}[f(\widetilde{\mathbf{x}}_{t+1})]\right)+\eta_t^2L\mathbb{E}\|\mathbf{p}_t-\overline{\mathbf{p}}_t\|_2^2+2\eta_t L^2\mathbb{E}\|\widetilde{\mathbf{x}}_t-\widehat{\mathbf{x}}_t\|_2^2\notag\\
    &\hspace{2cm}+2\eta_t L^2\frac{1}{R}\sum_{r=1}^R\mathbb{E}\|\widehat{\mathbf{x}}_t-\widehat{\mathbf{x}}_t^{(r)}\|_2^2. \label{temp111}
\end{align}
We have already bounded $\mathbb{E}\|\mathbf{p}_t-\overline{\mathbf{p}}_t\|_2^2\ \leq\ \frac{\sum_{r=1}^R\sigma_r^2}{bR^2}$ in \eqref{var_bound}.
Note that \Lemmaref{memory-maintenance} holds irrespective of the learning rate schedule, and together with \Lemmaref{bounded-memory-decaying-lr}, we can show that 
\begin{equation}\label{contracting-difference-true-virtual-decay-lr}
\mathbb{E}\|\widehat{\bx}_t-\btx_t\|^2\ \leq\ C\frac{4\eta_{t}^2}{\gamma^2}G^2H^2.
\end{equation}
The last term on the RHS of \eqref{temp111} is the deviation of local sequences and we bound it in \Lemmaref{sequence-bound-synchronous} for decaying learning rates. The details are provided in \Appendixref{dev_dec}.
\begin{lemma}[Contracting deviation of local sequences]\label{lem:sequence-bound-synchronous}
Let $gap(\mathcal{I}_T) \leq H$.
By running \Algorithmref{memQSGD-synchronous} with a decaying learning rate $\eta_t$, we have
\begin{align}
    \frac{1}{R}\sum_{r=1}^R\mathbb{E}\|\widehat{\mathbf{x}}_t-\widehat{\mathbf{x}}_t^{\left(r\right)}\|_2^2\ \leq\ 4\eta_t^2G^2H^2.\label{loc1023}
\end{align}
\end{lemma}
Observe that for the case of fixed learning rate, we can trivially bound $\frac{1}{R}\sum_{r=1}^R\mathbb{E}\|\btx_t-\btx_t^{(r)}\|_2^2$ by simply putting $\eta_t=\eta$ in \eqref{loc1023}. However, in \eqref{loc100}, we can get a slightly better bound (without the factor of 4) by directly working with a fixed learning rate.
Using these bounds in \eqref{temp111} gives
\begin{align*}
    \frac{\eta_t}{4R}\sum_{r=1}^R\mathbb{E}\|\nabla f(\widehat{\bx}_t^{(r)})\|^2\ &\leq\ \mathbb{E}[f(\btx_t)]-\mathbb{E}[f(\btx_{t+1})]+\frac{\eta_t^2L}{bR^2}\sum_{r=1}^R\sigma_r^2+\frac{8\eta_{t}^3}{\gamma^2}CL^2G^2H^2+8\eta_t^3L^2G^2H^2.
\end{align*}
Let $\delta_t:=\frac{\eta_t}{4R}$ and $P_T:=\sum_{t=0}^{T-1}\sum_{r=1}^R\delta_t$. Performing a telescopic sum from $t=0$ to $T-1$ and dividing by $P_T$ gives 
\begin{align}
    \frac{1}{P_{T}}\sum_{t=0}^{T-1}\sum_{r=1}^R\delta_t\mathbb{E}\|\nabla f(\widehat{\bx}_t^{(r)})\|^2\
    &\leq\ \frac{\mathbb{E}f(\bx_0)-f^*}{P_T}+\frac{L\xi^2}{bR^2(a-1)}\frac{\sum_{r=1}^R\sigma_r^2}{P_T}\notag\\
    &\hspace{2cm}+\left(\frac{8C}{\gamma^2}+8\right)L^2G^2H^2\frac{\xi^3}{2P_T(a-1)^2} \label{eq:nc_5}
\end{align}
In \eqref{eq:nc_5}, we used the following bounds, which are shown in \Appendixref{proof_convergence-non-convex-decay-local-het}:
$P_T\geq \frac{\xi}{4}\ln{\left(\frac{T+a-1}{a}\right)}$, $\sum_{t=0}^{T-1}\eta_t^2 \leq \frac{\xi^2}{a-1}$, and $\sum_{t=0}^{T-1}\eta_t^3\leq \frac{\xi^3}{2(a-1)^2}$.
This completes the proof of \Theoremref{convergence-non-convex-decay-local-het}.
\end{proof}

\subsubsection{Proof Outline of \Theoremref{convergence-strongly-convex-local}}
\begin{proof}
Using the definition of virtual sequences \eqref{eq:virtual_seq_defn} that, we have 
\begin{align}
    \|\widetilde{\mathbf{x}}_{t+1}-\mathbf{x}^*\|_2^2\ &=\ \|\widetilde{\mathbf{x}}_t-\mathbf{x}^*-\eta_t\overline{\mathbf{p}}_t\|_2^2+\eta_t^2\|\mathbf{p}_t-\overline{\mathbf{p}}_t\|_2^2-2\eta_t\left\langle\widetilde{\mathbf{x}}_t-\mathbf{x}^*-\eta_t\overline{\mathbf{p}}_t,\mathbf{p}_t-\overline{\mathbf{p}}_t\right\rangle.
\end{align}
Note that $\eta_t\leq \nicefrac{1}{4L}$, which follows from the assumption that $a>\frac{32L}{\mu}$.
Now, using $\mu$-strong convexity and $L$-smoothness of $f$, together with some algebraic manipulations provided in \Appendixref{convex_proof_sync},
by letting $e_t=\mathbb{E}[f(\widehat{\bx}_t)]-f^*$, 
we arrive at 
\begin{align}
\mathbb{E}\|\widetilde{\mathbf{x}}_{t+1}-\mathbf{x}^*\|_2^2\ &\leq\ \left(1-\frac{\mu\eta_t}{2}\right)\mathbb{E}\|\widetilde{\mathbf{x}}_t-\mathbf{x}^*\|_2^2-\frac{\eta_t\mu}{2L}e_t +\eta_t\left(\frac{3\mu}{2}+3L\right)\mathbb{E}\|\widehat{\mathbf{x}}_t-\widetilde{\mathbf{x}}_t\|_2^2\notag\\
&\hspace{2cm}+\frac{3\eta_tL}{R}\sum_{r=1}^R\mathbb{E}\|\widehat{\mathbf{x}}_t-\widehat{\mathbf{x}}_t^{\left(r\right)}\|_2^2+\eta_t^2\frac{\sum_{r=1}^R\sigma_r^2}{bR^2}\label{loc8}.
\end{align}
Note that the bounds in \eqref{contracting-difference-true-virtual-decay-lr} and \eqref{loc1023} hold irrespective of whether the function is convex or not. So, we can use them here as well in \eqref{loc8}, which gives
\begin{align}
    \mathbb{E}\|\widetilde{\mathbf{x}}_{t+1}-\mathbf{x}^*\|_2^2\ &\leq\  \left(1-\frac{\mu\eta_t}{2}\right)\mathbb{E}\|\widetilde{\mathbf{x}}_t-\mathbf{x}^*\|_2^2-\frac{\mu\eta_t}{2L}e_t+\eta_t\left(\frac{3\mu}{2}+3L\right)C\frac{4\eta_{t}^2}{\gamma^2}G^2H^2 \notag\\
    &\hspace{2cm}+ (3\eta_tL)4\eta_t^2LG^2H^2 +\eta_t^2\frac{\sum_{r=1}^R\sigma_r^2}{bR^2}. \label{eq:loc900}
\end{align}
Employing a slightly modified result than \cite[Lemma 3.3]{memSGD} with $a_{t}=\mathbb{E}\|\widetilde{\mathbf{x}}_t-\mathbf{x}^*\|_2^2$, $A=\frac{\sum_{r=1}^R\sigma_r^2}{bR^2}$ and $B=4\left(\left(\frac{3\mu}{2}+3L\right)\frac{ CG^2H^2}{\gamma^2}+3L^2G^2H^2\right)$, we have 
\begin{align}
    a_{t+1}\ \leq\ \left(1-\frac{\mu\eta_t}{2}\right)a_t-\frac{\mu\eta_t}{2L}e_t+\eta_t^2A+\eta_t^3B.
\end{align}
For $\eta_t=\frac{8}{\mu\left(a+t\right)}$ and $w_t=\left(a+t\right)^2$, $S_T=\sum_{t=o}^{T-1}\geq \frac{T^3}{3}$, we have 
\begin{align}
    \frac{\mu}{2LS_T}\sum_{t=0}^{T-1}w_te_t\ \leq\  \frac{\mu a^3}{8S_T}a_0+\frac{4T\left(T+2a\right)}{\mu S_T}A+\frac{64T}{\mu^2 S_T}B.
\end{align}
From convexity, we can finally write 
\begin{align}
    \mathbb{E}f\left(\overline{\mathbf{x}}_T\right)-f^*\ \leq\ \frac{L a^3}{4S_T}a_0+\frac{8LT\left(T+2a\right)}{\mu^2 S_T}A+\frac{128LT}{\mu^3 S_T}B.
\end{align}
Where $\overline{\mathbf{x}}_T:=\frac{1}{S_T}\sum_{t=0}^{T-1}\left[w_t\left(\frac{1}{R}\sum_{r=1}^R\widehat{\mathbf{x}}_t^{\left(r\right)}\right)\right]=\frac{1}{S_T}\sum_{t=0}^{T-1}w_t\widehat{\mathbf{x}}_t$. This completes the proof of \Theoremref{convergence-strongly-convex-local}.
\end{proof}

\section{Distributed Asynchronous Operation}
\label{sec:loc-async}
\label{sec:asynch_algo}
We propose and analyze a particular form of asynchronous operation, where the workers synchronize with the master at arbitrary times decided locally or by master picking a subset of nodes as in federated learning \cite{KonecnyThesis,McMahan16}. However, the local iterates evolve at the same rate, i.e., each worker takes the same number of steps per unit time according to a global clock. The asynchrony is therefore that updates occur after different number of local iterations but the local iterations are in synchrony with respect to the global clock. This is different from asynchronous algorithms studied for stragglers \cite{yin_sayed,hogwild_recht}, where only one gradient step is taken but occurs at different times due to delays.

In this asynchronous setting, $\I_T^{(r)}$'s may be different for different workers. However, we assume that $gap(\I_T^{(r)})\leq H$ holds for every $r\in[R]$, which means that there is a uniform bound on the maximum delay in each worker's update times.
The algorithmic difference from \Algorithmref{memQSGD-synchronous} is that, in this case, {\em a subset of} workers (including a single worker) can send their updates to the master at their synchronization time steps; master aggregates them, updates the global parameter vector, and sends that only to those workers.
Our algorithm is summarized in \Algorithmref{memQSGD-asynchronous} 

\begin{algorithm}[h]
   \caption{Qsparse-local-SGD with asynchronous updates}
   \label{algo:memQSGD-asynchronous}
\begin{algorithmic}[1]
   \STATE Initialize $\mathbf{x}_0=\bar{\bar{\bx}}_0=\bx_0^{(r)}=\widehat{\mathbf{x}}_0^{\left(r\right)}=m_0^{\left(r\right)}=\bzero,\,\,\forall r\in [R]$. Suppose $\eta_t$ follows a certain learning rate schedule.
   \FOR{$t=0$ {\bfseries to} $T-1$}
  \STATE \textbf{On Workers:}
   \FOR{$r=1$ {\bfseries to} $R$}
   \STATE $\widehat{\mathbf{x}}_{t+\frac{1}{2}}^{\left(r\right)}\leftarrow\widehat{\mathbf{x}}_t^{\left(r\right)}-\eta_t\nabla f_{i_t^{(r)}}\left(\widehat{\mathbf{x}}_t^{\left(r\right)}\right)$; $i_t^{(r)}$ is a mini-batch of size $b$ uniformly in $\mathcal{D}_r$
   
   \IF{$t+1\notin\mathcal{I}_T^{(r)}$}
   \STATE $\mathbf{x}_{t+1}^{(r)}\leftarrow \mathbf{x}_{t}^{(r)}$, $m_{t+1}^{\left(r\right)}\leftarrow m_{t}^{\left(r\right)}$ and $\widehat{\mathbf{x}}_{t+1}^{(r)}\leftarrow\widehat{\mathbf{x}}_{t+\frac{1}{2}}^{(r)}$
      \ELSE
   \STATE $g_t^{\left(r\right)}\leftarrow Q\,Comp_k\left(m_t^{\left(r\right)}+\mathbf{x}_t^{(r)}-\widehat{\mathbf{x}}_{t+\frac{1}{2}}^{\left(r\right)}\right)$ and send $g_t^{\left(r\right)}$ to the master
   \STATE $m_{t+1}^{\left(r\right)}\leftarrow m_t^{\left(r\right)}+\mathbf{x}_t^{(r)}-\widehat{\mathbf{x}}_{t+\frac{1}{2}}^{\left(r\right)}-g_t^{\left(r\right)}$
   \STATE Receive $\Bar{\Bar{\bx}}_{t+1}$ from the master and set $\bx_{t+1}^{(r)}\leftarrow\Bar{\Bar{\bx}}_{t+1}$ and $\widehat{\mathbf{x}}_{t+1}^{\left(r\right)}\leftarrow \Bar{\Bar{\mathbf{x}}}_{t+1}$
      \ENDIF
   
   \ENDFOR

    \STATE \textbf{At Master:}
   \IF{$t+1\notin\mathcal{I}_T^{(r)}$ for all $r\in[R]$}
   \STATE $\bar{\bar{\bx}}_{t+1}\leftarrow \bar{\bar{\bx}}_t$
      \ELSE
   \STATE Let $\calS\subseteq[R]$ be the set of all workers $r$ such that master receives $g_t^{\left(r\right)}$ from $r$
   \STATE Compute $\Bar{\Bar{\bx}}_{t+1}\leftarrow  \Bar{\Bar{\bx}}_t-\frac{1}{R}\sum_{r\in\calS}g_t^{\left(r\right)}$ and broadcast $\Bar{\Bar{\bx}}_{t+1}$ to all the workers in $\calS$ \label{lst:update}
  \ENDIF
      
   \ENDFOR
\end{algorithmic}
\end{algorithm}
\subsection{Main Results}
In this section we present our main convergence results with asynchronous updates, obtained by running \Algorithmref{memQSGD-asynchronous} for smooth objectives, both non-convex and strongly convex. 
Under the same assumptions as in the synchronous setting of \Subsectionref{assumptions}, the following theorems hold
even if \Algorithmref{memQSGD-asynchronous} is run with an arbitrary compression operators (including our composed operators from \Subsectionref{composed-operators}), whose
compression coefficient is equal to $\gamma$.
\begin{theorem}[Smooth (non-convex) case with fixed learning rate]\label{thm:async_convergence-non-convex-local-het2}
Under the same conditions as in \Theoremref{convergence-non-convex-fixed-local-het} with $gap(\mathcal{I}_T^{(r)})\leq H$, if $\{\widehat{x}_t^{(r)}\}_{t=0}^{T-1}$ is generated according to \Algorithmref{memQSGD-asynchronous}, the following holds.
\begin{align*}
    \mathbb{E}\|\nabla f(\mathbf{z}_T)\|_2^2&\leq \left(\frac{\mathbb{E}[f(\bx_0)]-f^*}{\widehat{C}}+\widehat{C} L\left(\frac{\sum_{r=1}^R\sigma_r^2}{bR^2}\right)\right)\frac{4}{\sqrt{T}}\notag\\
    &\hspace{2cm}+8\left(12\frac{(1-\gamma^2)}{\gamma^2}+(2+8C_1H^2)\right) \frac{\widehat{C}^2L^2G^2H^2}{T}.
\end{align*}
Here {\sf (i)} $C_1=(\frac{8}{\gamma^2}-6)(4-2\gamma)$; \textsf{(ii)} $\mathbf{z}_T$ is a random variable which samples a previous parameter $\widehat{\mathbf{x}}_t^{(r)}$ with probability $1/RT$; and \textsf{(iii)} $\widehat{C}$ is a constant such that $\frac{\widehat{C}}{\sqrt{T}}\leq \frac{1}{2L}$.
\end{theorem}
\begin{corollary}
Let $\mathbb{E}[f(\bx_0)]-f^*\leq J^2$, where $J<\infty$ is a constant, $\sigma_{max}=\max_{r\in[R]}\sigma_r$, and $\widehat{C}^2=\nicefrac{bR(\mathbb{E}[f(\bx_0)]-f^*)}{\sigma_{max}^2L}$. We can get a simplified expression below
\begin{align*}
    \mathbb{E}\|\nabla f(\mathbf{z}_T)\|_2^2\leq\mathcal{O}\left(\frac{J\sigma_{max}}{\sqrt{bRT}}\right)+\mathcal{O}\left(\frac{J^2bRG^2}{\sigma_{max}^2\gamma^2T}(H^2+H^4)\right).
\end{align*}
In order to ensure that the compression does not affect the dominating terms while converging at a rate of $\mathcal{O}\left(1/\sqrt{bRT}\right)$, we would require $H=\mathcal{O}\left(\sqrt{\gamma} T^{1/8}/(bR)^{3/8}\right)$.
\end{corollary}

\Theoremref{async_convergence-non-convex-local-het2} provides non asymptotic guarantees where we also observe that the compression comes for ``free". The corresponding asymptotic result is given below.
\begin{theorem}[Smooth (non-convex) case with decaying learning rate]\label{thm:convergence-non-convex-local-het}
Under the same conditions as in \Theoremref{convergence-non-convex-decay-local-het} with $gap(\mathcal{I}_T^{(r)})\leq H$,
if $\{\widehat{x}_t^{(r)}\}_{t=0}^{T-1}$ is generated according to \Algorithmref{memQSGD-asynchronous}, the following
holds.
\begin{align*}
    \mathbb{E}\|\nabla f(\mathbf{z}_T)\|^2&\leq \frac{\mathbb{E}f(\mathbf{x}_0)-f^*}{P_T}+\frac{L\xi^2}{(a-1)P_T}\left(\frac{\sum_{r=1}^R\sigma_r^2}{bR^2}\right)
    +\left(16+\frac{24C}{\gamma^2}+200C'H^2\right)\frac{\xi^3L^2G^2H^2}{2(a-1)^2P_T}
\end{align*}
Here {\sf (i)} $\delta_t:=\frac{\eta_t}{4R}$ and $P_{T}:=\sum_{t=0}^{T-1}\sum_{r=1}^R\delta_t$, which is lower bounded as
$P_{T}\geq \frac{\xi}{4}\ln{\left(\frac{T+a-1}{a}\right)}$;
{\sf (ii)} $C'=(4-2\gamma)(1+\frac{C}{\gamma^2})$;
and {\sf (iii)} $\mathbf{z}_T$ is a random variable which samples a previous parameter $\widehat{\mathbf{x}}_t^{(r)}$ with probability $\delta_t/P_{T}$.
\end{theorem}
\begin{theorem}[Smooth and strongly convex case with decaying learning rate]\label{thm:async_convergence-strongly-convex-local-het}
Under the same conditions as in \Theoremref{convergence-strongly-convex-local} with $gap(\mathcal{I}_T^{(r)})\leq H$,
if $\{\widehat{x}_t^{(r)}\}_{t=0}^{T-1}$ is generated according to \Algorithmref{memQSGD-asynchronous}, the following
holds.
\begin{align*}
&\mathbb{E}[f\left(\overline{\mathbf{x}}_T\right)]-f^* \leq \frac{L a^3}{4S_T}\|\mathbf{x}_0-\mathbf{x}^*\|_2^2+\frac{8LT\left(T+2a\right)}{\mu^2 S_T}A+\frac{128LT}{\mu^3 S_T}D 
\end{align*}
Here \textsf{(i)} $C\geq \frac{4a\gamma(1-\gamma^2)}{a\gamma-4H}$, $C_1=192(4-2\gamma)\left(1+\frac{C}{\gamma^2}\right)$, $C_2=8(4-2\gamma)(1+\frac{C}{\gamma^2})$; 
\textsf{(ii)} $A=\frac{\sum_{r=1}^R\sigma_r^2}{bR^2}$,
$D=\left(\frac{3\mu}{2}+3L\right)(\frac{12C G^2H^2}{\gamma^2}+C_1\eta_{t}^2H^4G^2)+24(1+C_2H^2)LG^2H^2$; and
\textsf{(iii)} $\overline{\mathbf{x}}_T$, $S_T$ are as defined in \Theoremref{convergence-strongly-convex-local}.
\end{theorem}

\begin{corollary}
Under the same conditions as in \Theoremref{convergence-strongly-convex-local} with $gap(\mathcal{I}_T^{(r)})\leq H$, 
$a>\max\{ \frac{4H}{\gamma},32\kappa,H\}$, $\sigma_{max}=\max_{r\in[R]}\sigma_r$, 
if $\{\widehat{\bx}_t^{(r)}\}_{t=0}^{T-1}$ is generated according to \Algorithmref{memQSGD-asynchronous}, the following
holds:
\begin{align*}
    \mathbb{E}[f\left(\overline{\mathbf{x}}_T\right)]-f^*\leq \mathcal{O}\left(\frac{G^2H^3}{\mu^2 \gamma^3T^3}\right)+\mathcal{O}\left(\frac{\sigma_{max}^2}{\mu^2 bRT}+\frac{H\sigma_{max}^2}{\mu^2 bR\gamma T^2}\right)+\mathcal{O}\left(\frac{G^2}{\mu^3\gamma^2T^2}(H^2+H^4)\right),
\end{align*}
where $\overline{\mathbf{x}}_T$, $S_T$ are as defined in \Theoremref{convergence-strongly-convex-local}.
In order to ensure that the compression does not affect the dominating terms while converging at a rate of $\mathcal{O}\left(1/(bRT)\right)$, we would require $H=\mathcal{O}\left(\sqrt{\gamma}(T/(bR))^{1/4}\right)$.
\end{corollary}

\subsection{Proof Outlines}
\label{sec:proof_asynch}
Our proofs of these results follow the same outlines of the corresponding proofs in the synchronous setting, but some technical details change significantly, which arise because, in our asynchronous setting, workers are allowed to update the global parameter vector in between two consecutive synchronization time steps of other workers. 
Specifically, in the asynchronous setting, we have to bound the deviation of local sequences $\frac{1}{R}\sum_{r=1}^R\mathbb{E}\|\btx_t-\btx_t^{(r)}\|_2^2$ and the difference between the virtual and true sequences $\mathbb{E}\|\btx_t-\bwx_t\|_2^2$, both with a fixed learning rate as well as with decaying learning rate.
We show these below in \Lemmaref{sequence-bound-asynchronous}-\ref{lem:sequence-bound-asynchronous-fixed-lr} and \Lemmaref{bounded-memory-decaying-lr-asynchronous}-\ref{lem:bounded-memory-fixed-lr-asynchronous}.
\begin{lemma}[Contracting local sequence deviation]\label{lem:sequence-bound-asynchronous}
Let $gap(\mathcal{I}_T^{(r)}) \leq H$ holds for every $r\in[R]$.
For $\widehat{\bx}_t^{(r)}$ generated according to \Algorithmref{memQSGD-asynchronous} with decaying learning rate $\eta_t$ and letting $\bwx_t=\frac{1}{R}\sum_{r=1}^R\bwx_t^{(r)}$, we have the following bound on the deviation of the local sequences:
\begin{align*}
    \frac{1}{R}\sum_{r=1}^R\mathbb{E}\|\widehat{\bx}_t-\widehat{\bx}_t^{\left(r\right)}\|_2^2\ \leq\ 8(1+C''H^2)\eta_t^2G^2H^2,
\end{align*}
where $C''=8(4-2\gamma)(1+\frac{C}{\gamma^2})$ and $C$ is a constant satisfying $C\geq\frac{4a\gamma(1-\gamma^2)}{a\gamma-4H}$.
\end{lemma}
\begin{lemma}[Bounded local sequence deviation]\label{lem:sequence-bound-asynchronous-fixed-lr}
Let $gap(\mathcal{I}_T^{(r)}) \leq H$ holds for every $r\in[R]$.
By running \Algorithmref{memQSGD-asynchronous} with fixed learning rate $\eta$, we have
\begin{align*}
    \frac{1}{R}\sum_{r=1}^R\mathbb{E}\|\widehat{\bx}_t-\widehat{\bx}_t^{\left(r\right)}\|_2^2\ \leq\ (2+H^2C')\eta^2G^2H^2,
\end{align*}
where $C'=(\frac{16}{\gamma^2}-12)(4-2\gamma)$.
\end{lemma}
We prove these above two lemmas in \Appendixref{asy1} and \Appendixref{asy2}, respectively.
Note that the bound in \Lemmaref{sequence-bound-asynchronous} is $\frac{1}{R}\sum_{r=1}^R\mathbb{E}\|\widehat{\mathbf{x}}_t-\widehat{\mathbf{x}}_t^{\left(r\right)}\|_2^2 \leq \mathcal{O}(\eta_t^2G^2(H^2+\nicefrac{H^4}{\gamma^2})$, which is weaker than the corresponding bound $\mathcal{O}(\eta_t^2G^2H^2)$ for the synchronous setting in \Lemmaref{sequence-bound-synchronous}. 
See \Lemmaref{sequence-bound-asynchronous-fixed-lr} and \Lemmaref{sequence-bound-synchronous-fixed-lr} for a similar comparison for the case of fixed learning rate. \\

Now we bound $\mathbb{E}\|\btx_t-\bwx_t\|_2^2$.
Fix a time $t$ and consider any worker $r\in[R]$. Let $t_r\in\I_T^{(r)}$ denote the last synchronization step until time $t$ for the $r$'th worker. Define $t_0':=\min_{r\in [R]}t_r$. We want to bound $\bbE\|\bwx_t-\btx_t\|_2^2$. 
Note that in the synchronous case, we have shown in \Lemmaref{memory-maintenance} that $\bwx_t-\bwx_t=\frac{1}{R}\sum_{r=1}^Rm_t^{(r)}$. This does not hold in the asynchronous setting, which makes upper-bounding $\bbE\|\bwx_t-\btx_t\|_2^2$ a bit more involved.
By definition $\widehat{\mathbf{x}}_t-\widetilde{\mathbf{x}}_t=\frac{1}{R}\sum_{r=1}^R\left(\widehat{\mathbf{x}}_t^{(r)}-\widetilde{\mathbf{x}}_t^{(r)}\right)$. By the definition of virtual sequences and the update rule for $\bwx_t^{(r)}$, 
we also have $\widehat{\bx}_t-\btx_t =\frac{1}{R}\sum_{r=1}^R\left(\widehat{\bx}_{t_r}^{(r)}-\btx_{t_r}^{(r)}\right)$.
This can be written as
\begin{align}
\widehat{\bx}_t-\btx_t\ &=\ \left[\frac{1}{R}\sum_{r=1}^R\widehat{\bx}_{t_r}^{(r)}-\bar{\bar{\bx}}_{t_0'}\right]+\left[\bar{\bar{\bx}}_{t_0'}-\bar{\bar{\bx}}_t\right]+\left[\bar{\bar{\bx}}_t-\frac{1}{R}\sum_{r=1}^R\btx_{t_r}^{(r)}\right]. \label{async-memory-bound1_1}
\end{align}
In \eqref{async-memory-bound1_1}, the third term on the RHS is equal to the average memory as shown in \eqref{eq:async_memory-equivalence} in \Appendixref{asy3}, and unlike \Lemmaref{memory-maintenance} in the synchronous setting, which states that $\widehat{\mathbf{x}}_t-\widetilde{\mathbf{x}}_t = \frac{1}{R}\sum_{r=1}^Rm_t^{(r)}$, does not hold here. 
However, we can show that $\widehat{\mathbf{x}}_t-\widetilde{\mathbf{x}}_t$ is equal to the sum of $\frac{1}{R}\sum_{r=1}^Rm_t^{(r)}$ and an additional term, which leads to potentially a weaker bound $\mathbb{E}\|\widehat{\mathbf{x}}_t-\widetilde{\mathbf{x}}_t\|_2^2 \leq \mathcal{O}\left(\nicefrac{\eta_{t}^2}{\gamma^2}G^2(H^2+H^4)\right)$, proved in \Lemmaref{bounded-memory-decaying-lr-asynchronous}-\ref{lem:bounded-memory-fixed-lr-asynchronous} in \Appendixref{asy3} and \Appendixref{asy4}, in comparison to~$\mathcal{O}\left(\nicefrac{\eta_{t}^2}{\gamma^2}G^2H^2\right)$ for the synchronous setting.
\begin{lemma}[Contracting distance between virtual and true sequence]\label{lem:bounded-memory-decaying-lr-asynchronous}
Let $gap(\mathcal{I}_T^{(r)}) \leq H$ holds for every $r\in[R]$.
If we run \Algorithmref{memQSGD-asynchronous} with a decaying learning rate $\eta_t$, then we have the following bound on the difference between the true and virtual sequences:
\begin{align*}
\mathbb{E}    \|\widehat{\bx}_t-\btx_t\|_2^2\ &\leq\ C'\eta_{t}^2H^4G^2 + 12C\frac{\eta_{t}^2}{\gamma^2}G^2H^2,
\end{align*}
where $C'=192(4-2\gamma)\left(1+\frac{C}{\gamma^2}\right)$ and $C$ is a constant satisfying $C\geq\frac{4a\gamma(1-\gamma^2)}{a\gamma-4H}$.
\end{lemma}
\begin{lemma}[Bounded distance between virtual and true sequence]\label{lem:bounded-memory-fixed-lr-asynchronous}
Let $gap(\mathcal{I}_T^{(r)}) \leq H$ holds for every $r\in[R]$.
If we run \Algorithmref{memQSGD-asynchronous} with a fixed learning rate $\eta$, we have
\begin{align*}
\mathbb{E}    \|\widehat{\bx}_t-\btx_t\|_2^2\ &\leq\ 6C'\eta^2H^4G^2+\frac{12\eta^2(1-\gamma^2)}{\gamma^2}G^2H^2,
\end{align*}
where $C'=(4-2\gamma)\left(\frac{8}{\gamma^2}-6\right)$.
\end{lemma}

\paragraph{Summary of our results.} Now we give a brief summary of our convergence results in the synchronous as well as asynchronous settings.
\begin{enumerate}
\item In the synchronous setting, \emph{Qsparse-local-SGD} asymptotically converges as fast as distributed vanilla SGD for $H=\mathcal{O}\left(\gamma T^{1/4}/(bR)^{3/4}\right)$ in the smooth and non-convex case and for $H=\mathcal{O}\left(\gamma\sqrt{T/(bR)}\right)$ in the strongly convex case.
\item In the asynchronous setting, \emph{Qsparse-local-SGD} asymptotically converges as fast as distributed vanilla SGD for $H=\mathcal{O}(\sqrt{\gamma} T^{1/8}/(bR)^{3/8})$ in the smooth and non-convex case and for $H=\mathcal{O}(\sqrt{\gamma}(T/(bR))^{1/4})$ in the strongly convex case.
\end{enumerate}
Therefore, our algorithm provides a lot of flexibility in terms of different ways of mitigating the communication bottleneck. For example, by increasing the batch size on each node, or by increasing the maximum synchronization period $H$ up to allowable limits. Furthermore, one could also choose to opt for different values of $k$ for the $\textrm{Top}_k$ sparsifier, as well as adjust the configurations of the quantizer.
We present numerics in \Sectionref{expmt} demonstrating significant savings in the number of bits exchanged over the state-of-the-art.

\section{Experimental Results}
\label{sec:expmt}
In this section we give extensive experimental results for validating our theoretical findings.
\subsection{Non-Convex Objective} \label{sec:numerics_NC} 
\subsubsection{Experiment Setup}
We train ResNet-50 \cite{resnet50} (which has $d=25,610,216$
parameters) on ImageNet dataset, using 8 NVIDIA Tesla V100 GPUs. We
use a learning rate schedule consisting of 5 epochs of linear warmup,
followed by a piecewise decay of 0.1 at epochs 30, 60 and 80, with a batch size of
256 per GPU. For the purpose of experiments, we focus on SGD with
momentum of 0.9, applied on the local iterations of the workers. We build our compression scheme into the Horovod framework
\cite{sergeev2018hvd}.
We use $SignTop_k$ as defined in \Lemmaref{composed-sign} and $QTop_k$ as defined in \Lemmaref{composed-compression} (which has an operating regime $\beta_{k,s}<1$), where $Q$ is from \cite{QSGD}, as our composed operators.\footnote{\label{foot:scaled-vs-unscaled}Even though the ``scaled'' $QTop_k$ from \Lemmaref{composed-quantizer} (with a scaling factor of $(1+\beta_{k,s})$) works with all values of $\beta_{k,s}$, and also does better than the ``unscaled'' $QTop_k$ from \Lemmaref{composed-compression} even when $\beta_{k,s}<1$ (see \Remarkref{scaled-vs-unscaled}), we report our experimental results in the non-convex setting only with unscaled $QTop_k$. We give some plots for a comparison on both these operators in \Appendixref{supp_expmts} and observe that our algorithm with the unscaled operator gives at least as good performance as it gives with the scaled operator. We can attribute this to the fact that scaling the composed operator is a sufficient condition to obtain better convergence results, which does not necessarily mean that in practice also it does better.}
In $Top_k$, we only update $k_t = \min(d_t, 1000)$ elements per step for each tensor $t$, where $d_t$ is
the number of elements in the tensor. For ResNet-50 architecture, this
amounts to updating a total of $k=99,400$ elements per step.

\begin{figure}[h]
\centering
\begin{subfigure}{0.49\textwidth}
\includegraphics[scale=0.54]{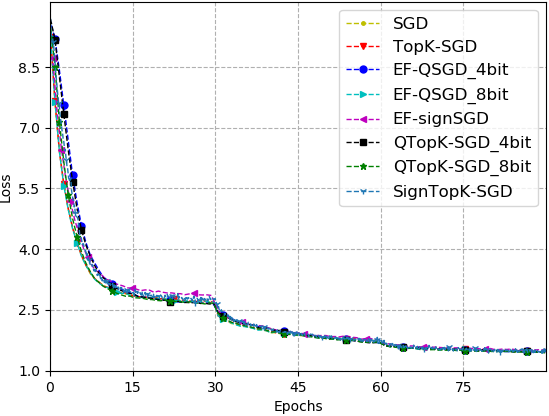}
\captionsetup{justification=centering}
\caption{Training loss vs epochs }
\label{fig:mid_lo-iter_a}
\centering
\end{subfigure}
\begin{subfigure}{0.49\textwidth}
\includegraphics[scale=0.54]{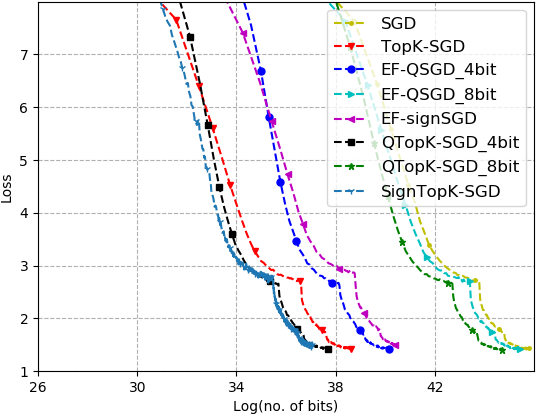}
\captionsetup{justification=centering}
\caption{Training loss vs $\textrm{log}_2$ of communication budget }
\label{fig:mid_lo-bit_a}
\centering
\end{subfigure}\\ \vspace{0.5cm}
\begin{subfigure}{0.49\textwidth}
\includegraphics[scale=0.54]{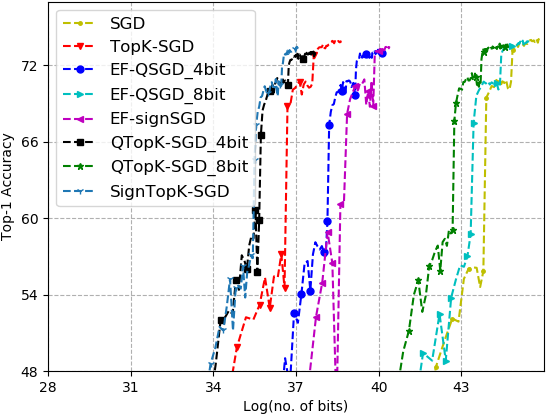}
\captionsetup{justification=centering}
\caption{top-1 accuracy \cite{lapin} for schemes in \Figureref{mid_lo-iter_a} }
\label{fig:te_1_a}
\centering
\end{subfigure}
\begin{subfigure}{0.49\textwidth}
\includegraphics[scale=0.54]{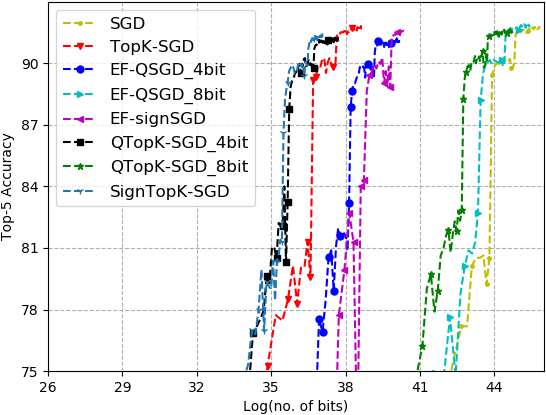}
\captionsetup{justification=centering}
\caption{top-5 accuracy \cite{lapin} for schemes in \Figureref{mid_lo-iter_a} }
\label{fig:te_5_a}
\centering
\end{subfigure}
\caption{\Figureref{mid_lo-iter_a}-\ref{fig:te_5_a} demonstrate the gains in performance achieved by our $Qsparse$ operators in the non-convex setting.}
\end{figure}

\subsubsection{Results}
From \Figureref{mid_lo-iter_a}, we observe that quantization and sparsification, both individually and combined, when
error compensation is enabled through accumulating errors, has almost
no penalty in terms of convergence rate, with respect to vanilla
SGD. We observe that both $QTop_k$-$SGD$, which employs a 4 bit quantizer and the $Top_k$ sparsifier, as well as $SignTop_k$-$SGD$, which employs the 1 bit sign quantizer and the $Top_k$ sparsifier, demonstrate superior performance over other schemes, both in terms of the required number of communicated bits for achieving certain target loss as well as test accuracy. This is because, in $QTop_k$, the $Q$ operator from \cite{QSGD} further induces sparsity, which results in fewer than $k$ coordinates being transmitted, and in $SignTop_k$, we send only 1 bit for each $Top_k$ coordinate.

\begin{figure*}[ht!]
\begin{subfigure}{0.49\textwidth}
\centering
\includegraphics[scale=0.55]{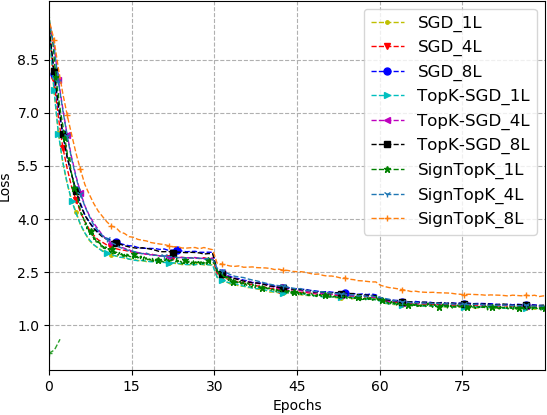}
\captionsetup{justification=centering}
\caption{Training loss vs epochs }
\label{fig:mid_lo-iter_d}
\centering
\end{subfigure}
\begin{subfigure}{0.49\textwidth}
\centering
\includegraphics[scale=0.55]{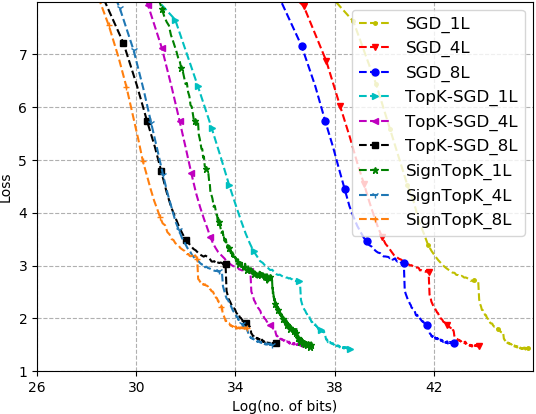}
\captionsetup{justification=centering}
\caption{Training loss vs $\textrm{log}_2$ of communication budget }
\label{fig:mid_lo-bit_d}
\centering
\end{subfigure}\\ \vspace{0.5cm}
\begin{subfigure}{0.49\textwidth}
\centering
\includegraphics[scale=0.55]{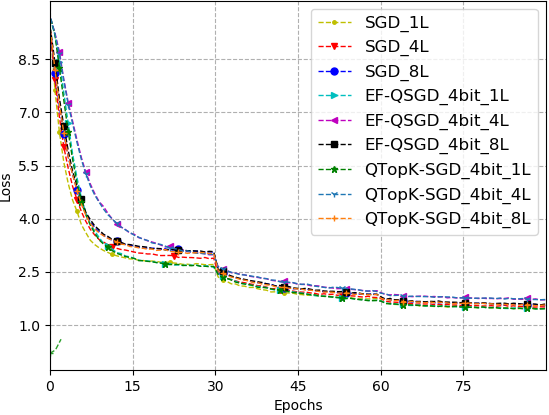}
\captionsetup{justification=centering}
\caption{Training loss vs epochs }
\label{fig:mid_lo-iter_c}
\centering
\end{subfigure}
\begin{subfigure}{0.49\textwidth}
\centering
\includegraphics[scale=0.55]{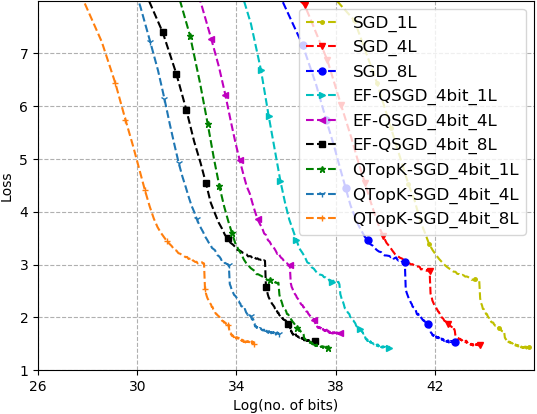}
\captionsetup{justification=centering}
\caption{Training loss vs $\textrm{log}_2$ of communication budget }
\label{fig:mid_lo-bit_c}
\centering
\end{subfigure}
\caption{\Figureref{mid_lo-iter_d}-\ref{fig:mid_lo-bit_d} demonstrate the effect of incorporating local iterations and compare these effects across vanilla SGD, the sparsifier $Top_k$, as well as its composition with the \emph{Sign} operator. Similar comparisons are also made between vanilla SGD, the quantizer QSGD with error accumulation, as well as its composition with the $Top_k$ sparsifier. $SignTopK\_hL$, for $h=1,4,8$, corresponds to running \Algorithmref{memQSGD-synchronous} with $SignTopK$ being the composed operator with a synchronization period of at most $h$.}
\end{figure*}

In Figure~\ref{fig:mid_lo-iter_d}-\ref{fig:mid_lo-bit_c}, we show how the performance of different methods (used in \Figureref{mid_lo-iter_a}-\ref{fig:te_5_a}) change when we incorporate local iterations on top of them.
Observe that the incorporation of local iterations in \Figureref{mid_lo-iter_d} and \ref{fig:mid_lo-iter_c} has very little impact on the convergence rates, as compared to vanilla SGD with the corresponding number of local iterations. 
Furthermore, this provides an added advantage over the \emph{Qsparse} operator, in terms of savings in communicated bits for achieving target loss as seen in \Figureref{mid_lo-bit_d} and \ref{fig:mid_lo-bit_c}, by a factor of 6 to 8 times on average. 

\begin{figure*}[ht!]
\centering
\begin{subfigure}{0.49\textwidth}
\includegraphics[scale=0.55]{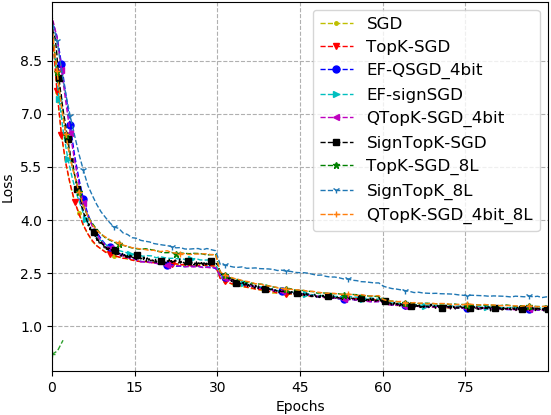}
\captionsetup{justification=centering}
\caption{Training loss vs against epochs }
\label{fig:mid_lo-iter}
\centering
\end{subfigure}
\begin{subfigure}{0.49\textwidth}
\includegraphics[scale=0.55]{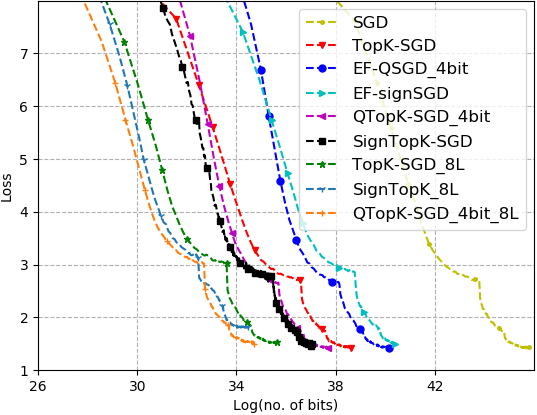}
\captionsetup{justification=centering}
\caption{Training loss vs $\textrm{log}_2$ of communication budget }
\label{fig:mid_lo-bit}
\centering
\end{subfigure}\\ \vspace{0.5cm}
\begin{subfigure}{0.49\textwidth}
\includegraphics[scale=0.55]{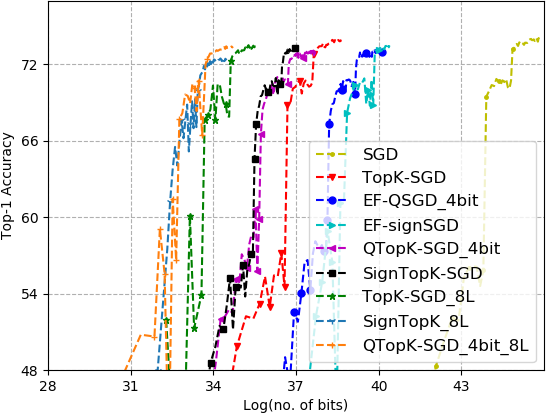}
\captionsetup{justification=centering}
\caption{top-1 accuracy \cite{lapin} for schemes in \Figureref{mid_lo-iter} }
\label{fig:te_1}
\centering
\end{subfigure}
\begin{subfigure}{0.49\textwidth}
\includegraphics[scale=0.55]{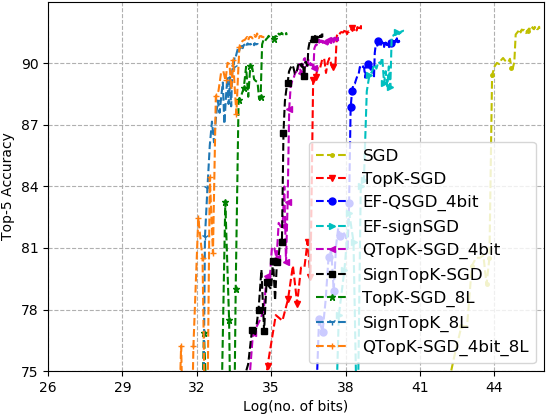}
\captionsetup{justification=centering}
\caption{top-5 accuracy \cite{lapin} for schemes in \Figureref{mid_lo-iter} }
\label{fig:te_5}
\centering
\end{subfigure}
\caption{\Figureref{mid_lo-iter}-\ref{fig:te_5} demonstrate the performance of our scheme in comparison with \textsc{ef}-\textsc{sign}SGD \cite{efsignsgd}, TopK-SGD \cite{memSGD,alistarh-sparsified} and local SGD \cite{localsgd2,alibaba_local} in the non-convex setting.}
\end{figure*}

\Figureref{mid_lo-bit}, \Figureref{te_1}, and \Figureref{te_5} show the training loss, top-1, and top-5
convergence rates\footnote{Here top-i refers to the accuracy of the top i predictions by the model from the list of possible classes, see \cite{lapin}.}
respectively, with respect to the total number of bits of
communication used.
We observe that \emph{Qsparse-local-SGD} combines the bit
savings of either the deterministic sign based operator or the stochastic quantizer (QSGD), and aggressive sparsifier along with infrequent communication, thereby, outperforming the cases where
these techniques are individually used. 
In particular, the required
number of bits to achieve the same loss or top-1 accuracy in the case of \emph{Qsparse-local-SGD} is around 1/16 in comparison with $Top_k$-$SGD$ and over 1000$\times$ less than vanilla SGD.
This also verifies that error compensation
through memory can be used to mitigate not only the missing components
from updates in previous synchronization rounds, but also explicit quantization error.

\subsection{Convex Objective} \label{sec:numerics_C}
The experiments in \Figureref{convex_sync_composed}-\ref{fig:convex_sync_performance} are in a synchronous distributed setting with 15 worker nodes, each processing a mini-batch size of 8 samples per iteration using the \textit{MNIST} \cite{mnist} handwritten digits dataset. The corresponding experiments for the asynchronous operation (as in \Algorithmref{memQSGD-asynchronous}) are shown in \Figureref{convex_async}.

\begin{figure*}[ht!]
\centering
\begin{subfigure}{0.49\textwidth}
\includegraphics[scale=0.58]{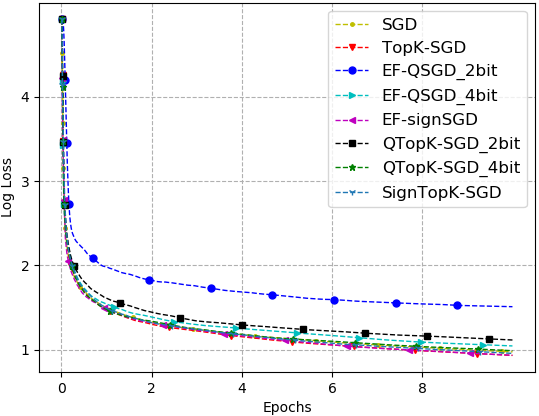}
\captionsetup{justification=centering}
\caption{Training loss vs epochs }
\label{fig:mid_lo-iter_j}
\end{subfigure}\hfill
\begin{subfigure}{0.49\textwidth}
\centering
\includegraphics[scale=0.58]{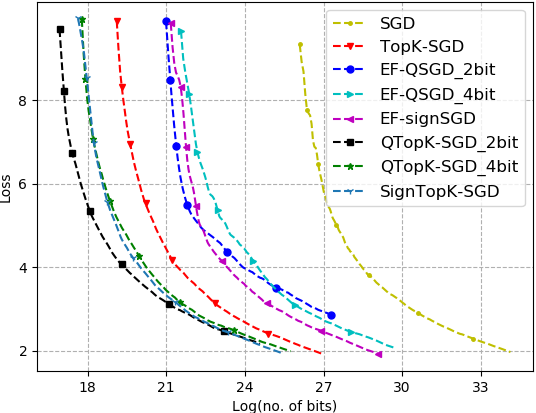}
\captionsetup{justification=centering}
\caption{Training loss vs $\textrm{log}_2$ of communication budget }
\label{fig:mid_lo-bit_k}
\end{subfigure}\\ \vspace{0.5cm}
\begin{subfigure}{0.49\textwidth}
\centering
\includegraphics[scale=0.58]{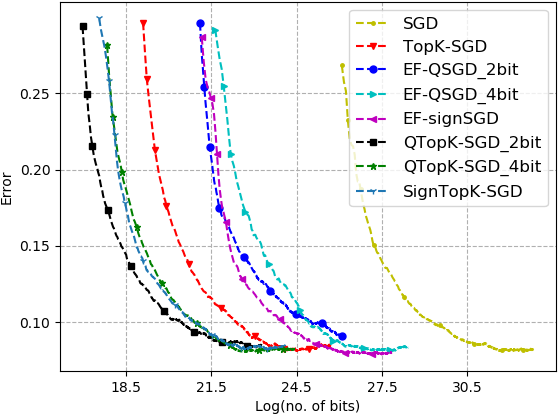}
\captionsetup{justification=centering}
\caption{top-1 accuracy \cite{lapin} for schemes in \Figureref{mid_lo-iter} }
\label{fig:te_1_l}
\end{subfigure}
\caption{\Figureref{mid_lo-iter_j}-\ref{fig:te_1_l} demonstrate the gains in performance achieved by our $Qsparse$ operators in the convex setting.}
\label{fig:convex_sync_composed}
\end{figure*}

\subsubsection{Model Architecture}
Define the softmax function as 
\begin{align}
    h_{\bx,z}\left(a^{(i)}\right)=\frac{\exp{\left(\bx_j^Ta^{(i)}+z^{(i)}\right)}}{\sum_{l=1}^L\exp{\left(\bx_l^Ta^{(i)}+z^{(l)}\right)}}.\notag
\end{align}
Our experiments are all for softmax regression with a standard $\ell_2$ regularizer. The cost function is 
\[-\frac{1}{n}\left(\sum_{i=1}^n\sum_{j=1}^L\mathbbm{1}\{b^{(i)}=j\}\log h_{\bx,z}\left(a^{(i)}\right)\right)+\frac{\lambda}{2}\|\bx\|^2\]
where $a^{(i)}\in\mathbb{R}^d$, $b^{(i)}\in [L]$ are the data points, which can belong to one of the $L$ classes, and $\bx_j\in\mathbb{R}^d$ for every $j\in[L]$, are columns of the parameter structured as follows 
\[\bx=
\begin{bmatrix}
\bx_1&\bx_2&\ldots&\bx_L
\end{bmatrix},\quad \bx_j\in\mathbb{R}^d,\,\,\forall j\in[L],\]
and $z^{(i)}$ for every $i\in [L]$ are the biases to be learnt corresponding to every class.
We set $\lambda$ to be $1/n$.

\subsubsection{Parameter Selection and Learning Rates}
We use the deterministic operator as in  \Lemmaref{composed-sign} and the stochastic operator $QSGD$ denoted by $Q$, as defined in \cite{QSGD}, as our quantizers and ${Top}_k$ with error compensation as the sparsifier. The schemes with which we compare our composed operators $QTop_k$ \Lemmaref{composed-quantizer}, and $Sign{Top}_k$ \Lemmaref{composed-sign}, are \textsc{ef-QSGD}\cite{ecqsgd}, \textsc{ef-signSGD}\cite{efsignsgd}, TopK-SGD \cite{memSGD,alistarh-sparsified}, and local SGD \cite{localsgd2}. The learning rate used for training is of the form $\frac{c}{\lambda (a+t)}$, where {\sf(i)} $\lambda$ is the regularization parameter; {\sf(ii)} $c$ is set with a careful hyperparameter sweep; {\sf(iii)} $w_t=(a+t)^2$ as in \Theoremref{convergence-strongly-convex-local}, where $a$ is set as $\frac{dH}{k}$ with $d$ being the dimension of the gradient vector (7850 for \textit{MNIST}); {\sf(iv)} $k=40$ is the sparsity;
{\sf(v)} $H$ is the synchronization period; {\sf(vi)} $t$ is the iteration index; {\sf(vii)} $b= 8$ is the batch size; and {\sf(viii)} $R=15$ is the number of workers.

\begin{figure*}[ht!]
\centering
\begin{subfigure}{0.49\textwidth}
\includegraphics[scale=0.55]{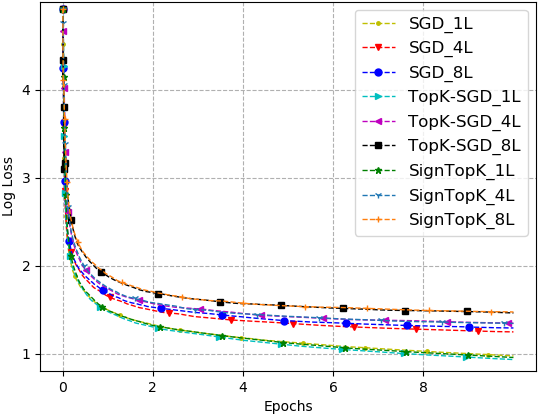}
\captionsetup{justification=centering}
\caption{Training loss vs epochs }
\label{fig:mid_lo-iter_m}
\centering
\end{subfigure}
\begin{subfigure}{0.49\textwidth}
\includegraphics[scale=0.55]{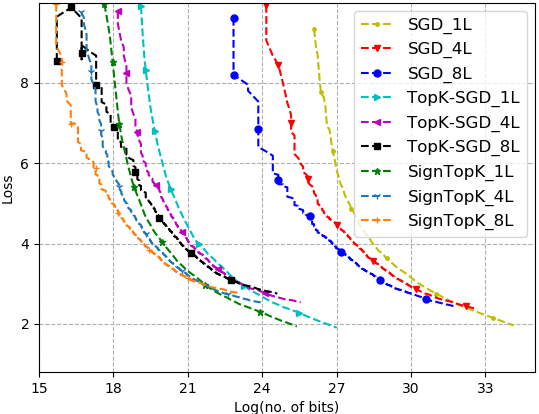}
\captionsetup{justification=centering}
\caption{Training loss vs $\textrm{log}_2$ of communication budget }
\label{fig:mid_lo-bit_n}
\centering
\end{subfigure}\\ \vspace{0.25cm}
\begin{subfigure}{0.49\textwidth}
\includegraphics[scale=0.55]{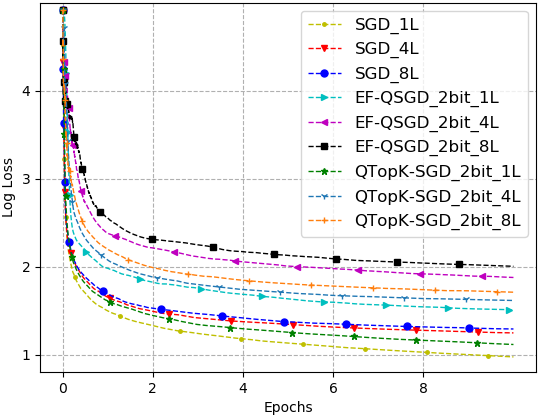}
\captionsetup{justification=centering}
\caption{Training loss vs epochs }
\label{fig:mid_lo-iter_o}
\centering
\end{subfigure}
\begin{subfigure}{0.49\textwidth}
\includegraphics[scale=0.55]{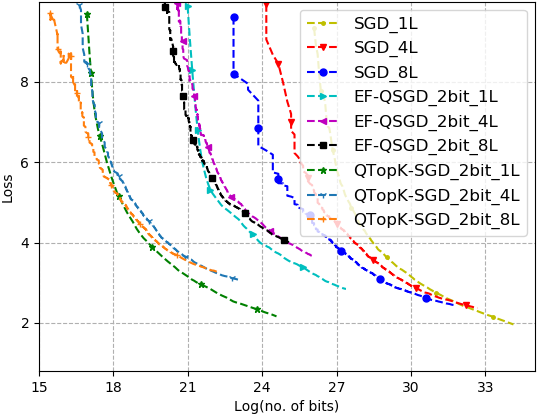}
\captionsetup{justification=centering}
\caption{Training loss vs $\textrm{log}_2$ of communication budget }
\label{fig:mid_lo-bit_p}
\centering
\end{subfigure}\\ \vspace{0.25cm}
\begin{subfigure}{0.49\textwidth}
\includegraphics[scale=0.55]{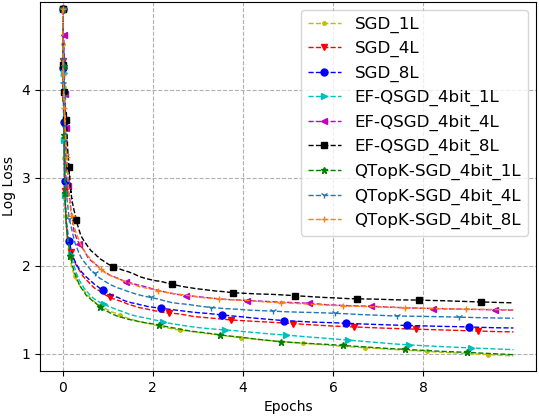}
\captionsetup{justification=centering}
\caption{Training loss vs epochs }
\label{fig:mid_lo-iter_q}
\centering
\end{subfigure}
\begin{subfigure}{0.49\textwidth}
\includegraphics[scale=0.55]{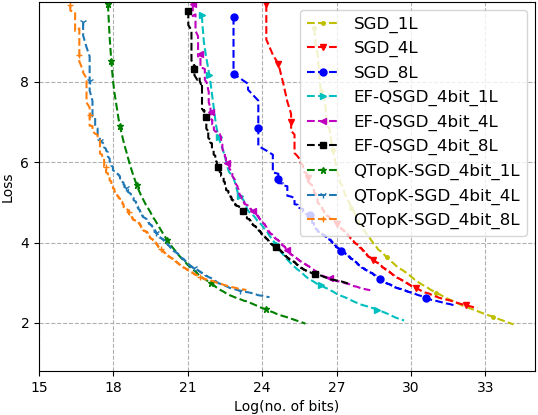}
\captionsetup{justification=centering}
\caption{Training loss vs $\textrm{log}_2$ of communication budget }
\label{fig:mid_lo-bit_r}
\centering
\end{subfigure}
\caption{\Figureref{mid_lo-iter_m}-\ref{fig:mid_lo-bit_n} demonstrate the effect of incorporating local iterations and compare these effects across vanilla SGD, the sparsifier \emph{TopK}, as well as its composition with the \emph{Sign} operator. Similar comparisons are also made between vanilla SGD, the quantizer QSGD with error accumulation, as well as its composition with the \emph{TopK} sparsifier.}
\label{fig:convex_sync_local}
\end{figure*}

\subsubsection{Results}
In \Figureref{mid_lo-iter_j}, we observe that the composition of a quantizer with a sparsifier has very little effect on the rate of convergence as compared to when the techniques are used individually. Observe that the algorithm run with the 2 bit $QSGD$ is slower than the 4 bit quantizer, both with or without sparsification, which can be attributed to the reduction in the compression coefficient $\gamma$ in going from 4 to 2 bits; see \Theoremref{convergence-strongly-convex-local}.
 From \Figureref{mid_lo-bit_k} and \ref{fig:te_1_l}, we see that our composed operators achieve gains in communicated bits by a factor of 6-8 times over the state-of-the-art.

\Figureref{mid_lo-iter_m}, demonstrates the effect of incorporating local iterations together with \emph{Qsparse} operators, and we see that the rate of convergence is not significantly affected as we go from 1 to 8 local iterations. Furthermore, observe that for a fixed number of local iterations, the \emph{Qsparse} operator maintains the same rates as vanilla $SGD$ or $Top_k$-$SGD$. In doing so, it is able to achieve gains in communicated bits as seen in \Figureref{mid_lo-bit_n}, simply by communicating infrequently with the master. On comparing \Figureref{mid_lo-iter_o} and \ref{fig:mid_lo-iter_q}, we observe that the $QTop_k$ operator is more sensitive to the increase in local computations for coarser quantizers (smaller values of $s$, in this case $s=2^{\#-bits}-1=3$). This can be verified from \Figureref{mid_lo-iter_q} which uses a 4 bit quantizer (which implies $s=15$ instead), and the corresponding effect of local iterations on the convergence rate is less prominent. We make comparisons between vanilla $SGD$, $QSGD$ with error accumulation (\textsc{ef-QSGD}) and our $QTop_k$ operator in \Figureref{mid_lo-iter_o} and \ref{fig:mid_lo-bit_p}, for which we do not observe much difference in performance between a finer and coarser quantizer, even though the convergence rates with respect to iterations are affected. This can be attributed to the precision of the quantizer itself.

\begin{figure*}[ht!]
\centering
\begin{subfigure}{0.49\textwidth}
\includegraphics[scale=0.58]{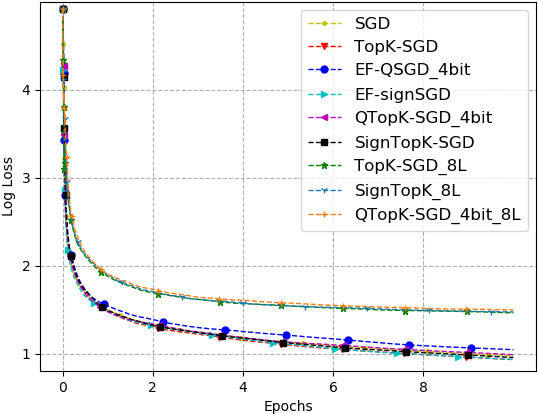}
\captionsetup{justification=centering}
\caption{Training loss vs epochs }
\label{fig:mid_lo-iterC}
\centering
\end{subfigure}\hfill
\begin{subfigure}{0.49\textwidth}
\includegraphics[scale=0.58]{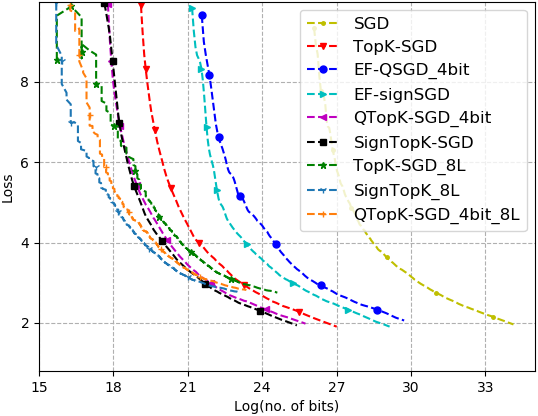}
\captionsetup{justification=centering}
\caption{Training loss vs $\textrm{log}_2$ of communication budget }
\label{fig:mid_lo-bitC}
\centering
\end{subfigure}\\ \vspace{0.5cm}
\begin{subfigure}{0.49\textwidth}
\includegraphics[scale=0.58]{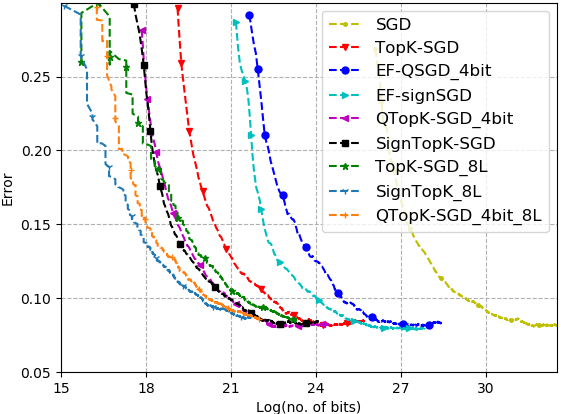}
\captionsetup{justification=centering}
\caption{top-1 accuracy \cite{lapin} for schemes in \Figureref{mid_lo-iterC} }
\label{fig:te_1C}
\centering
\end{subfigure}
\caption{\Figureref{mid_lo-iterC}-\ref{fig:te_1C} demonstrate the performance of our scheme in comparison with \textsc{ef-QSGD}, \textsc{ef}-\textsc{sign}SGD \cite{efsignsgd} and TopK-SGD \cite{memSGD,alistarh-sparsified} in a convex setting for synchronous updates. 
}
\label{fig:convex_sync_performance}
\end{figure*}

In \Figureref{mid_lo-iterC} and \Figureref{mid_lo-bitC}, we compare
the convergence of our proposed scheme in \Algorithmref{memQSGD-synchronous} with $QTop_k$ and $Sign{Top}_k$ being the composed operators, with
vanilla SGD (32 bit floating point), \textsc{ef-QSGD}, \textsc{ef-signSGD} \cite{efsignsgd}, and TopK-SGD
\cite{memSGD,alistarh-sparsified}. Both figures follow a similar trend,
where we observe $QTop_k$, $Sign{Top}_k$ and TopK-SGD to be converging
at the same rate as that of vanilla SGD,
which is similar to the observations in \cite{memSGD}. This implies
that the composition of quantization with sparsification does not
affect the convergence while achieving improved communication
efficiency, as can be seen in \Figureref{te_1C} and \Figureref{mid_lo-bitC}.
\Figureref{te_1C} shows that for test error approximately 0.1,
\emph{Qsparse-local-SGD} combines the benefits of the composed operator $Sign{Top}_k$ or $QTop_k$, with local computations, 
and needs 10-15 times less bits than TopK-SGD and 1000$\times$ less bits than vanilla SGD.

We observe similar trends in \Figureref{mid_lo-bitC1}-\ref{fig:te_1C1} for our asynchronous operation, where workers synchronize with the master at arbitrary time intervals as per \Algorithmref{memQSGD-asynchronous}. 
Specifically, in our experiments, for each $r\in[R]$, the time interval for the $r$th worker is decided uniformly at random from $[H]$ after every synchronization by that worker. 
This ensures that $gap(\I_T^{(r)})\leq H$ holds for every worker $r\in[R]$ and the schedule $\I_T^{(r)}$ is different for each of them.
\begin{figure*}[ht!]
\begin{subfigure}{0.49\textwidth}
\centering
\includegraphics[scale=0.57]{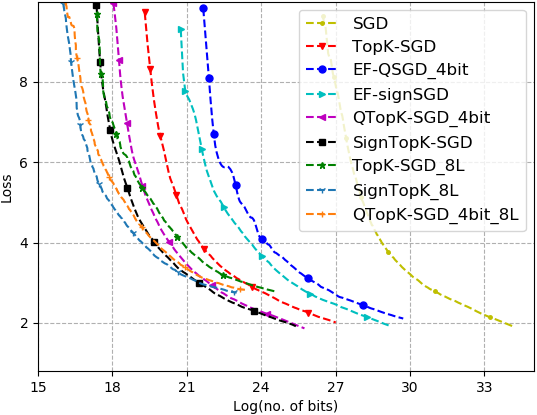}
\captionsetup{justification=centering}
\caption{Training loss with the communication budget for our schemes against baselines }
\label{fig:mid_lo-bitC1}
\end{subfigure}
\begin{subfigure}{0.49\textwidth}
\centering
\includegraphics[scale=0.57]{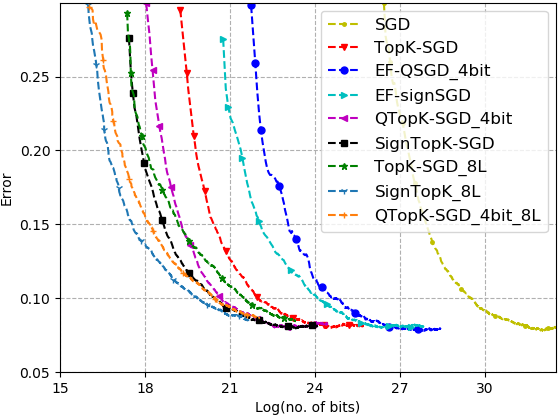}
\captionsetup{justification=centering}
\caption{Test error using a model trained for given number of iterations, as seen in \Figureref{mid_lo-bitC1} }
\label{fig:te_1C1}
\end{subfigure}
\caption{\Figureref{mid_lo-bitC1}-\ref{fig:te_1C1} demonstrate the performance of our scheme in comparison with \textsc{ef}-\textsc{sign}SGD \cite{efsignsgd} and TopK-SGD \cite{memSGD,alistarh-sparsified} in a convex setting for asynchronous operation.}
\label{fig:convex_async}
\end{figure*}

\section{Conclusion}
\label{sec:conclusion}
In this paper, we propose a gradient compression scheme that composes both unbiased and biased quantization
with aggressive sparsification. Furthermore, we incorporate local computations, which, when combined with quantization and explicit sparsification, results in a highly communication efficient distributed algorithm, which we call \emph{Qsparse-local-SGD}. We developed convergence analyses of our
scheme in both synchronous as well as asynchronous settings and for both convex and non-convex objectives, and we show that our proposed algorithm achieves the same rate as that of distributed vanilla SGD in each of these cases.
Our schemes provide flexibility in terms of different options for mitigating the communication bottlenecks that arise in training high-dimensional learning models over bandwidth limited networks. When run without compression, this also subsumes/generalizes several recent results from the literature on local SGD, with similar convergence rates, as mentioned at the end of \Sectionref{main_results_sync}.

Our numerics incorporate momentum acceleration, whose
analysis is a topic for future research ({\em e.g.}, potentially by incorporating
ideas from \cite{yu_momentum}). Although we use momentum for each local iteration,
our preliminary results suggest that 
our method works with momentum applied to a block of updates as well though it was not the main focus of this paper.

\section*{Acknowledgement}
The authors gratefully thank Navjot Singh for his help with experiments in the early stages of this work.
This work was partially supported by NSF grant \#1514531, by UC-NL grant LFR-18-548554 and by Army Research Laboratory under Cooperative Agreement W911NF-17-2-0196.
The views and conclusions contained in this document are those of the authors and should not be interpreted as representing the official policies, either expressed or implied, of the Army Research Laboratory or the U.S. Government. The U.S. Government is authorized to reproduce and distribute reprints for Government purposes notwithstanding any copyright notation here on.

\appendix

\section{Omitted Details from \Sectionref{operators}}
\label{app:supp_operators}
\subsection{Proof of \Lemmaref{composed-compression}}\label{app:op_1}
\begin{lemma*}[Restating \Lemmaref{composed-compression}]
Let $Comp_k\in\{\mathrm{Top}_k,\mathrm{Rand}_k\}$. Let $Q_s:\mathbb{R}^d\rightarrow\mathbb{R}^d$ be a quantizer with parameter $s$ that satisfies \Definitionref{QuantDef}. 
Let $Q_sComp_k:\R^d\to\R^d$ be defined as $Q_sComp_k(\mathbf{x}) := Q_s(Comp_k(\mathbf{x}))$ for every $\mathbf{x}\in\R^d$. If $k,s$ are such that $\beta_{k,s}<1$. then $Q_sComp_k:\R^d\to\R^d$ is a compression operator with the compression coefficient being equal to $\gamma=(1-\beta_{k,s})\frac{k}{d}$, i.e., for every $\mathbf{x}\in\R^d$, we have
\begin{align*}
  \mathbb{E}_{C,Q}[\|\mathbf{x}-Q_sComp_k(\mathbf{x})\|_2^2]\leq \left[1-\left(1-\beta_{k,s}\right)\frac{k}{d}\right]\|\mathbf{x}\|_2^2,
\end{align*}
where expectation is taken over the randomness of the compression operator $Comp_k$ as well as the quantizer $Q_s$.
\end{lemma*}
\begin{proof}
Fix an arbitrary $\mathbf{x}\in\R^d$.
\begin{align}
&\mathbb{E}_{C,Q}[\|\mathbf{x}-Q_sComp_k(\mathbf{x})\|_2^2] \nonumber\\
&= \mathbb{E}_{C,Q}[\|\mathbf{x}\|_2^2]+\mathbb{E}_{C,Q}[\|Q_sComp_k(\mathbf{x})\|_2^2]\nonumber\\
&\hspace{2cm}-2\mathbb{E}_C[\langle\mathbf{x},\mathbb{E}_{Q}[Q_sComp_k(\mathbf{x})]\rangle] \notag \nonumber\\
&= \|\mathbf{x}\|_2^2+\mathbb{E}_{C,Q}[\|Q_sComp_k(\mathbf{x})\|_2^2]-2\mathbb{E}_C[\langle\mathbf{x},Comp_k(\mathbf{x})\rangle] \nonumber
\end{align}
In the last equality, we used that $\mathbf{x}$ is constant with respect to the randomness of $Q_s$ and $Comp_k$, and that $\mathbb{E}_Q[Q_sComp_k(\mathbf{x})]=Comp_k(\mathbf{x})$, which follows from {\sf (i)} of Definition~\ref{defn:QuantDef}.
Observe that, for any $Comp_k\in\{\mathrm{Top}_k,\mathrm{Rand}_k\}$, we have $\langle\mathbf{x},Comp_k(\mathbf{x})\rangle=\|Comp_k(\mathbf{x})\|_2^2$.
Continuing from above, we get
\begin{align}
&\mathbb{E}_{C,Q}[\|\mathbf{x}-Q_sComp_k(\mathbf{x})\|_2^2] = \|\mathbf{x}\|_2^2-2\mathbb{E}_C[\|Comp_k(\mathbf{x})\|_2^2]\nonumber\\
&\hspace{3cm}+\mathbb{E}_{C,Q}[\|Q_sComp_k(\mathbf{x})\|_2^2] \label{eq:composed-interim2}
\end{align}
Observe that for any $Comp_k\in\{\mathrm{Top}_k,\mathrm{Rand}_k\}$, $Comp_k(\mathbf{x})$ is a length-$d$ vector, but only (at most) $k$ of its components are non-zero. 
This implies that, by treating $Comp_k(\mathbf{x})$ a length-$k$ vector whose entries correspond to the $k$ non-zero entries of $\mathbf{x}$, we can write $\mathbb{E}_Q[\|Q_sComp_k(\mathbf{x})\|_2^2]\leq(1+\beta_{k,s})\|Comp_k(\mathbf{x})\|_2^2$; see {\sf (ii)} of Definition~\ref{defn:QuantDef}.
Putting this back in \eqref{eq:composed-interim2}, we get
\begin{align}
&\mathbb{E}_{C,Q}[\|\mathbf{x}-Q_sComp_k(\mathbf{x})\|_2^2] \notag\\
&\leq \|\mathbf{x}\|_2^2-\mathbb{E}_{C}[\|Comp_k(\mathbf{x})\|_2^2]+\beta_{k,s}\mathbb{E}_{C}[\|Comp_k(\mathbf{x})\|_2^2] \notag \\
&= \|\mathbf{x}\|_2^2-\left(1-\beta_{k,s}\right)\mathbb{E}_{C}[\|Comp_k(\mathbf{x})\|_2^2] \label{eq:composed-interim3}
\end{align}
Using $\mathbb{E}_{C}[\|Comp_k(\mathbf{x})\|_2^2]\geq \frac{k}{d}\|\mathbf{x}\|_2^2$ (see \eqref{eq:comp-2norm-bound} in \Lemmaref{comp-moment-bound-appx})  in
\eqref{eq:composed-interim3} gives
\begin{align*}
\mathbb{E}_{C,Q}[\|\mathbf{x}-Q_sComp_k(\mathbf{x})\|_2^2] &\leq \|\mathbf{x}\|^2-\left(1-\beta_{k,s}\right)\frac{k}{d}\|\mathbf{x}\|_2^2 \\
&= \left[1-\left(1-\beta_{k,s}\right)\frac{k}{d}\right]\|\mathbf{x}\|_2^2.
\end{align*}
This completes the proof of \Lemmaref{composed-compression}.
\end{proof}

\subsection{Proof of \Lemmaref{composed-quantizer}}\label{app:op_2}
\begin{lemma*}[Restating \Lemmaref{composed-quantizer}]
Let $Comp_k\in\{\mathrm{Top}_k,\mathrm{Rand}_k\}$. Let $Q_s:\mathbb{R}^d\rightarrow\mathbb{R}^d$ be a stochastic quantizer with parameter $s$ that satisfies \Definitionref{QuantDef}. 
Let $Q_sComp_k:\R^d\to\R^d$ be defined as $Q_sComp_k(\mathbf{x}) := Q_s(Comp_k(\mathbf{x}))$ for every $\mathbf{x}\in\R^d$. 
Then $\frac{Q_sComp_k(\mathbf{x})}{1+\beta_{k,s}}$ is a compression operator with the compression coefficient being equal to $\gamma=\frac{k}{d(1+\beta_{k,s})}$, 
i.e., for every $\bx\in\R^d$
\begin{align*}
  \mathbb{E}_{C,Q}\left[\left\|\mathbf{x}-\frac{Q_sComp_k(\mathbf{x})}{1+\beta_{k,s}}\right\|_2^2\right]\ \leq\ \left[1-\frac{k}{d(1+\beta_{k,s})}\right]\|\mathbf{x}\|_2^2,
\end{align*}
\end{lemma*}
\begin{proof}
Fix an arbitrary $\mathbf{x}\in\R^d$.
{\allowdisplaybreaks
\begin{align}
\mathbb{E}_{C,Q}\left[\left\|\mathbf{x}-\frac{Q_sComp_k(\mathbf{x})}{(1+\beta_{k,s})}\right\|_2^2\right] &= \|\mathbf{x}\|_2^2-2\mathbb{E}_{C}\left[\left\langle \bx,\mathbb{E}_Q\left[\frac{Q_sComp_k(\mathbf{x})}{(1+\beta_{k,s})}\right]\right\rangle\right]+\mathbb{E}_{C,Q}\left[\frac{\|Q_sComp_k(\mathbf{x})\|_2^2}{(1+\beta_{k,s})^2}\right] \notag \\
&\stackrel{\text{(a)}}{=} \|\mathbf{x}\|_2^2 - \frac{2}{(1+\beta_{k,s})}\mathbb{E}_{C}\left[\left\langle \bx,Comp_k(\mathbf{x})\right\rangle\right] \notag \\
&\hspace{4cm}+\frac{1}{(1+\beta_{k,s})^2}\mathbb{E}_{C,Q}\left[\|Q_sComp_k(\mathbf{x})\|_2^2\right] \notag \\
&\stackrel{\text{(b)}}{=} \|\mathbf{x}\|_2^2 - \frac{2}{(1+\beta_{k,s})}\mathbb{E}_{C}\left[\|Comp_k(\mathbf{x})\|_2^2\right] \notag \\
&\hspace{4cm}+\frac{1}{(1+\beta_{k,s})^2}\mathbb{E}_{C,Q}\left[\|Q_sComp_k(\mathbf{x})\|_2^2\right] \notag \\
&\stackrel{\text{(c)}}{\leq} \|\mathbf{x}\|_2^2 - \frac{2}{1+\beta_{k,s}}\mathbb{E}_{C}\left[\|Comp_k(\mathbf{x})\|_2^2\right] \notag \\
&\hspace{4cm}+\frac{1}{(1+\beta_{k,s})}\mathbb{E}_{C}\left[\|Comp_k(\mathbf{x})\|_2^2\right] \notag \\
&= \|\mathbf{x}\|_2^2 - \frac{1}{(1+\beta_{k,s})}\mathbb{E}_{C}\left[\|Comp_k(\mathbf{x})\|_2^2\right] \notag \\
&\stackrel{\text{(d)}}{\leq} \left[1-\frac{k}{d(1+\beta_{k,s})}\right]\|\mathbf{x}\|_2^2.
\end{align}
In (a) we used $\mathbb{E}_Q[Q_sComp_k(\mathbf{x})]=Comp_k(\mathbf{x})$, in (b) we used $\langle\mathbf{x},Comp_k(\mathbf{x})\rangle=\|Comp_k(\mathbf{x})\|_2^2$; in (c) we used $\mathbb{E}_Q[\|Q_sComp_k(\mathbf{x})\|_2^2]\leq(1+\beta_{k,s})\|Comp_k(\mathbf{x})\|_2^2$; and in (d) we used $\mathbb{E}_{C}[\|Comp_k(\mathbf{x})\|_2^2]\geq \frac{k}{d}\|\mathbf{x}\|_2^2$. 
This completes the proof of \Lemmaref{composed-quantizer}.
}
\end{proof}

\subsection{Proof of \Lemmaref{composed-sign}}\label{app:op_3}
\begin{lemma*}[Restating \Lemmaref{composed-sign}]
For $Comp_k\in\{\mathrm{Top}_k,\mathrm{Rand}_k\}$, $\frac{\|Comp_k(\mathbf{x})\|_m \,SignComp_k(\mathbf{x})}{k}$, for any $m\in\mathbb{Z}_+$ is a compression operator with the compression coefficient $\gamma_m$ being equal to
\[\gamma_m = 
\begin{cases}
\max\left\{\frac{1}{d}, \frac{k}{d}\left(\frac{\|Comp_k(\mathbf{x})\|_1}{\sqrt{d}\|Comp_k(\mathbf{x})\|_2}\right)^2\right\} & \text{ if } m=1, \\
\frac{k^{\frac{2}{m}-1}}{d} & \text{ if } m\geq 2.
\end{cases}
\]
\end{lemma*}
\noindent For proving \Lemmaref{composed-sign} we first state and prove \Lemmaref{comp-moment-bound-appx} below.
\begin{lemma}\label{lem:comp-moment-bound-appx}
Let $Comp_k\in\{\mathrm{Top}_k,\mathrm{Rand}_k\}$. For any $\mathbf{x}\in\R^d$, we have 
\begin{align}
\mathbb{E}[\|Comp_k(\mathbf{x})\|_1^2] &\geq \max\left\{\frac{k}{d}\|\mathbf{x}\|_2^2, \frac{k^2}{d^2}\|\mathbf{x}\|_1^2\right\} \label{eq:comp-1norm-bound}\\
\mathbb{E}[\|Comp_k(\mathbf{x})\|_2^2] &\geq \frac{k}{d}\|\mathbf{x}\|_2^2. \label{eq:comp-2norm-bound}
\end{align}
\end{lemma}

\begin{proof}
{\allowdisplaybreaks
Let $m\in\{1,2\}$. Observe that for any $\mathbf{x}\in\R^d$, we have $\mathbb{E}[\|\mathrm{Top}_k(\mathbf{x})\|_m^2]=\|\mathrm{Top}_k(\mathbf{x})\|_m^2$ and that $\|\mathrm{Top}_k(\mathbf{x})\|_m^2\geq \mathbb{E}[\|\mathrm{Rand}_k(\mathbf{x})\|_m^2]$.
So, in order to prove the lemma, it suffices to show that $\mathbb{E}[\|\mathrm{Rand}_k(\mathbf{x})\|_m^2]\geq\frac{k}{d}\|\mathbf{x}\|_m^2$ holds for any $m\in\{1,2\}$,
and that $\mathbb{E}[\|\mathrm{Rand}_k(\mathbf{x})\|_1^2]\geq\frac{k^2}{d^2}\|\mathbf{x}\|_1^2$.
Let $\Omega_k$ be the set of all the $k$-elements subsets of $[d]$.
\begin{align*}
\mathbb{E}[\|\mathrm{Rand}_k(\mathbf{x})\|_m^2] &= \sum_{\omega\in\Omega_k}\frac{1}{|\Omega_k|}\left(\sum_{i=1}^d|x_i|^m\cdot\mathbbm{1}{\{i\in\omega\}}\right)^{2/m} \\
&\stackrel{\text{(a)}}{\geq} \sum_{\omega\in\Omega_k}\frac{1}{|\Omega_k|}\sum_{i=1}^d|x_i|^2\cdot\mathbbm{1}{\{i\in\omega\}} \\
&= \sum_{i=1}^d x_i^2\cdot \frac{1}{|\Omega_k|}\sum_{\omega\in\Omega_k}\mathbbm{1}{\{i\in\omega\}} \\
&= \sum_{i=1}^d x_i^2\cdot \frac{1}{|\Omega_k|} \binom{d-1}{k-1} \\
&= \frac{k}{d}\|\mathbf{x}\|_2^2
\end{align*}
Note that (a) holds only for $m\in\{1,2\}$, and it is equality for $m=2$. Now we show that $\mathbb{E}[\|\mathrm{Rand}_k(\mathbf{x})\|_1^2]\geq\frac{k^2}{d^2}\|\mathbf{x}\|_1^2$.
\begin{align*}
\mathbb{E}[\|\mathrm{Rand}_k(\mathbf{x})\|_1^2] &\geq \left(\mathbb{E}[\|\mathrm{Rand}_k(\mathbf{x})\|_1]\right)^2 \\
&= \left(\sum_{\omega\in\Omega_k}\frac{1}{|\Omega_k|}\sum_{i=1}^d|x_i|\cdot\mathbbm{1}{\{i\in\omega\}}\right)^2 \\
&= \left(\sum_{i=1}^d |x_i|\cdot \frac{1}{|\Omega_k|}\sum_{\omega\in\Omega_k}\mathbbm{1}{\{i\in\omega\}}\right)^2 \\
&= \left(\sum_{i=1}^d |x_i|\cdot \frac{1}{|\Omega_k|} \binom{d-1}{k-1}\right)^2 \\
&= \frac{k^2}{d^2}\|\mathbf{x}\|_1^2
\end{align*}
This completes the proof of \Lemmaref{comp-moment-bound-appx}.
}
\end{proof}

\begin{proof}[Proof of \Lemmaref{composed-sign}]
Fix an arbitrary $\bx\in\R^d$ and consider the following:
{\allowdisplaybreaks
\begin{align}
    &\mathbb{E}_C\left\|\frac{\|Comp_k(\mathbf{x})\|_m \,SignComp_k(\mathbf{x})}{k}-\mathbf{x}\right\|_2^2\notag\\
    &= \mathbb{E}_C\left[\frac{\|Comp_k(\mathbf{x})\|_m^2}{k}-2\left\langle\frac{\|Comp_k(\mathbf{x})\|_m \,SignComp_k(\mathbf{x})}{k},\mathbf{x}\right\rangle+\|\mathbf{x}\|_2^2\right]\notag\\
        &= \mathbb{E}_C\left[\frac{\|Comp_k(\mathbf{x})\|_m^2}{k}-2\frac{\|Comp_k(\mathbf{x})\|_m\|Comp_k(\mathbf{x})\|_1}{k}+\|\mathbf{x}\|_2^2\right]\notag\\
    &\leq\|\mathbf{x}\|_2^2-\frac{\mathbb{E}_C\|Comp_k(\mathbf{x})\|_m^2}{k} \label{eq:composed-sign-interim1}
    \end{align}
    In \eqref{eq:composed-sign-interim1} we used the fact that $\|\cdot\|_1 \geq \|\cdot\|_m$ for every $m\geq 1$. \\
}
\noindent{Case 1.} When $m=1$: Substituting $\mathbb{E}_C\|Comp_k(\mathbf{x})\|_1^2 \geq \max\left\{\frac{k}{d}\|\mathbf{x}\|_2^2, \frac{k^2}{d^2}\|\mathbf{x}\|_1^2\right\}$ (from 
\eqref{eq:comp-1norm-bound}) in \eqref{eq:composed-sign-interim1} gives 
\begin{align*}
\mathbb{E}_C\left\|\frac{\|Comp_k(\mathbf{x})\|_1 \,SignComp_k(\mathbf{x})}{k}-\mathbf{x}\right\|_2^2 &\leq 
\|\mathbf{x}\|_2^2 - \frac{1}{k}\max\left\{\frac{k}{d}\|\mathbf{x}\|_2^2, \frac{k^2}{d^2}\|\mathbf{x}\|_1^2\right\} \\
&\leq \left[1- \max\left\{\frac{1}{d}, \frac{k}{d}\left(\frac{\|Comp_k(\mathbf{x})\|_1}{\sqrt{d}\|Comp_k(\mathbf{x})\|_2}\right)^2\right\} \right]\|\mathbf{x}\|_2^2. \\
\end{align*}

\noindent{Case 2.} When $m\geq 2$: 
    Since $\|\mathbf{u}\|_p \leq k^{\frac{1}{p}-\frac{1}{q}}\|\mathbf{u}\|_q$ holds for every $\mathbf{u}\in\R^k$, whenever $p\leq q$, using this in 
    \eqref{eq:composed-sign-interim1} with $q=m$ and $p=2$ gives
    \begin{align}
    &\mathbb{E}_C\left\|\frac{\|Comp_k(\mathbf{x})\|_m \,SignComp_k(\mathbf{x})}{k}-\mathbf{x}\right\|_2^2\notag\\
    &\hspace{2cm}\leq \|\mathbf{x}\|_2^2-\frac{1}{k}k^{\frac{2}{m}-1}\mathbb{E}_C[\|Comp_k(\mathbf{x})\|_2^2]\notag\\
    &\hspace{2cm}\leq \|\mathbf{x}\|_2^2-\frac{1}{k}k^{\frac{2}{m}-1}(k/d)\|\mathbf{x}\|_2^2 \quad \text{(By \Lemmaref{comp-moment-bound-appx})}\notag\\
    &\hspace{2cm}=\left[1-\frac{k^{\frac{2}{m}-1}}{d}\right]\|\mathbf{x}\|_2^2.
\end{align}
This completes the proof of \Lemmaref{composed-sign}.
\end{proof}

\section{Omitted Details from \Sectionref{loc-sync}}
\label{app:supp_sync}
\subsection{\noindent Proof of \Lemmaref{bounded-memory-decaying-lr}}\label{app:mem_dec}
\begin{lemma*}[Restating \Lemmaref{bounded-memory-decaying-lr}]
Let $gap(\mathcal{I}_T)\leq H$ and $\eta_t=\frac{\xi}{a+t}$, where $\xi$ is a constant and $a>\frac{4H}{\gamma}$. Then there exists a constant $C\geq \frac{4a\gamma(1-\gamma^2)}{a\gamma-4H}$, such that the following holds for every worker $r\in[R]$ and for every $t\in \mathbb{Z}^+$:
\begin{align*}
    \mathbb{E}\|m_t^{(r)}\|_2^2\ \leq\ 4\frac{\eta_t^2}{\gamma^2}CH^2G^2.
\end{align*}
\end{lemma*}
\begin{proof}
Fix an arbitrary worker $r\in[R]$.
In order to prove the lemma, we need to show that $\mathbb{E}\|m_{t}^{\left(r\right)}\|^2 \leq 4\tfrac{\eta_{t}^2}{\gamma^2}CH^2G^2$ holds for every $t\in[T]$, where $C\geq \tfrac{4a\gamma(1-\gamma^2)}{a\gamma-4H}$.
We show this separately for two cases, depending on whether or not $t\in\I_T$. First consider the case when $t\in\I_T$.
Let $\mathcal{I}_T=\{t_{(1)},t_{(2)},\ldots,t_{(l)}=T\}$. Fix any $i=1,2,\hdots,l$ and consider $\mathbb{E}\|m_{t_{(i+1)}}^{(r)}\|^2$.
Note that local memory $m_t^{(r)}$ at any worker $r$ and the global parameter vector $\bx_t$ do not change in between the synchronization indices.
We define $m_{t_{(0)}}^{(r)}:=\bzero$ for every $r\in[R]$.
{\allowdisplaybreaks
\begin{align}
\mathbb{E}\|m_{t_{(i+1)}}^{(r)}\|^2 &= \mathbb{E}\|m_{t_{(i+1)}-1}^{(r)}+\mathbf{x}_{t_{(i+1)}-1}-\widehat{\mathbf{x}}_{t_{(i+1)}-\tfrac{1}{2}}^{(r)} - g_{t_{(i+1)}-1}^{(r)}\|^2\notag\\
&\stackrel{\text{(a)}}{\leq} (1-\gamma)\mathbb{E}\|m_{t_{(i+1)}-1}^{(r)}+\mathbf{x}_{t_{(i+1)}-1}-\widehat{\mathbf{x}}_{t_{(i+1)}-\tfrac{1}{2}}^{(r)}\|^2\notag\\
&\stackrel{\text{(b)}}{=} \left(1-\gamma\right)\mathbb{E}\|m_{t_{(i)}}^{(r)}+\mathbf{x}_{t_{(i)}}-\widehat{\mathbf{x}}_{t_{(i+1)}-\tfrac{1}{2}}^{(r)}\|^2\notag\\
&\stackrel{\text{(c)}}{=} \left(1-\gamma\right)\mathbb{E}\|m_{t_{(i)}}^{(r)}+\bwx_{t_{(i)}}^{(r)}-\widehat{\mathbf{x}}_{t_{(i+1)}-\tfrac{1}{2}}^{(r)}\|^2 \label{eq:bounded-memory-interim0}
\end{align}
Here (a) is due to the compression property, (b) holds since the memory and master parameter remain unchanged between two rounds of synchronization, and in (c) we used that $\bwx_{t_{(i)}}^{(r)}=\bx_{t_{(i)}}$, which holds for every $r$. 
Using the inequality $\|{\bf a} + {\bf b}\|^2\leq(1+\tau)\|{\bf a}\|^2+(1+\tfrac{1}{\tau})\|{\bf b}\|^2$, which holds for every $\tau>0$, in \eqref{eq:bounded-memory-interim0} gives (take any $p>1$ in the following):
\begin{align}
\mathbb{E}\|m_{t_{(i+1)}}^{(r)}\|^2&\leq \left(1-\gamma\right)\left[\left(1+\tfrac{(p-1)\gamma}{p}\right)\mathbb{E}\|m_{t_{(i)}}^{(r)}\|^2+\left(1+\tfrac{p}{(p-1)\gamma}\right)\mathbb{E}\|\bwx_{t_{(i)}}^{(r)}-\widehat{\mathbf{x}}_{t_{(i+1)}-\tfrac{1}{2}}^{(r)}\|^2\right]\notag\\
&\leq \left(1-\tfrac{\gamma}{p}\right)\mathbb{E}\|m_{t_{(i)}}^{(r)}\|^2+\tfrac{(1-\gamma)(p\gamma+p)}{(p-1)\gamma}\mathbb{E}\|\bwx_{t_{(i)}}^{(r)}-\widehat{\mathbf{x}}_{t_{(i+1)}-\tfrac{1}{2}}^{(r)}\|^2\notag\\
&= \left(1-\tfrac{\gamma}{p}\right)\mathbb{E}\|m_{t_{(i)}}^{(r)}\|^2+\tfrac{p(1-\gamma^2)}{(p-1)\gamma}\mathbb{E}\|\bwx_{t_{(i)}}^{(r)}-\widehat{\mathbf{x}}_{t_{(i+1)}-\tfrac{1}{2}}^{(r)}\|^2\notag\\
&=\left(1-\tfrac{\gamma}{p}\right)\mathbb{E}\|m_{t_{(i)}}^{(r)}\|^2+\tfrac{p(1-\gamma^2)}{(p-1)\gamma}\mathbb{E}\|\sum_{j=t_{(i)}}^{t_{(i+1)}-1}\eta_j\nabla f_{i_j^{(r)}}\left(\widehat{\mathbf{x}}_j^{\left(r\right)}\right)\|^2\notag\\
&\leq \left(1-\tfrac{\gamma}{p}\right)\mathbb{E}\|m_{t_{(i)}}^{(r)}\|^2+\tfrac{p(1-\gamma^2)}{(p-1)\gamma}\eta_{t_{(i)}}^2H^2G^2 \label{eq:bounded-memory-interim1}
\end{align}
In the last inequality \eqref{eq:bounded-memory-interim1} we used $\mathbb{E}\|\sum_{j=t_{(i)}}^{t_{(i+1)}-1}\eta_j\nabla f_{i_j^{(r)}}\left(\widehat{\mathbf{x}}_j^{\left(r\right)}\right)\|^2\leq \eta_{t_{(i)}}^2H^2G^2$, which can be seen as follows:
\begin{align*}
\mathbb{E}\|\sum_{j={t_{(i)}}}^{t_{(i+1)}-1}\eta_j\nabla^{\left(r\right)}f_{\left(i_j\right)}\left(\widehat{\mathbf{x}}_j^{\left(r\right)}\right)\|^2 &= (t_{(i+1)}-{t_{(i)}})^2\mathbb{E}\|\tfrac{1}{(t_{(i+1)}-{t_{(i)}})}\sum_{j={t_{(i)}}}^{t_{(i+1)}-1}\eta_j\nabla f_{i_j^{(r)}}\left(\widehat{\mathbf{x}}_j^{\left(r\right)}\right)\|^2 \\
&\hspace{-2cm}\stackrel{\text{(a)}}{\leq} (t_{(i+1)}-{t_{(i)}})\sum_{j={t_{(i)}}}^{t_{(i+1)}-1}\mathbb{E}\|\eta_j\nabla f_{i_j^{(r)}}\left(\widehat{\mathbf{x}}_j^{\left(r\right)}\right)\|^2 
\\
&\hspace{-2cm}\stackrel{\text{(b)}}{\leq} (t_{(i+1)}-{t_{(i)}})\eta_{{t_{(i)}}}^2\sum_{j={t_{(i)}}}^{t_{(i+1)}-1}\mathbb{E}\|\nabla f_{i_j^{(r)}}\left(\widehat{\mathbf{x}}_j^{\left(r\right)}\right)\|^2 
\\
&\hspace{-2cm}\leq (t_{(i+1)}-{t_{(i)}})\eta_{{t_{(i)}}}^2(t_{(i+1)}-{t_{(i)}})G^2 \\
&\hspace{-2cm}\stackrel{\text{(c)}}{\leq} \eta_{{t_{(i)}}}^2H^2G^2 
\end{align*}
}
Here (a) holds by Jensen's inequality, (b) holds since \text{since }$\eta_t \leq \eta_{t_{(i)}} \forall t\geq t_{(i)}$ and (c) holds because $(t_{(i+1)}-t_{(i)})\leq H$.
Define $\Tilde{\eta}_t=\tfrac{1}{a+t}$ and $A=\xi^2H^2G^2$. Using this in \eqref{eq:bounded-memory-interim1} gives
\begin{align}
    \mathbb{E}\|m_{t_{(i+1)}}^{\left(r\right)}\|^2&\leq\left(1-\tfrac{\gamma}{p}\right)\mathbb{E}\|m_{t_{(i)}}^{\left(r\right)}\|^2+\tfrac{p(1-\gamma^2)}{(p-1)\gamma}\Tilde{\eta}_{t_{(i)}}^2A. \label{eq:bounded-memory-interim2}
\end{align}

We want to show that $\mathbb{E}\|m_{t_{(i)}}^{\left(r\right)}\|^2 \leq 4C\tfrac{\tilde{\eta}_{t_{(i)}}^2}{\gamma^2}A$ holds for every $i=1,2,\hdots$, 
where $C\geq \tfrac{4a\gamma(1-\gamma^2)}{a\gamma-4H}$. In fact we prove a slightly stronger bound that $\mathbb{E}\|m_{t_{(i)}}^{\left(r\right)}\|^2 \leq C\tfrac{\tilde{\eta}_{t_{(i)}}^2}{\gamma^2}A$ holds for every $i=1,2,\hdots$.
We prove this using induction on $i$. \\

\textit{ Base case $(i=1)$:} Note that $m_{t_{(1)}-1}^{(r)}=m_0^{(r)}=\bzero$. Consider the following:
\begin{align*}
\mathbb{E}\|m_{t_{(1)}}^{(r)}\|^2 &= \mathbb{E}\|\bx_{t_{(1)}-1} - \bwx_{t_{(1)}-\tfrac{1}{2}} - g_{t_{(1)}-1}^{(r)}\|^2 \\
&\leq (1-\gamma)\mathbb{E}\|\bx_{t_{(1)}-1} - \bwx_{t_{(1)}-\tfrac{1}{2}}\|^2 \\
&\stackrel{\text{(a)}}{=} (1-\gamma)\mathbb{E}\|\bwx_{0}^{(r)} - \bwx_{t_{(1)}-\tfrac{1}{2}}\|^2 
\\
&= (1-\gamma)\mathbb{E}\|\sum_{j=0}^{t_{(1)}-1}\eta_j\nabla f_{i_j^{(r)}}\left(\bwx_j^{(r)}\right)\|^2 \\
&\leq (1-\gamma)\eta_0^2H^2G^2 \\
&= (1-\gamma)\tilde{\eta}_0^2A
\end{align*}
Here (a) holds \text{since} $\bx_{t_{(1)}-1}=\bx_0=\bwx_{0}^{(r)}$.
It is easy to verify that $(1-\gamma)\tilde{\eta}_0^2A \leq \tfrac{4a\gamma(1-\gamma^2)}{a\gamma-4H}\tfrac{\tilde{\eta}_{t_{(1)}}^2}{\gamma^2}A$. 
To show this, we use $\tfrac{\tilde{\eta}_0}{\tilde{\eta}_{t_{(1)}}}=\tfrac{a+t_{(1)}}{a}\leq\tfrac{a+H}{a}\leq 2$, where the first inequality follows from $t_{(1)}\leq H$ and the second inequality follows from $a\geq H$.
Now, since $C\geq \tfrac{4a\gamma(1-\gamma^2)}{a\gamma-4H}$, it follows that $\mathbb{E}\|m_{t_{(1)}}^{(r)}\|^2\leq C\tfrac{\tilde{\eta}_{t_{(1)}}^2}{\gamma^2}A$. \\

\textit{ Inductive case:}
Assume $\mathbb{E}\|m_{(i)}^{(r)}\|^2\leq C\tfrac{\Tilde{\eta}_{t_{(i)}}^2}{\gamma^2}A$ for some $i\in\mathbb{Z}^+$. We need to show that
$\mathbb{E}\|m_{(i+1)}^{(r)}\|^2\leq C\tfrac{\Tilde{\eta}_{t_{(i+1)}}^2}{\gamma^2}A$. Using the inductive hypothesis in \eqref{eq:bounded-memory-interim2}, we get
\begin{align}
    \mathbb{E}\|m_{(i+1)}^{\left(r\right)}\|^2&\leq\left(1-\tfrac{\gamma}{p}\right)C\tfrac{\Tilde{\eta}_{t_{(i)}}^2}{\gamma^2}A+\tfrac{p(1-\gamma^2)}{(p-1)\gamma}\Tilde{\eta}_{t_{(i)}}^2A\notag\\
    &=C\tfrac{\Tilde{\eta}_{t_{(i)}}^2}{\gamma^2}A\left(1-\tfrac{\gamma}{p}+\tfrac{p(1-\gamma^2)}{p-1}\tfrac{\gamma}{C}\right)\notag\\
    &=C\tfrac{\Tilde{\eta}_{t_{(i)}}^2}{\gamma^2}A\left(1-\tfrac{\gamma}{p}\left(1-\tfrac{p^2(1-\gamma^2)}{(p-1)C}\right)\right) \label{eq:bounded-memory-induction-interim1}
\end{align}
\begin{claim}\label{claim:bounded-memory_interim-claim}
For any $p>1$, if $\tfrac{\gamma}{p}\left(1-\tfrac{p^2(1-\gamma^2)}{(p-1)C}\right)\geq\tfrac{2H}{a}$, then $\Tilde{\eta}_{t_{(i)}}^2\left(1-\tfrac{\gamma}{p}\left(1-\tfrac{p^2(1-\gamma^2)}{(p-1)C}\right)\right)\leq \Tilde{\eta}_{t_{(i+1)}}^2$ holds.
\end{claim}
\begin{proof}
Let $\tfrac{\gamma}{p}\left(1-\tfrac{p^2(1-\gamma^2)}{(p-1)C}\right)=\tfrac{\beta}{a}$.
Since $t_{(i+1)}\leq t_{(i)}+H$ (which implies that $\tilde{\eta}_{t_{(i)}+H}^2\leq\tilde{\eta}_{t_{(i+1)}}^2$), it suffices to show that
$\Tilde{\eta}_{t_{(i)}}^2\left(1-\tfrac{\beta}{a}\right)\leq \Tilde{\eta}_{t_{(i)}+H}^2$ holds whenever $\beta\geq2H$. 
For simplicity of notation, let $t=t_{(i)}$.
Note that $\tilde{\eta}_t^2\left(1-\tfrac{\beta}{a}\right)=\tfrac{(a-\beta)}{a(a+t)^2}$. We show below that if $\beta>2H$, then $a(a+t)^2\geq(a+t+H)^2(a-\beta)$. This proves our claim, because now we have $\tfrac{(a-\beta)}{a(a+t)^2}\leq\tfrac{(a-\beta)}{(a+t+H)^2(a-\beta)}=\tfrac{1}{(a+t+H)^2}=\tilde{\eta}_{t+H}^2$. It only remains to show that $a(a+t)^2\leq(a+t+H)^2(a-\beta)$ holds if $\beta\geq 2H$.
\begin{align*}
    (a+t+H)^2(a-\beta)&= \left((a+t)^2+H^2+2H(a+t)\right)(a-\beta)\notag\\
    &=a(a+t)^2+aH^2+2Ha^2+2Hat-\beta(a+t)^2-\beta H^2-2H\beta(a+t)\notag\\
    &= a(a+t)^2+a(H^2+2Ht-2\beta t-2H\beta)+a^2(2H-\beta)\notag\\
    &\hspace{2cm}-\beta t^2-\beta H^2-2H\beta t \notag \\
    &\leq a(a+t)^2.
\end{align*}
The last inequality holds whenever $\beta\geq2H$.
\end{proof}

Therefore we need $\tfrac{\gamma}{p}\left(1-\tfrac{p^2(1-\gamma^2)}{(p-1)C}\right)\geq \tfrac{2H}{a}$, which is equivalent to requiring $C\geq \tfrac{\gamma a p^2(1-\gamma^2)}{(p-1)(a\gamma-2pH)}$, where $a>\tfrac{2pH}{\gamma}$. Since this holds for every $p>1$, by substituting $p=2$, we get $C\geq \tfrac{4\gamma a(1-\gamma^2)}{(a\gamma-4H)}$. This together with \eqref{eq:bounded-memory-induction-interim1} and \Claimref{bounded-memory_interim-claim} implies that if $C\geq \tfrac{4\gamma a(1-\gamma^2)}{(a\gamma-4H)}$, where $a>4H/\gamma$, then $\mathbb{E}\|m_{(i+1)}^{(r)}\|^2\leq C\tfrac{\Tilde{\eta}_{t_{(i+1)}}^2}{\gamma^2}A$ holds. This proves our inductive step.

We have shown that $\mathbb{E}\|m_{t}^{\left(r\right)}\|^2 \leq 4C\tfrac{\tilde{\eta}_{t}^2}{\gamma^2}A$ holds when $t\in\I_T$.
It only remains to show that $\mathbb{E}\|m_{t}^{\left(r\right)}\|^2 \leq 4C\tfrac{\tilde{\eta}_{t}^2}{\gamma^2}A$ also holds when $t\in[T]\setminus\I_T$.
Let $i\in\bbZ_+$ be such that $t_{(i)}\leq t< t_{(i+1)}$, which implies that $\tilde{\eta}_{t_{(i)}} \leq 2\tilde{\eta}_t$. Since local memory does not change in between the synchronization indices, we have that $m_t^{(r)}=m_{t_{(i)}}^{(r)}$. 
Thus we have $\mathbb{E}\|m_t^{(r)}\|^2 = \mathbb{E}\|m_{t_{(i)}}^{(r)}\|^2 \leq C\tfrac{\Tilde{\eta}_{t_{(i)}}^2}{\gamma^2}A \leq 4C\tfrac{\Tilde{\eta}_{{t}}^2}{\gamma^2}A$. This concludes the proof of \Lemmaref{bounded-memory-decaying-lr}.
\end{proof}

\subsection{\noindent Proof of \Lemmaref{bounded-memory-fixed-lr}}\label{app:mem_fix}
\begin{lemma*}[Restating \Lemmaref{bounded-memory-fixed-lr}]
Let $gap(\mathcal{I}_T)\leq H$. Then the following holds for every worker $r\in[R]$ and for every $t\in \mathbb{Z}^+$:
\begin{align*}
    \mathbb{E}\|m_t^{(r)}\|_2^2\leq 4\frac{\eta^2(1-\gamma^2)}{\gamma^2}H^2G^2. 
\end{align*}
\end{lemma*}
\begin{proof}
Observe that \eqref{eq:bounded-memory-interim1} holds irrespective of the learning rate schedule. 
In particular, using a fixed learning rate $\eta_t = \eta$ for every $t$ gives
\begin{align}
    \mathbb{E}\|m_{t_{(i+1)}}^{\left(r\right)}\|^2&\leq\left(1-\frac{\gamma}{p}\right)\mathbb{E}\|m_{t_{(i)}}^{\left(r\right)}\|^2+\frac{p(1-\gamma^2)}{(p-1)\gamma}\eta^2H^2G^2\notag
\end{align}
When rolled out we see that the memory is upper bounded by a geometric sum. 
\begin{align}
    \mathbb{E}\|m_{t_{(i+1)}}^{\left(r\right)}\|^2&\leq\frac{p(1-\gamma^2)}{(p-1)\gamma}\eta^2H^2G^2\sum_{j=0}^{\infty}\left(1-\frac{\gamma}{p}\right)^j\notag\\
    &\leq \frac{p^2(1-\gamma^2)}{(p-1)}\frac{\eta^2}{\gamma^2}H^2G^2.\notag
\end{align}
Note that the last inequality holds for every $p>1$, and is minimized when $p=2$. By plugging $p=2$, we get
\begin{align*}
    \mathbb{E}\|m_{t_{(i+1)}}^{\left(r\right)}\|^2    &\leq \frac{4(1-\gamma^2)\eta^2}{\gamma^2}H^2G^2.
\end{align*}
Since the RHS does not depend on $t$, it follows that $\mathbb{E}\|m_{t}^{\left(r\right)}\|^2 \leq \frac{4(1-\gamma^2)\eta^2}{\gamma^2}H^2G^2$ holds for every $t\in[T]$.
This completes the proof of \Lemmaref{bounded-memory-fixed-lr}.
\end{proof}

\subsection{\noindent Proof of \Lemmaref{memory-maintenance}}\label{app:mem_main}
\begin{lemma*}[Restating \Lemmaref{memory-maintenance}]
Let $\btx_t^{(r)},m_t^{(r)}$, $r\in[R]$, $t\geq0$ be generated according to \Algorithmref{memQSGD-synchronous} and let $\bwx_t^{(r)}$ be as defined in \eqref{eq:virtual_seq_defn}. Let $\btx_t=\frac{1}{R}\sum_{r=1}^R\btx_t^{(r)}$ and $\bwx_t=\frac{1}{R}\sum_{r=1}^R\bwx_t^{(r)}$. 
Then we have 
\begin{align*}
\widehat{\mathbf{x}}_t-\widetilde{\mathbf{x}}_t\ =\ \frac{1}{R}\sum_{r=1}^Rm_t^{(r)},
\end{align*}
i.e., the difference of the true and the virtual sequence is equal to the average memory.
\end{lemma*}
\begin{proof}
Now consider $\widehat{\mathbf{x}}_t-\widetilde{\mathbf{x}}_t=\frac{1}{R}\sum_{r=1}^R\widehat{\mathbf{x}}_t^{(r)}-\widetilde{\mathbf{x}}_t^{(r)}$.
For the nearest $t_r+1\in\mathcal{I}_T$ such that $t_r+1\leq t$ and the nearest $t_r'+1\in\mathcal{I}_T$ such that $t_r'+1\leq t_r$
\begin{align}
\widehat{\mathbf{x}}_t-\widetilde{\mathbf{x}}_t&=\frac{1}{R}\sum_{r=1}^R\left(\widehat{\mathbf{x}}_{t_r+1}^{(r)}-\widetilde{\mathbf{x}}_{t_r+1}^{(r)}\right)\notag\\
&=\frac{1}{R}\sum_{r=1}^R\left(\mathbf{x}_{t_r}-\frac{1}{R}\sum_{r=1}^Rg_{t_r}^{(r)}-(\widetilde{\mathbf{x}}_{t_r'+1}^{(r)}-(\widehat{\mathbf{x}}_{t_r'+1}^{(r)}-\widehat{\mathbf{x}}_{t_r+\frac{1}{2}}^{(r)}))\right)
\end{align}
Here we used that $\widehat{\mathbf{x}}_{t_r'+1}^{(r)}-\widehat{\mathbf{x}}_{t_r+\frac{1}{2}}^{(r)}=\overset{t_r}{\underset{j=t_r'+1}{\sum}}\eta_j\nabla^{\left(r\right)}f_{\left(i_j\right)}\left(\widehat{\mathbf{x}}_j^{\left(r\right)}\right) $. Substituting $\widehat{\mathbf{x}}_{t_r'+1}^{(r)}=\mathbf{x}_{t_r'+1}$ we get
\begin{align}
\widehat{\mathbf{x}}_t-\widetilde{\mathbf{x}}_t&=\frac{1}{R}\sum_{r=1}^R\left(\mathbf{x}_{t_r}-\frac{1}{R}\sum_{r=1}^Rg_{t_r}^{(r)}-(\widetilde{\mathbf{x}}_{t_r'+1}^{(r)}-(\mathbf{x}_{t_r'+1}-\widehat{\mathbf{x}}_{t_r+\frac{1}{2}}^{(r)}))\right)\notag\\
&=\mathbf{x}_{t_r'+1}-\frac{1}{R}\sum_{r=1}^Rg_{t_r}^{(r)}-(\widetilde{\mathbf{x}}_{t_r'+1}-(\mathbf{x}_{t_r'+1}-\widehat{\mathbf{x}}_{t_r+\frac{1}{2}}))\notag\\
&=\widehat{\mathbf{x}}_{t_r'+1}-\widetilde{\mathbf{x}}_{t_r'+1}+(\mathbf{x}_{t_r'+1}-\widehat{\mathbf{x}}_{t_r+\frac{1}{2}})-\frac{1}{R}\sum_{r=1}^Rg_{t_r}^{(r)}
\end{align}
Now since $\mathbf{x}_{t_r'+1}=\mathbf{x}_{t_r}$ we have 
\begin{align}
\widehat{\mathbf{x}}_t-\widetilde{\mathbf{x}}_t&=\widehat{\mathbf{x}}_{t_r'+1}-\widetilde{\mathbf{x}}_{t_r'+1}+(\mathbf{x}_{t_r}-\widehat{\mathbf{x}}_{t_r+\frac{1}{2}})-\frac{1}{R}\sum_{r=1}^Rg_{t_r}^{(r)}\label{eq:sync_memory}
\end{align}
On rolling out the expression in \eqref{eq:sync_memory} we get 
\begin{align}
    \widehat{\mathbf{x}}_t-\widetilde{\mathbf{x}}_t&=\frac{1}{R}\sum_{r=1}^R\left[\sum_{\substack{j:j+1\in\mathcal{I}_T\\j\leq t_r}}\left(\mathbf{x}_j^{(r)}-\widehat{\mathbf{x}}_{j+\frac{1}{2}}^{(r)}-g_j^{(r)}\right)\right]\notag\\
    &=\frac{1}{R}\sum_{r=1}^Rm_{t_r+1}^{(r)} \notag\\
    &=\frac{1}{R}\sum_{r=1}^R m_t^{(r)}
\end{align}
Therefore $\widehat{\mathbf{x}}_t-\widetilde{\mathbf{x}}_t=\frac{1}{R}\sum_{r=1}^Rm_t^{\left(r\right)}$ is the average memory.
This completes the proof of \Lemmaref{memory-maintenance}.
\end{proof}

\subsection{Proof of \Lemmaref{sequence-bound-synchronous-fixed-lr}}\label{app:dev_fix}
\begin{lemma*}[Restating \Lemmaref{sequence-bound-synchronous-fixed-lr}]
Let $gap(\mathcal{I}_T) \leq H$.
For $\widehat{\bx}_t^{(r)}$ generated according to \Algorithmref{memQSGD-synchronous} with a fixed learning rate $\eta$ and letting $\bwx_t=\frac{1}{R}\sum_{r=1}^R\bwx_t^{(r)}$, we have the following bound on the deviation of the local sequences:
\begin{align*}
    \frac{1}{R}\sum_{r=1}^R\mathbb{E}\|\widehat{\mathbf{x}}_t-\widehat{\mathbf{x}}_t^{\left(r\right)}\|_2^2\ \leq\ \eta^2G^2H^2.
\end{align*}
\end{lemma*}
\begin{proof}
To prove this, we follow the proof of \Lemmaref{sequence-bound-synchronous} until \eqref{eq:deviation_bound_pre} and put $\eta_{t_r}=\eta$ to get $\frac{1}{R}\sum_{r=1}^R\mathbb{E}\|\widehat{\mathbf{x}}_t-\widehat{\mathbf{x}}_t^{\left(r\right)}\|^2 \leq \eta^2G^2H^2$.
\end{proof}

\subsection{Proof of \Lemmaref{sequence-bound-synchronous}}\label{app:dev_dec}
\begin{lemma*}[Restating \Lemmaref{sequence-bound-synchronous}]
Let $gap(\mathcal{I}_T) \leq H$.
By running \Algorithmref{memQSGD-synchronous} with a decaying learning rate $\eta_t$, we have
\begin{align*}
    \frac{1}{R}\sum_{r=1}^R\mathbb{E}\|\widehat{\mathbf{x}}_t-\widehat{\mathbf{x}}_t^{\left(r\right)}\|_2^2\ \leq\ 4\eta_t^2G^2H^2.
\end{align*}
\end{lemma*}
\begin{proof}
We show this along the lines of the proof of \cite[Lemma 3.3]{localsgd2}.
We need to upper-bound $\frac{1}{R}\sum_{r=1}^R\mathbb{E}\|\widehat{\mathbf{x}}_t-\widehat{\mathbf{x}}_t^{(r)}\|^2$. 
Note that for any $R$ vectors $\mathbf{u}_1,\hdots,\mathbf{u}_R$, if we let $\bar{\mathbf{u}}=\frac{1}{R}\sum_{i=1}^r\mathbf{u}_i$, then
$\sum_{i=1}^n\|\mathbf{u}_i-\bar{\mathbf{u}}\|^2\leq\sum_{i=1}^R\|\mathbf{u}_i\|^2$. We use this in the first inequality below.
\begin{align}
    \frac{1}{R}\sum_{r=1}^R\mathbb{E}\|\widehat{\mathbf{x}}_t-\widehat{\mathbf{x}}_t^{\left(r\right)}\|^2 &= \frac{1}{R}\sum_{r=1}^R\mathbb{E}\|\widehat{\mathbf{x}}_t^{(r)}-\widehat{\mathbf{x}}_{t_r}^{(r)}-(\widehat{\mathbf{x}}_t-\widehat{\mathbf{x}}_{t_r}^{(r)})\|^2\notag\\
    &\leq \frac{1}{R}\sum_{r=1}^R\mathbb{E}\|\widehat{\mathbf{x}}_t^{(r)}-\widehat{\mathbf{x}}_{t_r}^{(r)}\|^2\notag\\
    &\leq \eta_{t_r}^2G^2H^2 \label{eq:deviation_bound_pre} \\
    &\leq 4\eta_t^2G^2H^2\label{eq:deviation_bound}
\end{align}
The last inequality \eqref{eq:deviation_bound} uses $\eta_{t_r}\leq 2\eta_{t_r+H}\leq2\eta_t$ and $t-t_r\leq H$.
\end{proof}

\subsection{Proof of \Theoremref{convergence-non-convex-fixed-local-het}}\label{app:smooth_proof_sync}
\begin{proof}
Let $\bx^*$ be the minimizer of $f(\bx)$, therefore we denote $f(\bx^*)$ by $f^*$. For the purpose of reusing the proof later while proving \Theoremref{convergence-non-convex-decay-local-het}, we start off with the decaying learning rate $\eta_t$ until \eqref{eq:non-convex3} and then switch to the fixed learning rate $\eta$. Note that the proof remains the same until \eqref{eq:non-convex3} irrespective of the learning rate schedule; in particular, we can take $\eta_t=\eta$ and the same proof holds until \eqref{eq:non-convex3}. 

\noindent By the definition of $L$-smoothness, we have
{\allowdisplaybreaks
\begin{align*}
    f(\widetilde{\mathbf{x}}_{t+1})-f(\widetilde{\mathbf{x}}_{t})&\leq \langle\nabla f(\widetilde{\mathbf{x}}_t),\widetilde{\mathbf{x}}_{t+1}-\widetilde{\mathbf{x}}_t\rangle+\frac{L}{2}\|\widetilde{\mathbf{x}}_{t+1}-\widetilde{\mathbf{x}}_t\|^2\\
    &=-\eta_t\langle\nabla f(\widetilde{\mathbf{x}}_t),\mathbf{p}_t\rangle+\frac{\eta_t^2L}{2}\|\mathbf{p}_t\|^2\\
    &=-\eta_t\langle\nabla f(\widetilde{\mathbf{x}}_t),\mathbf{p}_t\rangle+\frac{\eta_t^2L}{2}\|\mathbf{p}_t-\overline{\mathbf{p}}_t+\overline{\mathbf{p}}_t\|^2\\
    &\leq-\eta_t\langle\nabla f(\widetilde{\mathbf{x}}_t),\mathbf{p}_t\rangle+\eta_t^2L\|\mathbf{p}_t-\overline{\mathbf{p}}_t\|^2+\eta_t^2L\|\overline{\mathbf{p}}_t\|^2\quad\text{(Using Jensen's Inequality)}\\
    &=-\frac{\eta_t}{R}\sum_{r=1}^{R}\langle\nabla f(\widetilde{\mathbf{x}}_t),\nabla f_{i_t^{(r)}}(\widehat{\mathbf{x}}_t^{(r)})\rangle 
    +\eta_t^2L\|\frac{1}{R}\sum_{r=1}^{R}\nabla f^{(r)}(\widehat{\mathbf{x}}_t^{(r)})\|^2+\eta_t^2L\|\mathbf{p}_t-\overline{\mathbf{p}}_t\|^2
    \end{align*}
    Define $i_t$ as the set of random sampling of the mini-batches at each worker $\{i_t^{(1)},i_t^{(2)},\ldots,i_t^{(R)}\}$. Taking expectation w.r.t.~the sampling at time $t$ (conditioned on the past) and using the lipschitz continuity of the gradients of local functions gives
    \begin{align}
  \mathbb{E}_{i_t}[f(\widetilde{\mathbf{x}}_{t+1})]-f(\widetilde{\mathbf{x}}_t)&\leq -\frac{\eta_t}{2}\left(\|\nabla f(\widetilde{\mathbf{x}}_t)\|^2+\|\frac{1}{R}\sum_{r=1}^{R}\nabla f^{(r)}(\widehat{\mathbf{x}}_t^{(r)})\|^2-\|\nabla f(\widetilde{\mathbf{x}}_t)-\frac{1}{R}\sum_{r=1}^{R}\nabla f^{(r)}(\widehat{\mathbf{x}}_t^{(r)})\|^2\right)\notag\\
    &\hspace{2cm}+\eta_t^2L\|\frac{1}{R}\sum_{r=1}^{R}\nabla f^{(r)}(\widehat{\mathbf{x}}_t^{(r)})\|^2+\frac{\eta_t^2L}{bR^2}\sum_{r=1}^R\sigma_r^2\notag\\
    &\leq -\frac{\eta_t}{2R}\sum_{r=1}^{R}\left(\|\nabla f(\widetilde{\mathbf{x}}_t)\|^2-L^2\|\widetilde{\mathbf{x}}_t-\widehat{\mathbf{x}}_t^{(r)}\|^2\right)+\frac{2\eta_t^2L-\eta_t}{2}\|\frac{1}{R}\sum_{r=1}^{R}\nabla f^{(r)}(\widehat{\mathbf{x}}_t^{(r)})\|^2\notag\\&\hspace{2cm}+\frac{\eta_t^2L}{bR^2}\sum_{r=1}^R\sigma_r^2\notag\\
    &= -\frac{\eta_t}{2R}\sum_{r=1}^{R}\left(\|\nabla f(\btx_t)\|^2+L^2\|\btx_t-\widehat{\bx}_t^{(r)}\|^2\right)+\frac{2\eta_t^2L-\eta_t}{2R}\sum_{r=1}^R\|\nabla f(\widehat{\bx}_t^{(r)})\|^2 \notag \\
    &\hspace{2cm}+\frac{\eta_t^2L}{bR^2}\sum_{r=1}^R\sigma_r^2+\frac{\eta_t L^2}{R}\sum_{r=1}^R\|\btx_t-\widehat{\bx}_t^{(r)}\|^2. \label{eq:non-convex1}
    \end{align}
    }
We bound the first term in terms of $\|\nabla f(\widehat{\bx}_t^{(r)})\|^2$ as follows:
\begin{align}
    \|\nabla f(\widehat{\bx}_t^{(r)})\|^2 &\leq 2\|\nabla f(\widehat{\bx}_t^{(r)})-\nabla f(\btx_t)\|^2 + 2\|\nabla f(\btx_t)\|^2 \notag \\
    &\leq 2L^2\|\widehat{\bx}_t^{(r)}-\btx_t\|^2+2\|\nabla f(\btx_t)\|^2 \label{eq:non-convex2},
\end{align}
where the 2nd inequality follows from the smoothness ($L$-Lipschitz gradient) assumption.
Using this and that $\eta_t\leq \frac{1}{2L}$ in \eqref{eq:non-convex1} and rearranging terms give
\begin{align}
    \frac{\eta_t}{4R}\sum_{r=1}^R\|\nabla f(\widehat{\bx}_t^{(r)})\|^2\leq f(\btx_t)-\mathbb{E}_{(i_t)}[f(\btx_{t+1})]+\frac{\eta_t^2L}{bR^2}\sum_{r=1}^R\sigma_r^2+\frac{\eta_tL^2}{R}\sum_{r=1}^R\|\btx_t-\widehat{\bx}_t^{(r)}\|^2
\end{align}
Taking expectation w.r.t.~to the entire process and using the inequality $\|\bu+\bv\|^2\leq 2\|\bu\|^2 + 2\|\bv\|^2$ gives
\begin{align}
    \frac{\eta_t}{4R}\sum_{r=1}^R\mathbb{E}\|\nabla f(\widehat{\bx}_t^{(r)})\|^2 &\leq 
\mathbb{E}[f(\btx_t)]-\mathbb{E}[f(\btx_{t+1})]+\frac{\eta_t^2L}{bR^2}\sum_{r=1}^R\sigma_r^2+2\eta_t L^2\mathbb{E}\|\btx_t-\widehat{\bx}_t\|^2 \notag \\
&\hspace{2cm} +2\eta_t L^2\frac{1}{R}\sum_{r=1}^R\mathbb{E}\|\widehat{\bx}_t-\widehat{\bx}_t^{(r)}\|^2 \label{eq:non-convex3}
\end{align}
Observe that \eqref{eq:non-convex3} holds irrespective of the learning rate schedule. In particular, 
if we take a fixed learning rate $\eta_t=\eta\leq\frac{1}{2L}$ in \eqref{eq:non-convex3}, we get
\begin{align}
    \frac{\eta}{4R}\sum_{r=1}^R\mathbb{E}\|\nabla f(\widehat{\bx}_t^{(r)})\|^2&\leq \mathbb{E}[f(\btx_t)]-\mathbb{E}[f(\btx_{t+1})]+\frac{\eta^2L}{bR^2}\sum_{r=1}^R\sigma_r^2+2\eta L^2\mathbb{E}\|\btx_t-\widehat{\bx}_t\|^2 \notag \\
    &+2\eta L^2\frac{1}{R}\sum_{r=1}^R\mathbb{E}\|\widehat{\bx}_t-\widehat{\bx}_t^{(r)}\|^2\label{eq:loc22}
\end{align}
\Lemmaref{memory-maintenance} and \Lemmaref{bounded-memory-fixed-lr} together imply $\mathbb{E}\|\widehat{\bx}_t-\btx_t\|^2 \leq \frac{4\eta^2(1-\gamma^2)}{\gamma^2}G^2H^2$.
We also have from \Lemmaref{sequence-bound-synchronous-fixed-lr} that $\frac{1}{R}\sum_{r=1}^R\mathbb{E}\|\widehat{\bx}_t-\widehat{\bx}_t^{\left(r\right)}\|^2 \leq \eta^2G^2H^2$. Substituting these  in \eqref{eq:loc22} gives
\begin{align}
    \frac{\eta}{4R}\sum_{r=1}^R\mathbb{E}\|\nabla f(\widehat{\bx}_t^{(r)})\|^2&\leq \mathbb{E}[f(\btx_t)]-\mathbb{E}[f(\btx_{t+1})]+\frac{\eta^2L}{bR^2}\sum_{r=1}^R\sigma_r^2+8\frac{\eta^3(1-\gamma^2)}{\gamma^2} L^2G^2H^2\notag\\
    &+2\eta^3 L^2G^2H^2
\end{align}
By taking a telescopic sum from $t=0$ to $t=T-1$, we get 
\begin{align}
    \frac{1}{4RT}\sum_{t=0}^{T-1}\sum_{r=1}^R\mathbb{E}\|\nabla f(\widehat{\bx}_t^{(r)})\|^2&\leq \frac{\mathbb{E}[f(\btx_0)]-f^*}{\eta T}+\frac{\eta L}{bR^2}\sum_{r=1}^R\sigma_r^2+8\frac{\eta^2(1-\gamma^2)}{\gamma^2} L^2G^2H^2\notag\\
    &+2\eta^2 L^2G^2H^2
\end{align}
Take $\eta=\frac{\widehat{C}}{\sqrt{T}}$, where $\widehat{C}$ is a constant (that satisfies $\widehat{C}<\frac{\sqrt{T}}{2L}$). 
For example, we can take $\widehat{C}=\frac{1}{2L}$. This gives
\begin{align}
    \frac{1}{RT}\sum_{t=0}^{T-1}\sum_{r=1}^R\mathbb{E}\|\nabla f(\widehat{\bx}_t^{(r)})\|^2&\leq \left(\frac{\mathbb{E}[f(\bx_0)]-f^*}{\widehat{C}}+\frac{\widehat{C} L}{bR^2}\sum_{r=1}^R\sigma_r^2\right)\frac{4}{\sqrt{T}}+8\left(4\frac{(1-\gamma^2)}{\gamma^2}+1\right) \frac{\widehat{C}^2L^2G^2H^2}{T}. \label{eq:non-convex-fixed-lr23}
\end{align}
Sample a parameter $\bz_T$ from $\left\{\bwx_t^{(r)}\right\}$ for $r=1,\hdots,R$ and $t=0,1,\hdots,T-1$ 
with probability $\Pr[\bz_T = \bwx_t^{(r)}]=\frac{1}{RT}$, which implies $\mathbb{E}\|\bz_T\|^2=\frac{1}{RT}\sum_{t=0}^{T-1}\sum_{r=1}^R\mathbb{E}\|\nabla f(\widehat{\bx}_t^{(r)})\|^2$.
Using this in \eqref{eq:non-convex-fixed-lr23} gives
\[\mathbb{E}\|\bz_T\|^2=\left(\frac{\mathbb{E}[f(\bx_0)]-f^*}{\widehat{C}}+\frac{\widehat{C} L}{bR^2}\sum_{r=1}^R\sigma_r^2\right)\frac{4}{\sqrt{T}}+8\left(4\frac{(1-\gamma^2)}{\gamma^2}+1\right) \frac{\widehat{C}^2L^2G^2H^2}{T}.\]
This completes the proof of \Theoremref{convergence-non-convex-fixed-local-het}.
\end{proof}

\subsection{Proof of \Theoremref{convergence-non-convex-decay-local-het}}\label{app:proof_convergence-non-convex-decay-local-het}
\begin{proof}
Observe that we can use the proof of \Theoremref{convergence-non-convex-fixed-local-het} exactly until \eqref{eq:non-convex3}, for $\eta_t\leq \frac{1}{2L}$ (which follows from our assumption that $a\geq 2\xi L$), which gives 
\begin{align}
    \frac{\eta_t}{4R}\sum_{r=1}^R\mathbb{E}\|\nabla f(\widehat{\bx}_t^{(r)})\|^2 &\leq 
\mathbb{E}[f(\btx_t)]-\mathbb{E}[f(\btx_{t+1})]+\frac{\eta_t^2L}{bR^2}\sum_{r=1}^R\sigma_r^2+2\eta_t L^2\mathbb{E}\|\btx_t-\widehat{\bx}_t\|^2 \notag \\
&\hspace{2cm} +2\eta_t L^2\frac{1}{R}\sum_{r=1}^R\mathbb{E}\|\widehat{\bx}_t-\widehat{\bx}_t^{(r)}\|^2 \label{eq:decaying-non-convex3}
\end{align}
We have from \Lemmaref{sequence-bound-synchronous} that $\frac{1}{R}\sum_{r=1}^R\mathbb{E}\|\widehat{\bx}_t-\widehat{\bx}_t^{\left(r\right)}\|^2\leq 4\eta_t^2G^2H^2$. \Lemmaref{memory-maintenance} and \Lemmaref{bounded-memory-decaying-lr} together imply that $\mathbb{E}\|\widehat{\bx}_t-\btx_t\|^2\leq \frac{1}{R}\sum_{r=1}^R\|m_t^{\left(r\right)}\|^2\leq C\frac{4\eta_{t}^2}{\gamma^2}G^2H^2$.
Using these bounds in \eqref{eq:decaying-non-convex3} gives
\begin{align*}
    \frac{\eta_t}{4R}\sum_{r=1}^R\mathbb{E}\|\nabla f(\widehat{\bx}_t^{(r)})\|^2&\leq \mathbb{E}[f(\btx_t)]-\mathbb{E}[f(\btx_{t+1})]+\frac{\eta_t^2L}{bR^2}\sum_{r=1}^R\sigma_r^2+\frac{8\eta_{t}^3}{\gamma^2}CL^2G^2H^2+8\eta_t^3L^2G^2H^2
\end{align*}
Taking a telescopic sum from $t=0$ to $t=T-1$ gives
\begin{align}
    \sum_{t=0}^{T-1}\frac{\eta_t}{4R}\sum_{r=1}^R\mathbb{E}\|\nabla f(\widehat{\bx}_t^{(r)})\|^2\leq \mathbb{E}[f(\bx_{0})]-f^*+\frac{L\sum_{r=1}^R\sigma_r^2}{bR^2}\sum_{t=0}^{T-1}\eta_t^2+\left(\frac{8C}{\gamma^2}+8\right)L^2G^2H^2\sum_{t=0}^{T-1}\eta_t^3. \label{eq:non-convex4}
\end{align}
Let $\delta_t:=\frac{\eta_t}{4R}$ and $P_T:=\sum_{t=0}^{T-1}\sum_{r=1}^R\delta_t$. We show at the end of this proof that 
$P_T\geq \frac{\xi}{4}\ln{\left(\frac{T+a-1}{a}\right)}$, $\sum_{t=0}^{T-1}\eta_t^2 \leq \frac{\xi^2}{a-1}$, and that $\sum_{t=0}^{T-1}\eta_t^3\leq \frac{\xi^3}{2(a-1)^2}$. Using these in \eqref{eq:non-convex4} yields
\begin{align}
    \frac{1}{P_{T}}\sum_{t=0}^{T-1}\sum_{r=1}^R\delta_t\mathbb{E}\|\nabla f(\widehat{\bx}_t^{(r)})\|^2
    &\leq \frac{\mathbb{E}f(\bx_0)-f^*}{P_T}+\frac{L\xi^2}{bR^2(a-1)}\frac{\sum_{r=1}^R\sigma^2}{P_T}\notag\\
    &\hspace{2cm}+\left(\frac{8C}{\gamma^2}+8\right)L^2G^2H^2\frac{\xi^3}{2P_T(a-1)^2} \label{eq:non-convex5}
\end{align}
We therefore can show a weak convergence result, i.e.,
\begin{align}
    \min_{t\in\{0,\ldots,T-1\},\,r\in [R]}\mathbb{E}\|\nabla f(\widehat{\bx}_t^{(r)})\|^2\xrightarrow{T\rightarrow \infty} 0.
\end{align}
Sample a parameter $\bz_T$ from $\left\{\bwx_t^{(r)}\right\}$ for $r=1,\hdots,R$ and $t=0,1,\hdots,T-1$ 
with probability $\Pr[\bz_T = \bwx_t^{(r)}]=\frac{\delta_t}{P_T}$. 
This gives
$\mathbb{E}\|\nabla f(\bz_T)\|^2=\frac{1}{P_{T}}\sum_{t=0}^{T-1}\sum_{r=1}^R\delta_t\mathbb{E}\|\nabla f(\widehat{\bx}_t^{(r)})\|^2$. We therefore have the following from \eqref{eq:non-convex5}
\begin{align*}
    \mathbb{E}\|\nabla f(\bz_T)\|^2    &\leq \frac{\mathbb{E}f(\bx_0)-f^*}{P_T}+\frac{L\xi^2\sum_{r=1}^R\sigma^2}{bR^2(a-1)P_T} +\left(\frac{8C}{\gamma^2}+8\right)\frac{\xi^3L^2G^2H^2}{2(a-1)^2P_T}
\end{align*}
Since $\min_{t\in\{0,\ldots,T-1\},\,r\in [R]}\mathbb{E}\|\nabla f(\widehat{\bx}_t^{(r)})\|^2$, we have a weak convergence result:
\begin{align*}
\min_{t\in\{0,\ldots,T-1\},\,r\in [R]}\mathbb{E}\|\nabla f(\widehat{\bx}_t^{(r)})\|^2\xrightarrow{T\rightarrow \infty} 0.
\end{align*}
\noindent Bounding the terms $P_T$, $\sum_{t=0}^{T-1}\eta_t^2$ and $\sum_{t=0}^{T-1}\eta_t^3$:
\begin{align*}
    P_T&=\frac{1}{4}\sum_{t=0}^{T-1}\eta_t \geq \frac{1}{4}\sum_{t=0}^{T-1}\eta_t \geq \frac{\xi}{4}\ln{\left(\frac{T+a-1}{a}\right)}\\
    \sum_{t=0}^{T-1}\eta_t^2 &\leq \xi^2\left(\frac{1}{a-1}-\frac{1}{T+a-1}\right)=\frac{\xi^2T}{(a-1)(T+a-1)}\leq \frac{\xi^2}{a-1}\\
    \sum_{t=0}^{T-1}\eta_t^3 &\leq \frac{\xi^3}{2}\left(\frac{1}{(a-1)^2}-\frac{1}{(T+a-1)^2}\right)\leq\frac{\xi^3}{2(a-1)^2}
\end{align*}
This completes the proof of \Theoremref{convergence-non-convex-decay-local-het}.
\end{proof}

\subsection{Proof of \Theoremref{convergence-strongly-convex-local}}\label{app:convex_proof_sync}
\begin{proof}\label{main-proof}
Let $\bx^*$ be the minimizer of $f(\bx)$, therefore we have $\nabla f(\bx^*)=0$. We denote $f(\bx^*)$ by $f^*$. 
By taking the average of the virtual sequences $\widetilde{\mathbf{x}}_{t+1}^{(r)}=\widetilde{\mathbf{x}}_{t}^{(r)}-\eta_t\nabla f_{i_t^{(r)}}\left(\widehat{\mathbf{x}}_t^{(r)}\right)$ for each worker $r\in[R]$ and defining $\mathbf{p}_t:=\frac{1}{R}\sum_{r=1}^R\nabla f_{i_t^{(r)}}\left(\widehat{\mathbf{x}}_t^{(r)}\right)$, we get 
\begin{align}
\widetilde{\mathbf{x}}_{t+1} = \widetilde{\mathbf{x}}_t - \eta_t\mathbf{p}_t. \label{eq:global-virtual-seq}
\end{align}
Define $i_t$ as the set of random sampling of the mini-batches at each worker $\{i_t^{(1)},i_t^{(2)},\ldots,i_t^{(R)}\}$ and let $\overline{\mathbf{p}}_t=\mathbb{E}_{i_t}[\mathbf{p}_t]$. From \eqref{eq:global-virtual-seq} we can get
\begin{align}
    \|\widetilde{\mathbf{x}}_{t+1}-\mathbf{x}^*\|^2&=\|\widetilde{\mathbf{x}}_t-\mathbf{x}^*-\eta_t\overline{\mathbf{p}}_t\|^2+\eta_t^2\|\mathbf{p}_t-\overline{\mathbf{p}}_t\|^2-2\eta_t\left\langle\widetilde{\mathbf{x}}_t-\mathbf{x}^*-\eta_t\overline{\mathbf{p}}_t,\mathbf{p}_t-\overline{\mathbf{p}}_t\right\rangle\label{eq:loc-1}
\end{align}
Taking the expectation w.r.t.~the sampling $i_t$ at time $t$ (conditioning on the past) and noting that last term in \eqref{eq:loc-1} becomes zero gives:
\begin{align}
    \mathbb{E}_{i_t}\|\widetilde{\mathbf{x}}_{t+1}-\mathbf{x}^*\|^2&=\|\widetilde{\mathbf{x}}_t-\mathbf{x}^*-\eta_t\overline{\mathbf{p}}_t\|^2+\eta_t^2\mathbb{E}_{i_t}\|\mathbf{p}_t-\overline{\mathbf{p}}_t\|^2\label{eq:loc-interim1}
\end{align}
It follows from the Jensen's inequality and independence that $\mathbb{E}_{i_t}\|\mathbf{p}_t-\overline{\mathbf{p}}_t\|^2\leq \frac{\sum_{r=1}^R\sigma_r^2}{bR^2}$. This gives
\begin{align}
    \mathbb{E}_{i_t}\|\widetilde{\mathbf{x}}_{t+1}-\mathbf{x}^*\|^2 &\leq \|\widetilde{\mathbf{x}}_t-\mathbf{x}^*-\eta_t\overline{\mathbf{p}}_t\|^2+\eta_t^2\frac{\sum_{r=1}^R\sigma_r^2}{bR^2}. \label{eq:loc-interim2}
\end{align}    
Now we bound the first term on the RHS.
\begin{lemma}\label{lem:strong-main-proof-lem1}
If $\eta_t\leq \frac{1}{4L}$, then we have 
\begin{align}
    \|\widetilde{\mathbf{x}}_{t}-\mathbf{x}^*-\eta_t\overline{\mathbf{p}}_t\|^2&\leq \left(1-\frac{\mu\eta_t}{2}\right)\|\widetilde{\mathbf{x}}_t-\mathbf{x}^*\|^2-\frac{\eta_t\mu}{2L}(f(\widehat{\bx}_t)-f^*) \notag \\
    &+\eta_t\left(\frac{3\mu}{2}+3L\right)\|\widehat{\mathbf{x}}_t-\widetilde{\mathbf{x}}_t\|^2+\frac{3\eta_tL}{R}\sum_{r=1}^R\|\widehat{\mathbf{x}}_t-\widehat{\mathbf{x}}_t^{\left(r\right)}\|^2\label{eq:loc8}
\end{align}
\end{lemma}
\begin{proof}
\begin{align}
    \|\widetilde{\mathbf{x}}_t-\mathbf{x}^*-\eta_t\overline{\mathbf{p}}_t\|^2&=\|\widetilde{\mathbf{x}}_t-\mathbf{x}^*\|^2+\eta_t^2\|\overline{\mathbf{p}}_t\|^2-2\eta_t\left\langle\widetilde{\mathbf{x}}_t-\mathbf{x}^*,\overline{\mathbf{p}}_t\right\rangle\label{eq:loc0}
\end{align}
Using the definition of $\overline{\mathbf{p}}_t$ we have 
\begin{align}
    \|\overline{\mathbf{p}}_t\|^2&=\|\frac{1}{R}\sum_{r=1}^R\left(\nabla f^{(r)}\left(\widehat{\mathbf{x}}_t^{\left(r\right)}\right)-\nabla f^{(r)}(\widetilde{\bx}_t)\right)+\nabla f(\widetilde{\bx}_t)-\nabla f(\bx^*)\|^2\notag\\
    &\leq \frac{1}{R}\sum_{r=1}^R2\|\nabla f^{(r)}\left(\widehat{\mathbf{x}}_t^{\left(r\right)}\right)-\nabla f^{(r)}(\widetilde{\bx}_t)\|^2+2\|\nabla f(\widetilde{\bx}_t)-\nabla f\left(\bx^*\right)\|^2 \notag\\
    &\leq\frac{2L^2}{R}\sum_{r=1}^R\|\widehat{\bx}_t^{(r)}-\widetilde{\bx}_t\|+2\|\nabla f(\widetilde{\bx}_t)-\nabla f\left(\bx^*\right)\|^2\label{eq:loc_pbar}
\end{align}
By the definition of smoothness, we have
$\|\nabla f\left(\widetilde{\bx}_t\right)-\nabla f\left(\mathbf{x}^*\right)\|^2\leq 2L\left(f\left(\widetilde{\bx}_t\right)-f(\bx^*)\right)$, where $\nabla f(\mathbf{x}^*)=0$. Substituting this in \eqref{eq:loc_pbar} gives
\begin{align}
    \eta_t^2\|\overline{\mathbf{p}}_t\|^2\leq \frac{2\eta_t^2L^2}{R}\sum_{r=1}^R\|\widehat{\bx}_t^{(r)}-\widetilde{\bx}_t\|+4\eta_t^2L\left(f\left(\widetilde{\bx}_t\right)-f(\bx^*)\right)
    \label{eq:loc1}
\end{align}

Now we bound the last term of \eqref{eq:loc0}. By definition, we have
\begin{align}
    -2\eta_t\left\langle\widetilde{\mathbf{x}}_t-\mathbf{x}^*,\overline{\mathbf{p}}_t\right\rangle=-2\frac{\eta_t}{R}\sum_{r=1}^R\left\langle\widehat{\mathbf{x}}_t^{\left(r\right)}-\mathbf{x}^*,\nabla f^{(r)}\left(\widehat{\mathbf{x}}_t^{\left(r\right)}\right)\right\rangle-2\frac{\eta_t}{R}\sum_{r=1}^R\left\langle\widetilde{\mathbf{x}}_t-\widehat{\mathbf{x}}_t^{\left(r\right)},\nabla f^{(r)}\left(\widehat{\mathbf{x}}_t^{\left(r\right)}\right)\right\rangle\label{eq:loc2}
\end{align}
For the first term on the RHS of \eqref{eq:loc2}, we can use strong convexity
\begin{align}
    -2\left\langle\widehat{\mathbf{x}}_t^{\left(r\right)}-\mathbf{x}^*,\nabla f^{(r)}\left(\widehat{\mathbf{x}}_t^{\left(r\right)}\right)\right\rangle\leq -2\left(f^{(r)}\left(\widehat{\mathbf{x}}_t^{\left(r\right)}\right)-f^{(r)}(\bx^*)\right)-\mu\|\widehat{\mathbf{x}}_t^{\left(r\right)}-\mathbf{x}^*\|^2\label{eq:loc3}
\end{align}
For the second term on the RHS of \eqref{eq:loc2}, we can use the following by smoothness.
\begin{align}
-2\left\langle\widetilde{\mathbf{x}}_t-\widehat{\mathbf{x}}_t^{\left(r\right)},\nabla f^{(r)}\left(\widehat{\mathbf{x}}_t^{\left(r\right)}\right)\right\rangle
&\leq L\|\widetilde{\mathbf{x}}_t-\widehat{\mathbf{x}}_t^{\left(r\right)}\|^2+2\left(f^{(r)}\left(\widehat{\bx}_t^{(r)}\right)-f^{(r)}\left(\widetilde{\bx}_t\right)\right)
\label{eq:loc4}
\end{align}
Using \eqref{eq:loc3}-\eqref{eq:loc4} in \eqref{eq:loc2} we get
\begin{align}
    -2\eta_t\left\langle\widetilde{\mathbf{x}}_t-\mathbf{x}^*,\overline{\mathbf{p}}_t\right\rangle
    &\leq -\frac{2\eta_t}{R}\sum_{r=1}^R\left(f^{(r)}\left(\widetilde{\bx}_t\right)-f^{(r)}(\bx^*)\right)-\frac{\eta_t\mu}{R}\sum_{r=1}^R\|\widehat{\mathbf{x}}_t^{\left(r\right)}-\mathbf{x}^*\|^2+\frac{L\eta_t}{R}\sum_{r=1}^R\|\widetilde{\mathbf{x}}_t-\widehat{\mathbf{x}}_t^{\left(r\right)}\|^2\notag\\
    &= -2\eta_t\left(f\left(\widetilde{\bx}_t\right)-f(\bx^*)\right)-\frac{\eta_t\mu}{R}\sum_{r=1}^R\|\widehat{\mathbf{x}}_t^{\left(r\right)}-\mathbf{x}^*\|^2+L\frac{\eta_t}{R}\sum_{r=1}^R\|\widetilde{\mathbf{x}}_t-\widehat{\mathbf{x}}_t^{\left(r\right)}\|^2
    \label{eq:loc5}
\end{align}
Adding \eqref{eq:loc1} and \eqref{eq:loc5}  and using $a\geq 32L/\mu$ which implies $\eta_t\leq \nicefrac{1}{4L}$ yields 
\begin{align}
   \eta_t^2 \|\overline{\mathbf{p}}_t\|^2-2\eta_t\left\langle\widetilde{\mathbf{x}}_t-\mathbf{x}^*,\overline{\mathbf{p}}_t\right\rangle &\leq
   -2\eta_t(1-2\eta_tL)\left(f\left(\widetilde{\bx}_t\right)-f^*\right) -\frac{\eta_t\mu}{R}\sum_{r=1}^R\|\widehat{\mathbf{x}}_t^{\left(r\right)}-\mathbf{x}^*\|^2 \notag\\
&\hspace{2cm} +\frac{L\eta_t+2\eta_t^2L^2}{R}\sum_{r=1}^R\|\widetilde{\mathbf{x}}_t-\widehat{\mathbf{x}}_t^{\left(r\right)}\|^2\notag\\
&\leq -\eta_t\left(f\left(\widetilde{\bx}_t\right)-f^*\right) -\eta_t\mu\|\widehat{\mathbf{x}}_t-\mathbf{x}^*\|^2 \notag\\
&\hspace{2cm}+\frac{3L\eta_t}{R}\sum_{r=1}^R\left(\|\widetilde{\mathbf{x}}_t-\widehat{\mathbf{x}}_t\|^2+\|\widehat{\mathbf{x}}_t-\widehat{\mathbf{x}}_t^{\left(r\right)}\|^2\right)\label{eq:loc7}
\end{align}
Since $\|\mathbf{x}+\mathbf{y}\|^2\leq 2\|\mathbf{x}\|^2 + 2\|\mathbf{y}\|^2$, we have
\begin{align}
    -\|\widehat{\mathbf{x}}_t-\mathbf{x}^*\|^2\leq \|\widehat{\mathbf{x}}_t-\widetilde{\mathbf{x}}_t\|^2-\frac{1}{2}\|\widetilde{\mathbf{x}}_t-\mathbf{x}^*\|^2\label{eq:loc6}
\end{align}
Using \eqref{eq:loc6} in \eqref{eq:loc7} and then substituting \eqref{eq:loc7} in \eqref{eq:loc0} gives
\begin{align}
    \|\widetilde{\mathbf{x}}_{t}-\mathbf{x}^*-\eta_t\overline{\mathbf{p}}_t\|^2&\leq \left(1-\frac{\mu\eta_t}{2}\right)\|\widetilde{\mathbf{x}}_t-\mathbf{x}^*\|^2-\eta_t\left(f\left(\widetilde{\bx}_t\right)-f^*\right) \notag \\
    &+\eta_t\left(\mu+3L\right)\|\widehat{\mathbf{x}}_t-\widetilde{\mathbf{x}}_t\|^2+\frac{3L\eta_t}{R}\sum_{r=1}^R\|\widehat{\mathbf{x}}_t-\widehat{\mathbf{x}}_t^{\left(r\right)}\|^2\end{align}
    Using strong convexity of $f$ we have
    \begin{align}
    \|\widetilde{\mathbf{x}}_{t}-\mathbf{x}^*-\eta_t\overline{\mathbf{p}}_t\|^2&\leq\left(1-\frac{\mu\eta_t}{2}\right)\|\widetilde{\mathbf{x}}_t-\mathbf{x}^*\|^2-\frac{\eta_t\mu}{2}\|\widetilde{\bx}_t-\bx^*\|^2 \notag \\
    &+\eta_t\left(\mu+3L\right)\|\widehat{\mathbf{x}}_t-\widetilde{\mathbf{x}}_t\|^2+\frac{3L\eta_t}{R}\sum_{r=1}^R\|\widehat{\mathbf{x}}_t-\widehat{\mathbf{x}}_t^{\left(r\right)}\|^2
\end{align}
Now use $-\|\widetilde{\bx}_t-\bx^*\|^2\leq \|\widetilde{\bx}_t-\widehat{\bx}_t\|^2-\frac{1}{2}\|\widehat{\bx}_t-\bx^*\|^2$
We get 
\begin{align}
    \|\widetilde{\mathbf{x}}_{t}-\mathbf{x}^*-\eta_t\overline{\mathbf{p}}_t\|^2&\leq \left(1-\frac{\mu\eta_t}{2}\right)\|\widetilde{\mathbf{x}}_t-\mathbf{x}^*\|^2-\frac{\eta_t\mu}{4}\|\widehat{\bx}_t-\bx^*\|^2 \notag \\
    &+\eta_t\left(\frac{3\mu}{2}+3L\right)\|\widehat{\mathbf{x}}_t-\widetilde{\mathbf{x}}_t\|^2+\frac{3L\eta_t}{R}\sum_{r=1}^R\|\widehat{\mathbf{x}}_t-\widehat{\mathbf{x}}_t^{\left(r\right)}\|^2\notag\\
    &\leq \left(1-\frac{\mu\eta_t}{2}\right)\|\widetilde{\mathbf{x}}_t-\mathbf{x}^*\|^2-\frac{\eta_t\mu}{2L}(f(\widehat{\bx}_t)-f^*)\quad (\textrm{Using smoothness of }f(\bx)) \notag \\
    &+\eta_t\left(\frac{3\mu}{2}+3L\right)\|\widehat{\mathbf{x}}_t-\widetilde{\mathbf{x}}_t\|^2+\frac{3L\eta_t}{R}\sum_{r=1}^R\|\widehat{\mathbf{x}}_t-\widehat{\mathbf{x}}_t^{\left(r\right)}\|^2\label{eq:loc8}
\end{align}
This completes the proof of \Lemmaref{strong-main-proof-lem1}.
\end{proof}

Using \eqref{eq:loc8} in \eqref{eq:loc-interim2} and then taking the expectation over the entire process gives
\begin{align}
\mathbb{E}\|\widetilde{\mathbf{x}}_{t+1}-\mathbf{x}^*\|^2 &\leq \left(1-\frac{\mu\eta_t}{2}\right)\mathbb{E}\|\widetilde{\mathbf{x}}_t-\mathbf{x}^*\|^2-\frac{\eta_t\mu}{2L}(\mathbb{E}[f(\widehat{\bx}_t)]-f^*) \notag \\
    &\hspace{-1cm}+\eta_t\left(\frac{3\mu}{2}+3L\right)\mathbb{E}\|\widehat{\mathbf{x}}_t-\widetilde{\mathbf{x}}_t\|^2+\frac{3\eta_tL}{R}\sum_{r=1}^R\mathbb{E}\|\widehat{\mathbf{x}}_t-\widehat{\mathbf{x}}_t^{\left(r\right)}\|^2+\eta_t^2\frac{\sum_{r=1}^R\sigma_r^2}{bR^2}
    \label{eq:loc9}
\end{align}
\noindent From \Lemmaref{sequence-bound-synchronous}, we have $\frac{1}{R}\sum_{r=1}^R\mathbb{E}\|\widehat{\mathbf{x}}_t-\widehat{\mathbf{x}}_t^{\left(r\right)}\|^2\leq 4\eta_t^2G^2H^2$. \Lemmaref{memory-maintenance} and \Lemmaref{bounded-memory-decaying-lr} together imply that $\mathbb{E}\|\widehat{\mathbf{x}}_t-\widetilde{\mathbf{x}}_t\|^2\leq 4C\frac{\eta_t^2}{\gamma^2}H^2G^2$. 
Substituting these back in \eqref{eq:loc9} and letting $e_t=\mathbb{E}[f(\widehat{\mathbf{x}}_t)-f^*]$ gives
\begin{align}
    \mathbb{E}\|\widetilde{\mathbf{x}}_{t+1}-\mathbf{x}^*\|^2&\leq \left(1-\frac{\mu\eta_t}{2}\right)\mathbb{E}\|\widetilde{\mathbf{x}}_t-\mathbf{x}^*\|^2-\frac{\mu\eta_t}{2L}e_t+\eta_t\left(\frac{3\mu}{2}+3L\right)C\frac{4\eta_{t}^2}{\gamma^2}G^2H^2 \notag\\
    &+ (3L\eta_t)4\eta_t^2LG^2H^2 +\eta_t^2\frac{\sum_{r=1}^R\sigma_r^2}{bR^2} \label{eq:loc11}
\end{align}
Now using $\eta_t\leq\nicefrac{1}{4L}$ we have
\begin{align}
    \mathbb{E}\|\widetilde{\mathbf{x}}_{t+1}-\mathbf{x}^*\|^2&\leq \left(1-\frac{\mu\eta_t}{2}\right)\mathbb{E}\|\widetilde{\mathbf{x}}_t-\mathbf{x}^*\|^2-\frac{\mu\eta_t}{2L}e_t+\eta_t\left(\frac{3\mu}{2}+3L\right)C\frac{4\eta_{t}^2}{\gamma^2}G^2H^2 \notag\\
    &+ (3\eta_tL)4\eta_t^2LG^2H^2 +\eta_t^2\frac{\sum_{r=1}^R\sigma_r^2}{bR^2} \label{eq:loc11}
\end{align}
Employing a slightly modified Lemma 3.3 from \cite{memSGD} with $a_{t}=\mathbb{E}\|\widetilde{\mathbf{x}}_t-\mathbf{x}^*\|^2$. $A=\frac{\sum_{r=1}^R\sigma_r^2}{bR^2}$ and $B=4\left(\left(\frac{3\mu}{2}+3L\right)\frac{ CG^2H^2}{\gamma^2}+3L^2G^2H^2\right)$, we have 
\begin{align}
    a_{t+1}\leq \left(1-\frac{\mu\eta_t}{2}\right)a_t-\frac{\mu\eta_t}{2L}e_t+\eta_t^2A+\eta_t^3B
\end{align}
For $\eta_t=\frac{8}{\mu\left(a+t\right)}$ and $w_t=\left(a+t\right)^2$, $S_T=\sum_{t=o}^{T-1}\geq \frac{T^3}{3}$ we have 
\begin{align}
    \frac{\mu}{2LS_T}\sum_{t=0}^{T-1}w_te_t\leq \frac{\mu a^3}{8S_T}a_0+\frac{4T\left(T+2a\right)}{\mu S_T}A+\frac{64T}{\mu^2 S_T}B
\end{align}
From convexity we can finally write 
\begin{align}
    \mathbb{E}f\left(\overline{\mathbf{x}}_T\right)-f^*\leq \frac{L a^3}{4S_T}a_0+\frac{8LT\left(T+2a\right)}{\mu^2 S_T}A+\frac{128LT}{\mu^3 S_T}B
\end{align}
Where $\overline{\mathbf{x}}_T:=\frac{1}{S_T}\sum_{t=0}^{T-1}\left[w_t\left(\frac{1}{R}\sum_{r=1}^R\widehat{\mathbf{x}}_t^{\left(r\right)}\right)\right]=\frac{1}{S_T}\sum_{t=0}^{T-1}w_t\widehat{\mathbf{x}}_t$.
This completes the proof of \Theoremref{convergence-strongly-convex-local}.
\end{proof}

\section{Omitted Details from \Sectionref{loc-async}}
\label{app:supp_async}
As before, in order to prove our results in the asynchronous setting, we define virtual sequences for every worker $r\in[R]$ and for all $t\geq 0$ as follows:
\begin{align*}
& \btx_0^{(r)} := \bwx_0^{(r)} &&&\widetilde{\mathbf{x}}_{t+1}^{\left(r\right)} := \widetilde{\mathbf{x}}_{t}^{\left(r\right)}-\eta_t\nabla f_{i_t^{(r)}}\left(\widehat{\mathbf{x}}_t^{\left(r\right)}\right)
\end{align*}
Define 
\begin{enumerate}
    \item $\widetilde{\mathbf{x}}_{t+1}:=\frac{1}{R}\sum_{r=1}^R\widetilde{\mathbf{x}}_{t+1}^{\left(r\right)}=\widetilde{\mathbf{x}}_t-\frac{\eta_t}{R}\sum_{r=1}^R\nabla f_{i_t^{(r)}}\left(\widehat{\mathbf{x}}_t^{\left(r\right)}\right)$
    \item $\mathbf{p}_t:=\frac{1}{R}\sum_{r=1}^R\nabla f_{i_t^{(r)}}\left(\widehat{\mathbf{x}}_t^{\left(r\right)}\right)$
    \item $\overline{\mathbf{p}}_t:=\mathbb{E}_{\left(i_t\right)}[\mathbf{p}_t]=\frac{1}{R}\sum_{r=1}^R\nabla f^{(r)}\left(\widehat{\mathbf{x}}_t^{\left(r\right)}\right)$
    \item $\widehat{\mathbf{x}}_t=\frac{1}{R}\sum_{r=1}^R\widehat{\mathbf{x}}_t^{\left(r\right)}$
    \item $\mathcal{I}_T^{(r)}=\{t_{(i)}^{(r)}:i\in\mathbb{Z}^+,t_{(i)}^{(r)}\in[T],\,|t_{(i)}^{(r)}-t_{(j)}^{(r)}|\leq H,\forall|i-j|\leq 1\}$
\end{enumerate}

\subsection{Proof of \Lemmaref{sequence-bound-asynchronous}}\label{app:asy1}
\begin{lemma*}[Restating \Lemmaref{sequence-bound-asynchronous}]
Let $gap(\mathcal{I}_T^{(r)}) \leq H$ holds for every $r\in[R]$.
For $\widehat{\bx}_t^{(r)}$ generated according to \Algorithmref{memQSGD-asynchronous} with decaying learning rate $\eta_t$ and letting $\bwx_t=\frac{1}{R}\sum_{r=1}^R\bwx_t^{(r)}$, we have the following bound on the deviation of the local sequences:
\begin{align*}
    \frac{1}{R}\sum_{r=1}^R\mathbb{E}\|\widehat{\bx}_t-\widehat{\bx}_t^{\left(r\right)}\|_2^2\ \leq\ 8(1+C''H^2)\eta_t^2G^2H^2,
\end{align*}
where $C''=8(4-2\gamma)(1+\frac{C}{\gamma^2})$ and $C$ is a constant satisfying $C\geq\frac{4a\gamma(1-\gamma^2)}{a\gamma-4H}$.
\end{lemma*}
\begin{proof}
Fix a time $t$ and consider any worker $r\in[R]$. Let $t_r\in\I_T^{(r)}$ denote the last synchronization step until time $t$ for the $r$'th worker.
Define $t_0':=\min_{r\in [R]}t_r$. 
We need to upper-bound $\frac{1}{R}\sum_{r=1}^R\bbE\|\bwx_t-\bwx_t^{(r)}\|^2$. 
Note that for any $R$ vectors $\bu_1,\hdots,\bu_R$, if we let $\bar{\bu}=\frac{1}{R}\sum_{i=1}^r\bu_i$, then
$\sum_{i=1}^n\|\bu_i-\bar{\bu}\|^2\leq\sum_{i=1}^R\|\bu_i\|^2$. We use this in the first inequality below.
\begin{align}
    \frac{1}{R}\sum_{r=1}^R\mathbb{E}\|\widehat{\mathbf{x}}_t-\widehat{\mathbf{x}}_t^{\left(r\right)}\|^2 &= \frac{1}{R}\sum_{r=1}^R\mathbb{E}\|\widehat{\mathbf{x}}_t^{(r)}-\bar{\bar{\mathbf{x}}}_{t_0'}-(\widehat{\mathbf{x}}_t-\bar{\bar{\mathbf{x}}}_{t_0'})\|^2\notag\\
    &\leq \frac{1}{R}\sum_{r=1}^R\mathbb{E}\|\widehat{\mathbf{x}}_t^{(r)}-\bar{\bar{\mathbf{x}}}_{t_0'}\|^2\notag\\
    &\leq \frac{2}{R}\sum_{r=1}^R \mathbb{E}\|\widehat{\mathbf{x}}_t^{(r)}-\widehat{\mathbf{x}}_{t_r}^{(r)}\|^2+\frac{2}{R}\sum_{r=1}^R\mathbb{E}\|\widehat{\mathbf{x}}_{t_r}^{(r)}-\bar{\bar{{\mathbf{x}}}}_{t_0'}\|^2 \label{eq:async-seq-bound1}
\end{align}
We bound both the terms separately. For the first term:
\begin{align}
\bbE\|\bwx_t^{(r)}-\bwx_{t_r}^{(r)}\|^2 &= \bbE\|\sum_{j=t_r}^{t-1}\eta_j\nabla f_{i_j^{(r)}}\left(\bwx_j^{(r)}\right)\|^2 \notag\\
&\leq (t-t_r)\sum_{j=t_r}^{t-1}\bbE\|\eta_j\nabla f_{i_j^{(r)}}\left(\bwx_j^{(r)}\right)\|^2 \notag\\
&\leq (t-t_r)^2\eta_{t_r}^2G^2 \notag\\
&\leq 4\eta_t^2H^2G^2. \label{eq:async-seq-bound2}
\end{align}
The last inequality \eqref{eq:async-seq-bound2} uses $\eta_{t_r}\leq 2\eta_{t_r+H}\leq2\eta_t$ and $t-t_r\leq H$.
To bound the second term of \eqref{eq:async-seq-bound1}, note that we have
\begin{align}
    \bar{\bar{\mathbf{x}}}_{t_r}^{(r)}=\bar{\bar{\mathbf{x}}}_{t_0'}-\frac{1}{R}\sum_{s=1}^R\sum_{j=t_0'}^{t_r-1}\mathbbm{1}\{j+1\in \mathcal{I}_T^{(s)}\}g_j^{(s)}. \label{eq:async-seq-bound3}
\end{align}
Note that $\bwx_{t_r}^{(r)}=\doublebarbx_{t_r}^{(r)}$, because at synchronization steps, the local parameter vector becomes equal to the global parameter vector. Using this, the Jensen's inequality, and that $\|\mathbbm{1}\{j+1\in\I_T^{(s)}\}g_j^{(s)}\|^2\leq\|g_j^{(s)}\|^2$, we can upper-bound \eqref{eq:async-seq-bound3} as
\begin{align}
\bbE\|\bwx_{t_r}^{(r)}-\doublebarbx_{t_0'}\|^2 &\leq \frac{(t_r-t_0')}{R}\sum_{s=1}^R\sum_{j=t_0'}^{t_r}\bbE\|g_j^{(s)}\|^2 \label{eq:async-seq-bound4}
\end{align}
Now we bound $\bbE\|g_j^{(s)}\|^2$ for any $j\in\{t_0',\hdots,t_r\}$ and $s\in[R]$:
Since $\bbE\|QComp_k(\bu)\|^2\leq B\|\bu\|^2$ holds for every $\bu$, with $B=(4-2\gamma),$\footnote{This can be seen as follows: $\bbE\|QC(\bu)\|^2\leq 2\bbE\|\bu-QC(\bu)\|^2+2\|\bu\|^2\leq 2(1-\gamma)\|\bu\|^2+2\|\bu\|^2$.} we have for any $s\in[R]$ that 
\begin{align}
\bbE\|g_j^{(s)}\|^2 &\leq B\bbE\|m_j^{(s)}+\bx_j^{(s)}-\bwx_{j+\frac{1}{2}}^{(s)}\|^2 \label{eq:async-seq-bound5} \\
&\leq 2B\bbE\|m_j^{(s)}\|^2 + 2B\bbE\|\bx_j^{(s)}-\bwx_{j+\frac{1}{2}}^{(s)}\|^2 \label{eq:async-seq-bound6}
\end{align}
Observe that the proof of \Lemmaref{bounded-memory-decaying-lr} does not depend on the synchrony of the network;
it only uses the fact that $gap(\I_T^{(s)})\leq H$ for any worker $s\in[R]$. Therefore, we can directly use \Lemmaref{bounded-memory-decaying-lr} to bound the first term in \eqref{eq:async-seq-bound2} as $\bbE\|m_j^{(s)}\|^2\leq4C\frac{\eta_j^2}{\gamma^2}H^2G^2$. In order to bound the second term of \eqref{eq:async-seq-bound2}, note that $\bx_j^{(s)}=\bwx_{t_s}^{(s)}$, which implies that $\|\bx_j^{(s)}-\bwx_{j+\frac{1}{2}}^{(s)}\|^2=\|\sum_{l=t_s}^j\eta_l\nabla f_{i_l^{(s)}}\left(\bwx_l^{(s)}\right)\|^2$. Taking expectation 
yields $\bbE\|\bx_j^{(s)}-\bwx_{j+\frac{1}{2}}^{(s)}\|^2\leq 4\eta_{t_s}^2H^2G^2\leq 4\eta_{t_0'}^2H^2G^2$, where in the last inequality we used that $t_0' \leq t_s$. Using these in \eqref{eq:async-seq-bound6} gives
\begin{align}
\bbE\|g_j^{(s)}\|^2 \leq 8B\left(1+\frac{C}{\gamma^2}\right)\eta_{t_0'}^2H^2G^2. \label{eq:async-seq-bound7}
\end{align}
Since $t_0'\leq t\leq t_0'+H$, we have $\eta_{t_0'}\leq 2\eta_{t_0'+H}\leq 2\eta_t$. 
Putting the bound on $\bbE\|g_j^{(s)}\|^2$ (after substituting $\eta_{t_0'}\leq 2\eta_t$ in \eqref{eq:async-seq-bound7}) in \eqref{eq:async-seq-bound4} gives 
\begin{align}
\bbE\|\bwx_{t_r}^{(r)}-\doublebarbx_{t_0'}\|^2\leq 32B\left(1+\frac{C}{\gamma^2}\right)\eta_{t}^2H^4G^2. \label{eq:async-seq-bound8}
\end{align} 
Putting this and the bound from \eqref{eq:async-seq-bound2} back in \eqref{eq:async-seq-bound1} gives
\begin{align*}
\frac{1}{R}\sum_{r=1}^R\bbE\|\bwx_t-\bwx_t^{(r)}\|^2 &\leq 8\eta_t^2H^2G^2 + 64B\left(1+\frac{C}{\gamma^2}\right)\eta_{t}^2H^4G^2 \\
&\leq 8\left[1+8BH^2\left(1+\frac{C}{\gamma^2}\right)\right]\eta_{t}^2H^2G^2.
\end{align*}
This completes the proof of \Lemmaref{sequence-bound-asynchronous}.
\end{proof}
\subsection{Proof of \Lemmaref{sequence-bound-asynchronous-fixed-lr}}\label{app:asy2}
\begin{lemma*}[Restating \Lemmaref{sequence-bound-asynchronous-fixed-lr}]
Let $gap(\mathcal{I}_T^{(r)}) \leq H$ holds for every $r\in[R]$.
By running \Algorithmref{memQSGD-asynchronous} with fixed learning rate $\eta$, we have
\begin{align*}
    \frac{1}{R}\sum_{r=1}^R\mathbb{E}\|\widehat{\bx}_t-\widehat{\bx}_t^{\left(r\right)}\|_2^2\ \leq\ (2+H^2C')\eta^2G^2H^2,
\end{align*}
where $C'=(\frac{16}{\gamma^2}-12)(4-2\gamma)$.
\end{lemma*}
\begin{proof}From \eqref{eq:async-seq-bound5} and \eqref{eq:async-seq-bound6} and using the fact that $\mathbb{E}\|QComp_k(\mathbf{u})\|^2\leq B\|\mathbf{u}\|^2$ for every $\mathbf{u}$, where $B=(4-2\gamma)$, we have the following:
\begin{align}\mathbb{E}\|g_j^{(s)}\|^2&\leq 2B\mathbb{E}\|m_j^{(s)}\|^2+2B\eta^2H^2G^2\notag\\&\leq 8B\frac{(1-\gamma^2)\eta^2}{\gamma^2}H^2G^2+2\eta^2BH^2G^2\notag\\&=2B\left(\frac{4}{\gamma^2}-3\right)\eta^2H^2G^2\label{eq:update_bound_fixedlr}\end{align}
For a fixed learning rate $\eta$, using \eqref{eq:update_bound_fixedlr} and following similar analysis as in \eqref{eq:async-seq-bound2} we can bound  the first term in \eqref{eq:async-seq-bound1} as follows
\begin{align}
    \mathbb{E}\|\widehat{\mathbf{x}}_t^{(r)}-\widehat{\mathbf{x}}_{t_r}^{(r)}\|^2\leq \eta^2H^2G^2\label{eq:async_seq-bound1-fixedlr}
\end{align}
Similarly as in \eqref{eq:async-seq-bound3}-\eqref{eq:async-seq-bound7} we can bound the second term in \eqref{eq:async-seq-bound1} as follows
\begin{align}
    \bbE\|\widehat{\mathbf{x}}_{t_r}^{(r)}-\bar{\bar{{\mathbf{x}}}}_{t_0'}\|^2&\leq 2B\left(\frac{4}{\gamma^2}-3\right)\eta^2H^4G^2\label{eq:async_seq-bound2-fixedlr}
\end{align}
Using \eqref{eq:async_seq-bound1-fixedlr} and \eqref{eq:async_seq-bound2-fixedlr} in \eqref{eq:async-seq-bound1} we can show that 
\begin{align}
    \frac{1}{R}\sum_{r=1}^R\mathbb{E}\|\widehat{\mathbf{x}}_t-\widehat{\mathbf{x}}_t^{\left(r\right)}\|^2&\leq\left[2+4BH^2\left(\frac{4}{\gamma^2}-3\right)\right]\eta^2H^2G^2
\end{align}
This completes the proof of \Lemmaref{sequence-bound-asynchronous-fixed-lr}.
\end{proof}
\subsection{Proof of \Lemmaref{bounded-memory-decaying-lr-asynchronous}}\label{app:asy3}
\begin{lemma*}[Restating \Lemmaref{bounded-memory-decaying-lr-asynchronous}]
Let $gap(\mathcal{I}_T^{(r)}) \leq H$ holds for every $r\in[R]$.
If we run \Algorithmref{memQSGD-asynchronous} with a decaying learning rate $\eta_t$, then we have the following bound on the difference between the true and virtual sequences:
\begin{align*}
\mathbb{E}\|\widehat{\bx}_t-\btx_t\|_2^2\ &\leq\ C'\eta_{t}^2H^4G^2 + 12C\frac{\eta_{t}^2}{\gamma^2}G^2H^2,
\end{align*}
where $C'=192(4-2\gamma)\left(1+\frac{C}{\gamma^2}\right)$ and $C$ is a constant satisfying $C\geq\frac{4a\gamma(1-\gamma^2)}{a\gamma-4H}$.
\end{lemma*}
\begin{proof}
Fix a time $t$ and consider any worker $r\in[R]$. Let $t_r\in\I_T^{(r)}$ denote the last synchronization step until time $t$ for the $r$'th worker. Define $t_0':=\min_{r\in [R]}t_r$. We want to bound $\bbE\|\bwx_t-\btx_t\|^2$. 
Note that in the synchronous case, we have shown in \Lemmaref{memory-maintenance} that $\bwx_t-\bwx_t=\frac{1}{R}\sum_{r=1}^Rm_t^{(r)}$. This does not hold in the asynchronous setting, which makes upper-bounding $\bbE\|\bwx_t-\btx_t\|^2$ a bit more involved.
By definition $\widehat{\mathbf{x}}_t-\widetilde{\mathbf{x}}_t=\frac{1}{R}\sum_{r=1}^R\left(\widehat{\mathbf{x}}_t^{(r)}-\widetilde{\mathbf{x}}_t^{(r)}\right)$. By the definition of virtual sequences and the update rule for $\bwx_t^{(r)}$, 
we also have $\widehat{\bx}_t-\btx_t =\frac{1}{R}\sum_{r=1}^R\left(\widehat{\bx}_{t_r}^{(r)}-\btx_{t_r}^{(r)}\right)$.
This can be written as
\begin{align}
\widehat{\bx}_t-\btx_t &=\left[\frac{1}{R}\sum_{r=1}^R\widehat{\bx}_{t_r}^{(r)}-\bar{\bar{\bx}}_{t_0'}\right]+\left[\bar{\bar{\bx}}_{t_0'}-\bar{\bar{\bx}}_t\right]+\left[\bar{\bar{\bx}}_t-\frac{1}{R}\sum_{r=1}^R\btx_{t_r}^{(r)}\right] \label{async-memory-bound1}
\end{align}
Applying Jensen's inequality and taking expectation gives
\begin{align}
\bbE\|\widehat{\bx}_t-\btx_t\|^2 &\leq \left[\frac{3}{R}\sum_{r=1}^R\bbE\|\widehat{\bx}_{t_r}^{(r)}-\bar{\bar{\bx}}_{t_0'}\|^2\right] + \left[3\bbE\|\bar{\bar{\bx}}_{t_0'}-\bar{\bar{\bx}}_t\|^2\right] + \left[3\bbE\|\bar{\bar{\bx}}_t-\frac{1}{R}\sum_{r=1}^R\btx_{t_r}^{(r)}\|^2\right] \label{eq:async-memory-bound2}
\end{align}
We bound each of the three terms of \eqref{eq:async-memory-bound2} separately.
We have upper-bounded the first term earlier in \eqref{eq:async-seq-bound8}, which is
\begin{align}
\bbE\|\bwx_{t_r}^{(r)}-\doublebarbx_{t_0'}\|^2\leq 32B\left(1+\frac{C}{\gamma^2}\right)\eta_{t}^2H^4G^2, \label{eq:async-memory-bound21}
\end{align}
where $B=(4-2\gamma)$. 
To bound the second term of \eqref{eq:async-memory-bound2}, note that
\begin{align}
    \bar{\bar{\bx}}_t &=\bar{\bar{\bx}}_{0}-\frac{1}{R}\sum_{r=1}^R\sum_{j=0}^{t_r-1}\mathbbm{1}\{j+1\in \mathcal{I}_T^{(r)}\}g_j^{(r)} \label{eq:async-memory-bound22}\\
    &=\bar{\bar{\bx}}_{t_0'}-\frac{1}{R}\sum_{r=1}^R\sum_{j=t_0'}^{t_r-1}\mathbbm{1}\{j+1\in \mathcal{I}_T^{(r)}\}g_j^{(r)} \label{eq:async-memory-bound3}
\end{align}
By applying Jensen's inequality, using $\|\mathbbm{1}\{j+1\in\I_T^{(r)}\}g_j^{(r)}\|^2\leq\|g_j^{(r)}\|^2$, and taking expectation, we can upper-bound \eqref{eq:async-memory-bound3} as
\begin{align*}
\bbE\|\doublebarbx_{t_0'}-\doublebarbx_t\|^2 &\leq  \frac{(t_r-t_0')}{R}\sum_{r=1}^R\sum_{j=t_0'}^{t_r}\bbE\|g_j^{(r)}\|^2
\end{align*}
Using the bound on $\bbE\|g_j^{(r)}\|^2$'s from \eqref{eq:async-seq-bound8} gives
\begin{align}
\bbE\|\doublebarbx_{t_0'}-\doublebarbx_t\|^2 &\leq 32B\left(1+\frac{C}{\gamma^2}\right)\eta_{t}^2H^4G^2. \label{eq:async-memory-bound35}
\end{align}
To bound the last term of \eqref{eq:async-memory-bound2}, note that 
\begin{align}
\btx_{t_r}^{(r)}=\bar{\bar{\bx}}_0-\sum_{j=0}^{t_r-1}\eta_j\nabla f_{i_j^{(r)}}\left(\widehat{\mathbf{x}}_j^{\left(r\right)}\right) \label{eq:async-memory-bound4}
\end{align}
From \eqref{eq:async-memory-bound22} and \eqref{eq:async-memory-bound4}, we can write
\begin{align}\label{eq:async-memory-bound5}
    \bar{\bar{\bx}}_t-\frac{1}{R}\sum_{r=1}^R\btx_{t_r}^{(r)} 
    &=\frac{1}{R}\sum_{r=1}^R\left[\sum_{j=0}^{t_r-1}\eta_j\nabla^{\left(r\right)}f_{\left(i_j\right)}\left(\widehat{\mathbf{x}}_j^{\left(r\right)}\right)-\sum_{j=0}^{t_r-1}\mathbbm{1}\{j+1\in \mathcal{I}_T^{(r)}\}g_j^{(r)}\right]
    \end{align}
Let $t_r^{(1)}$ and $t_r^{(2)}$ be two consecutive synchronization steps in $\I_T^{(r)}$. Then, by the update rule of $\bwx_t^{(r)}$, we have $\bwx_{t_r^{(1)}}^{(r)}-\bwx_{t_r^{(2)}-\frac{1}{2}}^{(r)}=\sum_{j=t_r^{(1)}}^{t_r^{(2)}-1}\nabla f_{i_j^{(r)}}\left(\bwx_j^{(r)}\right)$. Since $\bx_{t_r^{(1)}}^{(r)}=\bwx_{t_r^{(1)}}^{(r)}$ and the workers do not modify their local $\bx_t^{(r)}$'s in between the synchronization steps, we have $\bx_{t_r^{(2)}-1}^{(r)}=\bx_{t_r^{(1)}}^{(r)}=\bwx_{t_r^{(1)}}^{(r)}$.
Therefore, we can write
\begin{align}\label{eq:async-memory-bound66}
\bx_{t_r^{(2)}-1}^{(r)}-\bwx_{t_r^{(2)}-\frac{1}{2}}^{(r)}=\sum_{j=t_r^{(1)}}^{t_r^{(2)}-1}\nabla f_{i_j^{(r)}}\left(\bwx_j^{(r)}\right).
\end{align}
   Using \eqref{eq:async-memory-bound66} for every consecutive synchronization steps, we can equivalently write \eqref{eq:async-memory-bound5} as
    \begin{align}
    \bar{\bar{\bx}}_t-\frac{1}{R}\sum_{r=1}^R\btx_{t_r}^{(r)}&=\frac{1}{R}\sum_{r=1}^R\left[\sum_{\substack{j:j+1\in\mathcal{I}_T^{(r)}\\j\leq t_r-1}}\left(\bx_j^{(r)}-\widehat{\bx}_{j+\frac{1}{2}}^{(r)}-g_j^{(r)}\right)\right]\notag\\
    &=\frac{1}{R}\sum_{r=1}^Rm_{t_r}^{(r)} \notag\\
    &=\frac{1}{R}\sum_{r=1}^R m_t^{(r)}\label{eq:async_memory-equivalence}
\end{align}
In the last inequality, we used the fact that the workers do not update their local memory in between the synchronization steps.
For the reasons given in the proof of \Lemmaref{sequence-bound-asynchronous}, we can directly apply \Lemmaref{bounded-memory-decaying-lr} to bound the local memories and obtain $\bbE\|\frac{1}{R}\sum_{r=1}^R m_t^{\left(r\right)}\|^2\leq \frac{1}{R}\sum_{r=1}^R\bbE\|m_t^{(r)}\|^2\leq 4C\frac{\eta_{t}^2}{\gamma^2}G^2H^2$. This implies
\begin{align}\label{eq:async-memory-bound6}
\bbE\| \bar{\bar{\bx}}_t-\frac{1}{R}\sum_{r=1}^R\btx_{t_r}^{(r)}\|^2 \leq 4C\frac{\eta_{t}^2}{\gamma^2}G^2H^2.
\end{align}
Putting the bounds from \eqref{eq:async-memory-bound21}, \eqref{eq:async-memory-bound35}, and \eqref{eq:async-memory-bound6} in \eqref{eq:async-memory-bound2} and using $B=(4-2\gamma)$ give
\begin{align*}
\bbE\|\widehat{\bx}_t-\btx_t\|^2 &\leq 192(4-2\gamma)\left(1+\frac{C}{\gamma^2}\right)\eta_{t}^2H^4G^2 + 12C\frac{\eta_{t}^2}{\gamma^2}G^2H^2.
\end{align*}
This completes the proof of \Lemmaref{bounded-memory-decaying-lr-asynchronous}.
\end{proof}

\subsection{Proof of \Lemmaref{bounded-memory-fixed-lr-asynchronous}}\label{app:asy4}
\begin{lemma*}[Restating \Lemmaref{bounded-memory-fixed-lr-asynchronous}]
Let $gap(\mathcal{I}_T^{(r)}) \leq H$ holds for every $r\in[R]$.
If we run \Algorithmref{memQSGD-asynchronous} with a fixed learning rate $\eta$, we have
\begin{align*}
\mathbb{E}    \|\widehat{\bx}_t-\btx_t\|_2^2\ &\leq\ 6C'\eta^2H^4G^2+\frac{12\eta^2(1-\gamma^2)}{\gamma^2}G^2H^2,
\end{align*}
where $C'=(4-2\gamma)\left(\frac{8}{\gamma^2}-6\right)$.
\end{lemma*}
\begin{proof}
For a constant learning rate the first term in \eqref{eq:async-memory-bound2} has been bounded earlier in \eqref{eq:async_seq-bound2-fixedlr}. Following similar steps as in \eqref{eq:async-memory-bound3} we would have 
\begin{align}
    \bbE\|\bar{\bar{\bx}}_{t_0'}-\bar{\bar{\bx}}_t\|^2&\leq 2B\left(\frac{4}{\gamma^2}-3\right)\eta^2H^4G^2. \label{eq:async_bound_rand}
\end{align}
Finally, using \eqref{eq:async_seq-bound2-fixedlr},\eqref{eq:async_memory-equivalence}, \Lemmaref{bounded-memory-fixed-lr} and \eqref{eq:async_bound_rand} in \eqref{eq:async-memory-bound2} we have 
\begin{align}
    \bbE\|\widehat{\bx}_t-\btx_t\|^2&\leq 12B\left(\frac{4}{\gamma^2}-3\right)\eta^2H^4G^2+\frac{12\eta^2(1-\gamma^2)}{\gamma^2}G^2H^2,
\end{align}
where $B=(4-2\gamma)$. This completes the proof of \Lemmaref{bounded-memory-fixed-lr-asynchronous}.
\end{proof}

\section{Omitted Details from \Sectionref{expmt}}
\label{app:supp_expmts}
As mentioned in \Footnoteref{scaled-vs-unscaled}, here we compare the performance of Qsparse-local-SGD with scaled and unscaled composed operator $QTop_k$ in the non-convex setting.
We will see that even though the scaled $QTop_k$ from \Lemmaref{composed-quantizer} works better than unscaled $QTop_k$ from \Lemmaref{composed-compression} theoretically (see \Remarkref{scaled-vs-unscaled}), our experiments show the opposite phenomena, that the unscaled $QTop_k$ works at least as good as the scaled $QTop_k$, and strictly better in some cases. 
We can attribute this to the fact that scaling the composed operator is a sufficient condition to obtain better convergence results, which does not 
necessarily mean that in practice also it does better.
Therefore, we perform our experiments in the non-convex setting \Sectionref{expmt} with unscaled $QTop_k$.

\begin{figure*}[ht!]
\centering
\begin{subfigure}{0.49\textwidth}
\centering
\includegraphics[scale=0.55]{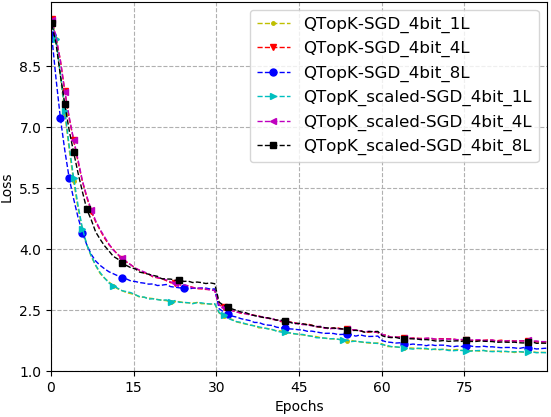}
\captionsetup{justification=centering}
\caption{Training loss vs epochs }
\label{fig:mid_lo-iter_a_q}
\centering
\end{subfigure}
\begin{subfigure}{0.49\textwidth}
\centering
\includegraphics[scale=0.55]{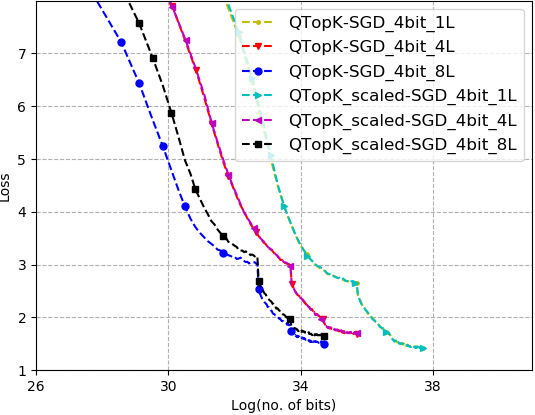}
\captionsetup{justification=centering}
\caption{Training loss vs $\textrm{log}_2$ of communication budget }
\label{fig:mid_lo-bit_a_q}
\centering
\end{subfigure}\\ \vspace{0.5cm}
\begin{subfigure}{0.49\textwidth}
\centering
\includegraphics[scale=0.55]{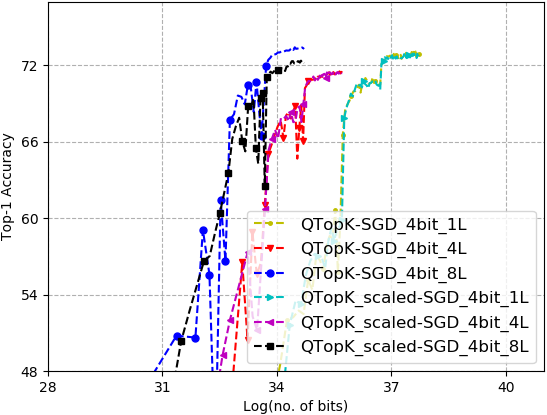}
\captionsetup{justification=centering}
\caption{top-1 accuracy \cite{lapin} for schemes in \Figureref{mid_lo-iter_a_q} }
\label{fig:te_1_a_q}
\centering
\end{subfigure}
\begin{subfigure}{0.49\textwidth}
\centering
\includegraphics[scale=0.55]{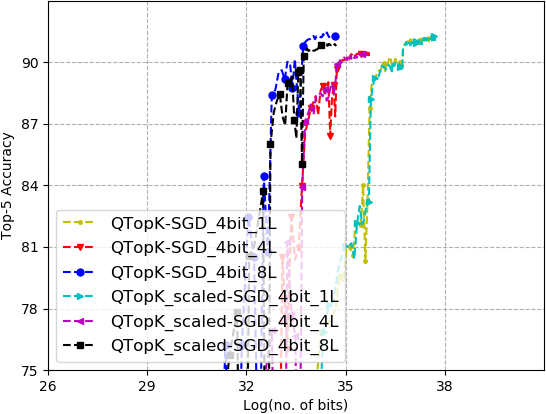}
\captionsetup{justification=centering}
\caption{top-5 accuracy \cite{lapin} for schemes in \Figureref{mid_lo-iter_a_q} }
\label{fig:te_5_a_q}
\centering
\end{subfigure}
\caption{\Figureref{mid_lo-iter_a_q}-\ref{fig:te_5_a_q} demonstrate the comparable performance of \emph{Qsparse-local-SGD} in the non-convex setting with scaled and unscaled $QTop_k$ operators from \Lemmaref{composed-compression} and \Lemmaref{composed-quantizer}, respectively.}
\label{fig:fig1_q}
\end{figure*}

We give plots for the above-mentioned comparison in \Figureref{fig1_q}. 
From \cite{QSGD}, we know that for quantized SGD, without any form of error compensation, the dominating term in the convergence rate is affected by the variance blow-up induced due to stochastic quantization; however, with error compensation, we recover rates matching vanilla SGD despite compression and infrequent communication \Sectionref{main_results_sync}. In \Figureref{fig1_q}, 
$QTop_k$ refers to QSGD composed with the $Top_k$ operator as in \Lemmaref{composed-compression}, and when used with the subscript \emph{scaled}, we introduce a scaling factor of $(1+\beta_{k,s})$ as in \Lemmaref{composed-quantizer}. 
Let $L$ denote the number of local iterations in between two synchronization indices. 
Observe that to achieve a certain target loss or accuracy, both the composed operators perform almost equally in terms of the number of bits transmitted when $L=0,4$, but unscaled operator performs better when $L=8$. 
Therefore, we restrict our use of composed operator in the non-convex setting to the unscaled $QTop_k$ from \Lemmaref{composed-compression}.

\bibliographystyle{alpha}
\bibliography{reference}

\newcommand{\etalchar}[1]{$^{#1}$}
\begin{thebibliography}{NNvD{\etalchar{+}}18}

\bibitem[ABC{\etalchar{+}}16]{AbadiOSDI16}
M.~Abadi, P.~Barham, J.~Chen, Z.~Chen, A.~Davis, J.~Dean, M.~Devin,
  S.~Ghemawat, G.~Irving, M.~Isard, M.~Kudlur, J.~Levenberg, R.~Monga,
  S.~Moore, D.~G. Murray, B.~Steiner, P.~A. Tucker, V.~Vasudevan, P.~Warden,
  M.~Wicke, Y.~Yu, and X.~Zheng.
\newblock Tensorflow: {A} system for large-scale machine learning.
\newblock In {\em {OSDI}}, pages 265--283, 2016.

\bibitem[AGL{\etalchar{+}}17]{QSGD}
D.~Alistarh, D.~Grubic, J.~Li, R.~Tomioka, and M.~Vojnovic.
\newblock {QSGD:} communication-efficient {SGD} via gradient quantization and
  encoding.
\newblock In {\em {NIPS}}, pages 1707--1718, 2017.

\bibitem[AH17]{AjiHeafield17}
Alham~Fikri Aji and Kenneth Heafield.
\newblock Sparse communication for distributed gradient descent.
\newblock In {\em {EMNLP}}, pages 440--445, 2017.

\bibitem[AHJ{\etalchar{+}}18]{alistarh-sparsified}
D.~Alistarh, T.~Hoefler, M.~Johansson, N.~Konstantinov, S.~Khirirat, and
  C.~Renggli.
\newblock The convergence of sparsified gradient methods.
\newblock In {\em {NeurIPS}}, pages 5977--5987, 2018.

\bibitem[BM11]{bach_nonasymp}
Francis~R. Bach and Eric Moulines.
\newblock Non-asymptotic analysis of stochastic approximation algorithms for
  machine learning.
\newblock In {\em {NIPS}}, pages 451--459, 2011.

\bibitem[Bot10]{Bottou10}
L.~Bottou.
\newblock Large-scale machine learning with stochastic gradient descent.
\newblock In {\em COMPSTAT}, pages 177--186, 2010.

\bibitem[BWAA18]{signsgd1}
J.~Bernstein, Y.~Wang, K.~Azizzadenesheli, and A.~Anandkumar.
\newblock {SignSGD:} compressed optimisation for non-convex problems.
\newblock In {\em {ICML}}, pages 559--568, 2018.

\bibitem[CH16]{chen2016scalable}
Kai Chen and Qiang Huo.
\newblock Scalable training of deep learning machines by incremental block
  training with intra-block parallel optimization and blockwise model-update
  filtering.
\newblock In {\em {ICASSP}}, pages 5880--5884, 2016.

\bibitem[Cop15]{Coppola15}
Gregory~F. Coppola.
\newblock {\em Iterative parameter mixing for distributed large-margin training
  of structured predictors for natural language processing}.
\newblock PhD thesis, University of Edinburgh, {UK}, 2015.

\bibitem[DCLT18]{BERT-arxiv18}
Jacob Devlin, Ming{-}Wei Chang, Kenton Lee, and Kristina Toutanova.
\newblock {BERT:} pre-training of deep bidirectional transformers for language
  understanding.
\newblock {\em CoRR}, abs/1810.04805, 2018.

\bibitem[GMT73]{GitlinMazo73}
R.~Gitlin, J.~Mazo, and M.~Taylor.
\newblock On the design of gradient algorithms for digitally implemented
  adaptive filters.
\newblock {\em IEEE Transactions on Circuit Theory}, 20(2):125--136, March
  1973.

\bibitem[HK14]{hazan}
Elad Hazan and Satyen Kale.
\newblock Beyond the regret minimization barrier: optimal algorithms for
  stochastic strongly-convex optimization.
\newblock {\em Journal of Machine Learning Research}, 15(1):2489--2512, 2014.

\bibitem[HM51]{RobbinsMonro51}
Robbins Herbert and Sutton Monro.
\newblock A stochastic approximation method.
\newblock {\em The Annals of Mathematical Statistics. JSTOR}, 22, no.
  3:400--407, 1951.

\bibitem[HZRS16]{resnet50}
K.~He, X.~Zhang, S.~Ren, and J.~Sun.
\newblock Deep residual learning for image recognition.
\newblock In {\em {CVPR}}, pages 770--778, 2016.

\bibitem[KB15]{adam}
Diederik~P. Kingma and Jimmy Ba.
\newblock Adam: {A} method for stochastic optimization.
\newblock In {\em {ICLR}}, 2015.

\bibitem[Kon17]{KonecnyThesis}
Jakub Konecn{\'{y}}.
\newblock Stochastic, distributed and federated optimization for machine
  learning.
\newblock {\em CoRR}, abs/1707.01155, 2017.

\bibitem[KRSJ19]{efsignsgd}
Sai~Praneeth Karimireddy, Quentin Rebjock, Sebastian~U. Stich, and Martin
  Jaggi.
\newblock Error feedback fixes signsgd and other gradient compression schemes.
\newblock In {\em {ICML}}, pages 3252--3261, 2019.

\bibitem[KSJ19]{stich_gossip}
Anastasia Koloskova, Sebastian~U. Stich, and Martin Jaggi.
\newblock Decentralized stochastic optimization and gossip algorithms with
  compressed communication.
\newblock In {\em {ICML}}, pages 3478--3487, 2019.

\bibitem[LBBH98]{mnist}
Y.~LeCun, L.~Bottou, Y.~Bengio, and P.~Haffner.
\newblock Gradient-based learning applied to document recognition.
\newblock In {\em Proceedings of the IEEE, 86(11):2278-2324}, 1998.

\bibitem[LHM{\etalchar{+}}18]{DeepCompICLR18}
Y.~Lin, S.~Han, H.~Mao, Y.~Wang, and W.~J. Dally.
\newblock Deep gradient compression: Reducing the communication bandwidth for
  distributed training.
\newblock In {\em ICLR}, 2018.

\bibitem[LHS15]{lapin}
Maksim Lapin, Matthias Hein, and Bernt Schiele.
\newblock Top-k multiclass {SVM}.
\newblock In {\em {NIPS}}, pages 325--333, 2015.

\bibitem[MMR{\etalchar{+}}17]{McMahan16}
B.~McMahan, E.~Moore, D.~Ramage, S.~Hampson, and B.~A. y~Arcas.
\newblock Communication-efficient learning of deep networks from decentralized
  data.
\newblock In {\em {AISTATS}}, pages 1273--1282, 2017.

\bibitem[MPP{\etalchar{+}}17]{perturbed}
H.~Mania, X.~Pan, D.~S. Papailiopoulos, B.~Recht, K.~Ramchandran, and M.~I.
  Jordan.
\newblock Perturbed iterate analysis for asynchronous stochastic optimization.
\newblock {\em {SIAM} Journal on Optimization}, 27(4):2202--2229, 2017.

\bibitem[NJLS09]{nemirovski}
Arkadi Nemirovski, Anatoli Juditsky, Guanghui Lan, and Alexander Shapiro.
\newblock Robust stochastic approximation approach to stochastic programming.
\newblock {\em {SIAM} Journal on Optimization}, 19(4):1574--1609, 2009.

\bibitem[NNvD{\etalchar{+}}18]{nguyen18}
Lam~M. Nguyen, Phuong~Ha Nguyen, Marten van Dijk, Peter Richt{\'{a}}rik, Katya
  Scheinberg, and Martin Tak{\'{a}}c.
\newblock {SGD} and hogwild! convergence without the bounded gradients
  assumption.
\newblock In {\em {ICML}}, pages 3747--3755, 2018.

\bibitem[RB93]{rprop}
M.~{Riedmiller} and H.~{Braun}.
\newblock A direct adaptive method for faster backpropagation learning: the
  rprop algorithm.
\newblock In {\em IEEE International Conference on Neural Networks}, pages
  586--591 vol.1, March 1993.

\bibitem[RRWN11]{hogwild_recht}
Benjamin Recht, Christopher R{\'{e}}, Stephen~J. Wright, and Feng Niu.
\newblock Hogwild: {A} lock-free approach to parallelizing stochastic gradient
  descent.
\newblock In {\em {NIPS}}, pages 693--701, 2011.

\bibitem[RSS12]{rakhlin}
A.~Rakhlin, O.~Shamir, and K.~Sridharan.
\newblock Making gradient descent optimal for strongly convex stochastic
  optimization.
\newblock In {\em {ICML}}, 2012.

\bibitem[SB18]{sergeev2018hvd}
A.~Sergeev and M.~D. Balso.
\newblock Horovod: fast and easy distributed deep learning in tensorflow.
\newblock {\em CoRR}, abs/1802.05799, 2018.

\bibitem[SCJ18]{memSGD}
S.~U. Stich, J.~B. Cordonnier, and M.~Jaggi.
\newblock Sparsified {SGD} with memory.
\newblock In {\em {NeurIPS}}, pages 4452--4463, 2018.

\bibitem[SFD{\etalchar{+}}14]{speech1}
F.~Seide, H.~Fu, J.~Droppo, G.~Li, and D.~Yu.
\newblock 1-bit stochastic gradient descent and its application to
  data-parallel distributed training of speech dnns.
\newblock In {\em {INTERSPEECH}}, pages 1058--1062, 2014.

\bibitem[SSS07]{shalev-shwartz}
Shai Shalev{-}Shwartz, Yoram Singer, and Nathan Srebro.
\newblock Pegasos: Primal estimated sub-gradient solver for {SVM}.
\newblock In {\em {ICML}}, pages 807--814, 2007.

\bibitem[Sti19]{localsgd2}
Sebastian~U. Stich.
\newblock Local {SGD} converges fast and communicates little.
\newblock In {\em ICLR}, 2019.

\bibitem[Str15]{speech2}
Nikko Strom.
\newblock Scalable distributed {DNN} training using commodity {GPU} cloud
  computing.
\newblock In {\em {INTERSPEECH}}, pages 1488--1492, 2015.

\bibitem[SYKM17]{teertha}
A.~Theertha Suresh, F.~X. Yu, S.~Kumar, and H.~B. McMahan.
\newblock Distributed mean estimation with limited communication.
\newblock In {\em {ICML}}, pages 3329--3337, 2017.

\bibitem[TH12]{rmsprop}
T.~Tieleman and G~Hinton.
\newblock {\em RMSprop. Coursera: Neural Networks for Machine Learning, Lecture
  6.5}.
\newblock 2012.

\bibitem[WHHZ18]{ecqsgd}
J.~Wu, W.~Huang, J.~Huang, and T.~Zhang.
\newblock Error compensated quantized {SGD} and its applications to large-scale
  distributed optimization.
\newblock In {\em {ICML}}, pages 5321--5329, 2018.

\bibitem[WJ18]{coop_sgd}
Jianyu Wang and Gauri Joshi.
\newblock Cooperative {SGD:} {A} unified framework for the design and analysis
  of communication-efficient {SGD} algorithms.
\newblock {\em CoRR}, abs/1808.07576, 2018.

\bibitem[WSL{\etalchar{+}}18]{atomo}
H.~Wang, S.~Sievert, S.~Liu, Z.~B. Charles, D.~S. Papailiopoulos, and
  S.~Wright.
\newblock {ATOMO:} communication-efficient learning via atomic sparsification.
\newblock In {\em {NeurIPS}}, pages 9872--9883, 2018.

\bibitem[WWLZ18]{wangni}
J.~Wangni, J.~Wang, J.~Liu, and T.~Zhang.
\newblock Gradient sparsification for communication-efficient distributed
  optimization.
\newblock In {\em {NeurIPS}}, pages 1306--1316, 2018.

\bibitem[WXY{\etalchar{+}}17]{terngrad}
W.~Wen, C.~Xu, F.~Yan, C.~Wu, Y.~Wang, Y.~Chen, and H.~Li.
\newblock Terngrad: Ternary gradients to reduce communication in distributed
  deep learning.
\newblock In {\em {NIPS}}, pages 1508--1518, 2017.

\bibitem[WYL{\etalchar{+}}18]{yin_sayed}
Tianyu Wu, Kun Yuan, Qing Ling, Wotao Yin, and Ali~H. Sayed.
\newblock Decentralized consensus optimization with asynchrony and delays.
\newblock {\em {IEEE} Trans. Signal and Information Processing over Networks},
  4(2):293--307, 2018.

\bibitem[YJY19]{yu_momentum}
Hao Yu, Rong Jin, and Sen Yang.
\newblock On the linear speedup analysis of communication efficient momentum
  sgd for distributed non-convex optimization.
\newblock In {\em {ICML}}, pages 7184--7193, 2019.

\bibitem[YYZ19]{alibaba_local}
Hao Yu, Sen Yang, and Shenghuo Zhu.
\newblock Parallel restarted {SGD} with faster convergence and less
  communication: Demystifying why model averaging works for deep learning.
\newblock In {\em {AAAI}}, pages 5693--5700, 2019.

\bibitem[ZDJW13]{ZhangNIPS13}
Y.~Zhang, J.~C. Duchi, M.~I. Jordan, and M.~J. Wainwright.
\newblock Information-theoretic lower bounds for distributed statistical
  estimation with communication constraints.
\newblock In {\em {NIPS}}, pages 2328--2336, 2013.

\bibitem[ZDW13]{ZhangDuchiJMLR13}
Y.~Zhang, J.~C. Duchi, and M.~J. Wainwright.
\newblock Communication-efficient algorithms for statistical optimization.
\newblock {\em Journal of Machine Learning Research}, 14(1):3321--3363, 2013.

\bibitem[ZSMR16]{zhang_local}
Jian Zhang, Christopher~De Sa, Ioannis Mitliagkas, and Christopher R{\'{e}}.
\newblock Parallel {SGD:} when does averaging help?
\newblock {\em CoRR}, abs/1606.07365, 2016.

\end{thebibliography}
\end{document}